\documentclass[10pt, letterpaper]{article}
\usepackage[margin=1in, top=0.8in]{geometry}
\usepackage[colorlinks, pagebackref, linkcolor=red, citecolor=blue, linktoc=page]{hyperref}
\usepackage{graphicx}
\usepackage{subfigure}
\usepackage{caption}
\usepackage{enumitem}

\usepackage{color}
\usepackage{xcolor}

\definecolor{tianqin}{RGB}{0, 139, 139}
\definecolor{luolan}{RGB}{199, 21, 133}
\definecolor{stanford}{HTML}{81221c}

\newcommand{\email}[2]{\href{mailto:#1@#2}{\color{luolan}\tt #1@#2}}

\usepackage{amsmath}
\usepackage{amsfonts}
\usepackage{amssymb}
\usepackage{mathtools}
\mathtoolsset{showonlyrefs}

\usepackage{accents}
\usepackage{dsfont} 
\usepackage{pifont}

\usepackage{natbib}

\usepackage{amsthm}

\newtheoremstyle{mytheoremstyle}    
    {}                              
    {}                              
    {\it}                           
    {}                              
    {\bfseries}                     
    {.}                             
    {.5em}                          
    {}                              
\theoremstyle{mytheoremstyle}

\newtheorem{theorem}{Theorem}
\newtheorem{lemma}[theorem]{Lemma}

\newtheorem{assumption}[theorem]{Assumption}
\newtheorem{proposition}[theorem]{Proposition}

\newtheorem{remark}[theorem]{Remark}

\numberwithin{theorem}{section}
\numberwithin{equation}{section}

\usepackage{amsmath,amsfonts,bm}

\def\1{\bm{1}}

\newcommand{\ub}{\mathbf{u}}
\newcommand{\vb}{\mathbf{v}}
\newcommand{\wb}{\mathbf{w}}
\newcommand{\xb}{\mathbf{x}}

\newcommand{\bx}{\bm{x}}

\newcommand{\Ib}{\mathbf{I}}

\newcommand{\Wb}{\mathbf{W}}

\newcommand{\cC}{\mathcal{C}}

\newcommand{\cO}{\mathcal{O}}

\newcommand{\cU}{\mathcal{U}}
\newcommand{\cV}{\mathcal{V}}
\newcommand{\cW}{\mathcal{W}}

\newcommand{\NN}{\mathbb{N}}

\newcommand{\bxi}{\bm{\xi}}

\newcommand{\argmin}{\mathop{\mathrm{argmin}}}
\newcommand{\argmax}{\mathop{\mathrm{argmax}}}

\newcommand{\Tr}{\mathop{\text{tr}}\kern.2ex}

\newcommand{\norm}[1]{\|#1\|}

\newcommand{\dotp}[2]{\langle{#1},{#2}\rangle}

\newcommand{\et}{\tilde{\eta}}
\DeclareRobustCommand{\munderbar}[1]{\underaccent{\bar}{#1}}

\title{Benign Oscillation of Stochastic Gradient Descent with Large Learning Rates}

\author{
    Miao Lu\thanks{Equal contribution. 
    This work was done when Miao Lu and Beining Wu were visiting the Department of Computer Science at the University of Hong Kong.} \thanks{Department of Management Science and Engineering, Stanford University. Email: \email{miaolu}{stanford.edu}}\and
    Beining Wu\footnotemark[1] \thanks{Department of Statistics, University of Chicago. Email: \email{beiningw}{uchicago.edu}}\and
    Xiaodong Yang\thanks{Department of Statistics, Harvard University. Email: \email{xyang}{g.harvard.edu}}\and
    Difan Zou\thanks{Department of Computer Science \& Institute of Data Science, The University of Hong Kong. Email: \email{dzou}{cs.hku.hk}}
}

\date{\small{\today}}

\begin{document}


\maketitle

\begin{abstract}
    In this work, we theoretically investigate the generalization properties of neural networks (NN) trained by stochastic gradient descent (SGD) algorithm with \emph{large learning rates}. 
    Under such a training regime, our finding is that, the \emph{oscillation} of the NN weights caused by the large learning rate SGD training turns out to be beneficial to the generalization of the NN, which potentially improves over the same NN trained by SGD with small learning rates that converges more smoothly.
    In view of this finding, we call such a phenomenon ``\emph{benign oscillation}". 
    Our theory towards demystifying such a phenomenon builds upon the \emph{feature learning} perspective of deep learning. 
    Specifically, we consider a feature-noise data generation model that consists of (i) \emph{weak features} which have a small $\ell_2$-norm and appear in each data point; (ii) \emph{strong features} which have a larger $\ell_2$-norm but only appear in a certain fraction of all data points; and (iii) noise.
    We prove that NNs trained by oscillating SGD with a large learning rate can effectively learn the weak features in the presence of those strong features.
    In contrast, NNs trained by SGD with a small learning rate can only learn the strong features but makes little progress in learning the weak features. 
    Consequently, when it comes to the new testing data which consist of only weak features, the NN trained by oscillating SGD with a large learning rate could still make correct predictions consistently, while the NN trained by small learning rate SGD fails.
    Our theory sheds light on how large learning rate training benefits the generalization of NNs.
    Experimental results demonstrate our  finding on ``benign oscillation".
\end{abstract}





\section{Introduction}\label{sec: intro}

While deep neural networks (NNs) have achieved tremendous empirical success in various domains including images, language processing, decision-making, etc, the theoretical understanding of deep learning is still far behind satisfactory, especially the relationships between optimization of the NN and its generalization. 
From the viewpoint of optimization, using a \emph{large learning rate} in NN training has been empirically shown to be of vital importance to its generalization \citep{he2016deep, xing2018walk, smith2019super, frankle2019early, damian2021label, kaur2023maximum}.
Nevertheless, a principled theoretical understanding for the mechanism behind the benefits of large learning rate NN training still remains limited.

To better capture the key ingredients in the training dynamics of stochastic gradient descent (SGD) with large learning rates, we train two ResNets \citep{he2016deep} using SGD with small and large learning rates respectively, and present the training and testing results in Figure \ref{fig:oscillation}. 
When using a large learning rate SGD, we can clearly observe an ``oscillating'' training curve, i.e., the training loss fluctuates at different iterations (generally this happens only when the learning rate exceeds the inverse of the objective smoothness), while for small learning rate SGD, the training curve is much smoother and converges more rapidly.
On the other hand, the smooth convergence in training loss can not bring any benefit for the testing performance -- SGD with a large learning rate achieves a significantly higher test accuracy than SGD with a small learning rate.
These empirical observations suggest that the \emph{oscillation} during training could be closely tied to the better generalization performance achieved by SGD with large learning rates. 

Motivated by the previous observations, in this paper, we study the learning dynamics of SGD with large learning rates by investigating the oscillation happening during the optimization process,  and we explain the benefits of oscillation to the generalization performance. 
The key message is that, compared to the smooth convergence achieved by SGD with small learning rates, 
\begin{center}
    \emph{the oscillation prevents the over-greedy convergence and serves as the engine that drives the learning of less-prominent data patterns}.
\end{center}
These data patterns would be beneficial for the NN to generalize well on unseen testing data.
This interprets from the theoretical side why large learning rate training can help NN to generalize better in practice.



Our investigation of SGD with large learning rates for NN training builds upon the feature learning perspective of deep learning theory \citep{allen2020towards}, which explicitly considers data models consisting of different types of features and noise. 
For the sake of our goal, we consider a new feature-noise data model that consists of two types of features with different strength and different distributions among data.
Based on this data model, by carefully tracking the process of feature learning of a NN trained by SGD with large or small learning rates, we prove that only when trained by large learning rate SGD would the NN effectively learn the key features for generalizing to \emph{each} new data point.
The NN trained by small learning rate SGD fails to generalize to certain testing data 
because of the limited learning of the features which are crucial to the generalization to those new data points.
This shows a division of the generalization property
of the NN trained by large and small learning rates respectively.

To explain this phenomenon,
our theory then identifies the core incentives for the superior performance of large learning rate SGD to learn the key features as the oscillation during NN training, which also seems to be related to the regime of ``\emph{edge of stability}" \citep{cohen2020gradient} in the NN optimization.
Intuitively, the oscillation can prevent over-greedy convergence which could only leverage the most prominent components of the data, thus allowing for all useful components to be discovered and learned via gradient descent.
In view of our finding that indicates oscillating NN training with large learning rates possibly resulting in better generalization, we refer to such a phenomenon as ``\emph{benign oscillation}".

\subsection{Outline of the Paper}
This paper is organized as following. 
The remaining of this section summarizes our key contributions and discusses related works on large learning rate NN training and feature learning.
Section~\ref{sec: problem setup} defines the problem settings.
In Section~\ref{sec: one data}, we go through the key motivations and results of our theory on a simplified one-data noiseless setting. 
Without the burden of multiple data training analysis, we could clearly present our core ideas.
In Section~\ref{sec: main theory}, we show our main results on the multiple data setting with numerical demonstrations.


\begin{figure}[!t]

\centering
\subfigure{
\includegraphics[width=.45\textwidth]{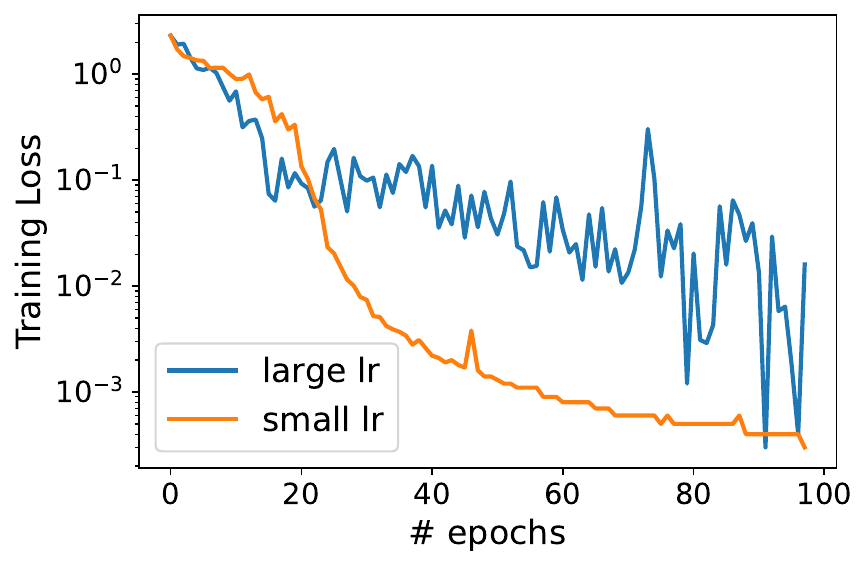}
}
\subfigure{
\includegraphics[width=.435\textwidth]{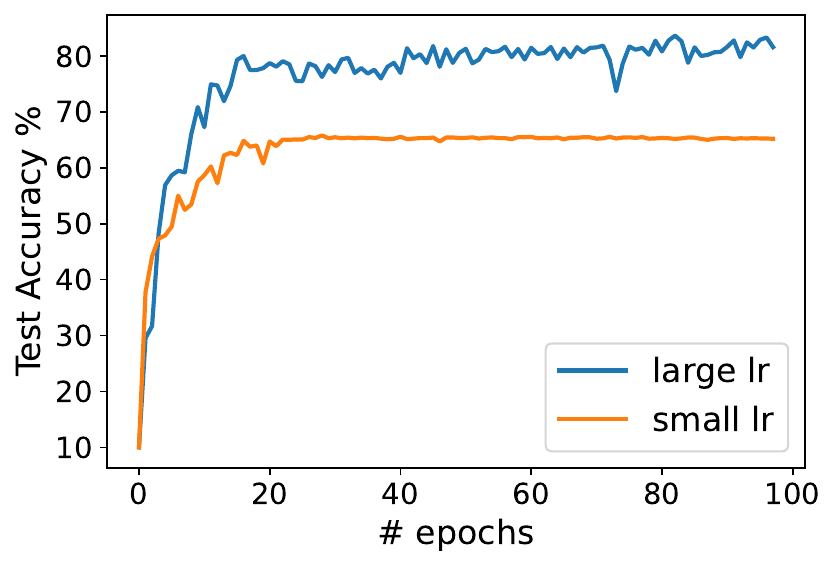}
}
\caption{Training and testing performance of ResNet-18 on CIFAR-10 dataset, when trained via SGD with small and large learning rates ($\eta=0.01$ vs. $\eta=0.75$). We adopt the same configuration as in \citet{andriushchenko2023sgd}: using weight decay but no momentum and no data augmentation. 
A clear difference between the large learning rate training and  small learning rate training can be observed: SGD with a large learning rate leads to an ``oscillating'' training curve with higher testing accuracy; SGD with a small learning rate has a rapid and smooth convergence but gives lower testing accuracy.\label{fig:oscillation}
}
\end{figure}

\subsection{Our Contributions}


\paragraph{Dynamic analysis framework for SGD with large learning rates.}  
We provide a theoretical framework to understand and explain the oscillation in NN training by SGD with a large learning rate.
Specifically, we consider a feature-noise data generation model which consists of two types of features -- the \emph{strong features} and the \emph{weak features} -- that have different strength and distributions among data to capture our core ideas towards explaining the relationship between large learning rate SGD training and generalization.
Then, our theoretical framework establishes a sharp characterization of the learning dynamics of these features and noise, based on which we can precisely analyze the generalization of the NN trained by SGD with small or large learning rates. 
We remark that in general studying the NN optimization dynamics when the learning rate is greater than twice inversed smoothness is quite challenging, and our theoretical analysis framework based upon the feature-noise model potentially provides useful guidance which could be leveraged to study other nonconvex optimization problems of independent interest in the future.

\paragraph{A new theoretical argument for feature learning driven by oscillation.} 
The key to explaining the large learning rate training regime is a novel theoretical argument for learning the weak features driven by oscillation.
As we illustrate in Section~\ref{sec: one data}, the oscillation of the NN value (prediction) around the target (label) does \emph{not} cancel with each other.
Instead, the fluctuations would accumulate linearly over time (Lemma~\ref{lem: linear accumulation}).
This further serves as the engine driving the learning of the weak features, resulting in better generalization.
This characterizes the distinctive training dynamics of SGD under the large learning rate training regime,
revealing the benefits of the oscillation in learning useful data patterns.

\paragraph{Division for generalization by different learning rates.}
In contrast to effectively learning the weak features by large learning rate oscillating SGD training, we also show that the smooth and rapid convergence achieved by SGD with small learning rates would \emph{not} help the NN learn the weak features, thus being unable to generalize to the new data without the strong features.
This gives a division of the generalization property of NNs trained by large and small learning rates, which further demonstrates the benefits of the oscillation.

\subsection{Related Works}
\label{subsec: related works}


\paragraph{Large learning rate NN training.} 
Gradient descent training with large learning rates for deep learning is receiving an ever increasing attention in recent years \citep{cohen2020gradient, lewkowycz2020large, jastrzebski2021catastrophic, cohen2022adaptive, andriushchenko2023sgd}.
For GD training, the phenomenon of ``\emph{edge of stability}"  \citep{cohen2020gradient,  cohen2022adaptive} showed that the sharpness of the loss Hessian would finally hover just above $2/\eta$, and thus a larger learning rate would prefer a flatter minimum and possibly better generalization, and this have received great attention for recent years \citep{arora2022understanding,ahn2022understanding,chen2022gradient, damian2021label, wang2022analyzing,zhu2022understanding, wu2023implicit}.
Another related phenomenon is the ``catapults" in GD training \citep{lewkowycz2020large} which happens during large learning rate training phase with a sharp increase-then-decrease spike in the training loss.
The mechanism behind this phenomenon is further investigated by \cite{kalra2023phase,zhu2022quadratic, zhu2023catapults}.

Besides,
\cite{li2019towards} studied the regularization effect of large learning rates on SGD at initialization which results in better generalization than using a small initial learning rate.
\cite{wu2021direction} considered the implicit bias of SGD with a moderate large learning rate for overparametrized linear regression.
\cite{wu2023implicit} then studied the implicit bias of large learning rate GD training in logistic regression.
Additionally, \cite{andriushchenko2023sgd} showed that SGD with a large learning rate could help NNs to learn sparse features from data, but did not provide rigorous theoretical justifications.
We highlight that our theoretical work on large learning rate SGD builds upon a multi-pass fashion of SGD and a feature-noise data generation model, which is different from previous works \citep{li2019towards, wu2021direction, andriushchenko2023sgd} where noise-approximated-SGD is used for analysis.
Also, we study the behavior of large learning rate SGD by focusing on the role of oscillation, which is also largely different from the prior works.

\paragraph{Feature learning in deep learning theory.}
There has been a long line of research in deep learning theory from the perspective of feature learning \citep{allen2020towards, wen2021toward, zou2021understanding, cao2022benign, chen2022towards, zou2023benefits, huang2023graph, yang2023stochastic}. 
The core idea is that, by explicitly characterizing the dynamics of feature learning during training, one can figure out how different algorithms and data structures can influence the learning of features by the neural network, further uncovering the properties of interests in deep learning, e.g., ensemble \citep{allen2020towards}, adaptive gradients \citep{zou2021understanding}, the phenomenon of benign overfitting \citep{cao2022benign}, data augmentation via mixup \citep{zou2023benefits}, etc.
Specifically, the work of \cite{cao2022benign} showed that under small learning rate regimes, training on data with low \emph{signal-to-noise ratio} (SNR) would result in \emph{harmful overfitting}, leading to the poor generalization abilities of the neural network. 
Our work extends this line of research to the less theoretically understood regime of large learning rates by characterizing the feature learning process when oscillation happens during gradient descent and explaining its benefits to generalization.

\section{Problem Setting}\label{sec: problem setup}

In this section, we introduce the theoretical setting for our investigation of generalization properties of SGD through the task of binary classification. 
We first introduce the multi-view data generation model and then define the two-layer convolutional neural network and the SGD algorithm.

\paragraph{Data generation model.}
We let $\vb \perp \ub\in\mathbb{R}^d$ be two fixed vectors, denoting the signal (or feature) part shared by each data point. Each data point, denoted by $(\xb,y)$ where $\xb=(\xb^{(1)},\xb^{(2)},\xb^{(3)})$ contains $3$ patches, is generated as following: let $y\in\{1, -1\}$ be independently generated by $\mathbb{P}(y = 1) = \mathbb{P}(y = -1) = 1/2$, and 
\begin{itemize}
\item \textbf{Weak signal patch.} One patch of $\xb$ is taken by the weak signal $y\cdot \vb$;
\item \textbf{Strong signal patch.} With probability $1-\rho$, one patch of $\xb$ that is different from $y\cdot \vb$, is taken by the strong signal $y\cdot \ub$;
\item \textbf{Noise patch.} All the remaining patches are taken by independent Gaussian noise $\boldsymbol{\xi}\sim N(0,\sigma_p^2(\Ib_d-\vb\vb^\top/\|\vb\|_2^2 - \ub\ub^\top/\|\ub\|_2^2))$ for some variance $\sigma_p>0$.
\end{itemize}

For simplicity, we refer to the data with strong signal patch as the \emph{strong data}, denoted by 
\begin{align}\label{eq: strong data}
    (\mathbf{x},y) = \big((y\cdot\mathbf{u},y\cdot\mathbf{v},\boldsymbol{\xi}),y\big),
\end{align} and we refer to the data with only the weak signal patch as the \emph{weak data}, denoted by 
\begin{align}\label{eq: weak data}
    (\mathbf{x},y) = \big((\widetilde{\boldsymbol{\xi}},y\cdot\mathbf{v},\boldsymbol{\xi}),y\big).
\end{align}
Here by ``strong", we mean a vector with a larger $\ell_2$-norm, as we would specify in the theory part.
Intuitively, the weak signal $y\cdot\mathbf{v}$ can be interpreted as the invariant and common signals across data like the shape of key objects in an image.
The strong signal $y\cdot\mathbf{u}$ can be understood as the background or the domain information which is stronger but only appears in a certain fraction of all data points.
This indicates that in order for a classifier to generalize well to all new data, it must effectively learn the weak signal $y\cdot\mathbf{v}$.

Our proposed data generation model is adapted from the feature-learning-based line of research on deep learning \citep{allen2020towards, cao2022benign, zou2023benefits},
and it can serve as a good theoretical platform to investigate the relationship between oscillating NN training by large learning rate and the NN generalization.
Finally, we remark that this data model can be extended for generality, e.g., multiple features, more patches, multi-class data. 
In fact, as long as the signal and noise patches have properly different strength and fractions among data, our theoretical analysis can be directly applied.

\paragraph{Two-layer CNN.} 
We consider a two-layer convolutional neural network (CNN) with filters applied to the three patches separately.
We assign the parameters of the second layer of the CNN to a fixed $+1$ and $-1$, respectively.
Formally, the CNN function $f(\cdot;\mathbf{W}):\mathbb{R}^{3d}\mapsto\mathbb{R}$ is defined as 
\begin{align}\label{eq: cnn}
    f(\mathbf{x};\mathbf{W}) = \sum_{j\in\{\pm 1\}} jF_j(\mathbf{x};\mathbf{W}_j),\quad \text{with}\quad F_j(\mathbf{x};\mathbf{W}_j) = \frac{1}{m} \sum_{r\in[m]} \sum_{p=1}^3\sigma(\langle\mathbf{w}_{j,r},\mathbf{x}^{(p)}\rangle),
\end{align}
where $m\in\mathbb{N}_+$ is the number of filters (i.e., neurons), $\sigma(z) = (\max\{z,0\})^2$ is the $\mathrm{ReLU}^2$ activation function, and $\mathbf{w}_{j,r}\in\mathbb{R}^d$ denotes the weights of the $r$-th neuron of $F_j$.
We use $\mathbf{W} = \{\mathbf{W}_j\}_{j\in\{\pm 1\}}$ and $\mathbf{W}_j = \{\mathbf{w}_{j,r}\}_{r\in[m]}$ to denote the collections of the weights.

\paragraph{Loss function and stochastic gradient descent (SGD).}
Having access to $n$ i.i.d. samples from the data generation model, $\mathcal{S} = \{(\mathbf{x}_i,y_i)\}_{i\in[n]}$, we solve a binary classification task by minimizing the following mean squared loss,
\begin{align}\label{eq: loss two way multiple data sgd}
    L(\mathbf{W}) = \frac{1}{n}\sum_{i\in[n]} \ell(f(\mathbf{x}_i;\mathbf{W}), y_i) = \frac{1}{2n}\sum_{i\in[n]}\big(f(\mathbf{x}_i;\mathbf{W}) - y_i\big)^2,
\end{align}
where $\ell(f(\mathbf{x}_i;\mathbf{W}), y_i) = (f(\mathbf{x}_i;\mathbf{W}) - y_i)^2/2$ is the loss on a single data point. Inspired by ``edge of stability'' \citep{cohen2020gradient}, adopting mean squared error is believed to make it easier to identify the effects of large learning rates. Besides, mean squared loss has also been demonstrated to be comparable or even better than cross-entropy loss in many classification tasks \citep{hui2020evaluation}.

We optimize the loss function \eqref{eq: loss two way multiple data sgd} by \emph{multi-pass} stochastic gradient descent (SGD), starting from some Gaussian weights, where each entry of $\mathbf{W}_{+1}^{(0)}$ and $\mathbf{W}_{-1}^{(0)}$ is sampled from $N(0,\sigma_0^2)$. 
The SGD goes for several epochs. 
In each epoch, we use each data $(\mathbf{x}_i,y_i)$ for exactly once, in the exact order of $(\mathbf{x}_1,y_1)\rightarrow(\mathbf{x}_2,y_2)\rightarrow\cdots\rightarrow(\mathbf{x}_n,y_n)$\footnote{We consider the same order for all epochs for the simplicity of analysis. Our analysis can be easily extended to multi-pass SGD with shuffling.}. 
Thus, the weights of the CNN are updated obeying the following rule,
\begin{align}
    \mathbf{w}_{j,r}^{(t+1)} &= \mathbf{w}_{j,r}^{(t)} -\eta \cdot \nabla_{\mathbf{w}_{j,r}}\ell(f(\mathbf{W}^{(t)},\mathbf{x}_{i_t}),y_{i_t}) \notag \\
    &= \mathbf{w}_{j,r}^{(t)} - \frac{j\eta}{m}\cdot \big(f(\mathbf{W}^{(t)},\mathbf{x}_{i_t}) - y_{i_t}\big)\cdot\sum_{p=1}^3 \sigma^{\prime}(\langle\mathbf{w}_{j,r}^{(t)},\mathbf{x}^{(p)}_{i_t}\rangle)\cdot \mathbf{x}_{i_t}^{(p)},\quad \forall t\geq 0,\label{eq: two way multiple data sgd}
\end{align}
for each $j\in\{\pm 1\}$ and $r\in[m]$, where $i_t = (t+1) \mod n$ and $\eta>0$ is the learning rate.

\paragraph{Generalization via signal (feature) learning.}
Our goal is to study the generalization property of the CNN \eqref{eq: cnn} trained by SGD \eqref{eq: two way multiple data sgd}.
Given a new testing data point $(\mathbf{x}^{\diamond}, y^{\diamond })$ sampled from the data generation model, we measure the generalization of the CNN by the correctness of its classification \citep{zhang2021understanding},
\begin{align}\label{eq: classification rate}
    \mathbb{E}[\mathbf{1}\{y^{\diamond}\cdot f(\mathbf{x}^{\diamond};\mathbf{W}_{\mathtt{sgd}}) > 0\}] = \mathbb{P}(y^{\diamond}\cdot f(\mathbf{x}^{\diamond};\mathbf{W}_{\mathtt{sgd}}) > 0),
\end{align}
where $\mathbf{W}_{\mathtt{sgd}}$ denotes the weights trained by SGD \eqref{eq: two way multiple data sgd}.

The way we investigate the generalization \eqref{eq: classification rate} is to track the process of signal (feature) learning.
More specifically, by the SGD updates \eqref{eq: two way multiple data sgd}, the weights $\mathbf{w}_{j,r}^{(t)}$ of the CNN is a linear combination of the initialization $\mathbf{w}_{j,r}^{(0)}$, the strong signal $j\cdot\mathbf{u}$, the weak signal $j\cdot\mathbf{v}$, and the noise vectors.
This motivates us to consider the following representation of the CNN weights, for $j\in\{\pm 1\}$ and $r\in[m]$,
\begin{align*}
    \mathbf{w}_{j,r}^{(t)} \approx \mathbf{w}_{j,r}^{(0)} + \frac{\langle\mathbf{w}_{j,r}^{(t)},j\mathbf{u}\rangle}{\|\mathbf{u}\|_2^2} \cdot j\mathbf{u} + \frac{\langle\mathbf{w}_{j,r}^{(t)},j\mathbf{v}\rangle}{\|\mathbf{v}\|_2^2} \cdot j\mathbf{v} + \text{noise parts}.
\end{align*}
The relative scales of these combination coefficients actually imply how the weights learn the strong signal $j\cdot\mathbf{u}$, the weak signal $j\cdot\mathbf{v}$, or memorize the noise, which determines how the CNN can generalize.

As is shown by \cite{cao2022benign}, the CNN tends to fit the training dataset utilizing patches with higher strength when trained by small learning rate gradient descents.
Therefore in that training regime, the CNN tends to fit the training data using the strong signal $y\cdot \mathbf{u}$, making less progress in learning the weak signal $y\cdot\mathbf{v}$.
Thus when it comes to the testing data which lack the strong signal patch $y\cdot\mathbf{u}$, the CNN trained in this manner would make a false classification.

On the contrary, our paper investigates the large learning rate regime, and suggests that the \emph{oscillation} happening during SGD training is beneficial for learning the weak signal, giving better generalization results.
So our main focus in the sequel would be studying the dynamics of the inner products 
\begin{align*}
    \langle\mathbf{w}_{j,r}^{(t)},j\mathbf{u}\rangle,\quad \langle\mathbf{w}_{j,r}^{(t)},j\mathbf{v}\rangle, \quad \langle\mathbf{w}_{j,r}^{(t)},\boldsymbol{\xi}\rangle,\quad \forall t\geq 0,
\end{align*}
where $\boldsymbol{\xi}$ is either $\boldsymbol{\xi}_i$ or $\widetilde{\boldsymbol{\xi}}_i$.
We show that by oscillating SGD training with large learning rates, $\langle\mathbf{w}_{j,r}^{(t)},j\mathbf{v}\rangle$ can be effectively learned to a relatively large scale compared to its initialization. 
This is further provably useful for the CNN to generalize to all new data points.

\section{Understand the Oscillation: Single Training Data Case}\label{sec: one data}

Before giving our main theory on large learning rate SGD training, let's first study a simplified setup where we consider only a \emph{single} training data point consisting only of a weak signal patch $y\cdot\mathbf{v}$ and a strong signal patch $y\cdot\mathbf{u}$, without the noise patch.
This setting helps to illustrate the key insights behind our main theory regarding the understanding of oscillation.
Without loss of generality, we denote the single training data as 
\begin{align*}
    (\mathbf{x},y) = \big((y\cdot\mathbf{u}, y\cdot\mathbf{v}), y\big),
\end{align*}
without the sample index $i$.
We can also simplify the CNN expression \eqref{eq: cnn} and the SGD updates \eqref{eq: two way multiple data sgd} to
\begin{align}
    f(\mathbf{x};\mathbf{W})& = \sum_{j\in\{\pm 1\}} jF_j(\mathbf{x};\mathbf{W}_j),\quad\text{with}\quad F_j(\mathbf{x};\mathbf{W}_j) = \frac{1}{m} \sum_{r\in[m]} \sigma(\langle\mathbf{w}_{j,r},y\mathbf{u}\rangle) + \sigma(\langle\mathbf{w}_{j,r},y\mathbf{v}\rangle), \label{eq: single data cnn}\\
    \mathbf{w}^{(t+1)}_{j,r} &= \mathbf{w}^{(t)}_{j,r} - \frac{\eta j}{m}\cdot\big(f(\mathbf{x};\mathbf{W}^{(t)}) - y\big)\cdot\Big(\sigma'(\langle\mathbf{w}_{j,r}^{(t)},y\mathbf{u}\rangle) \cdot y\mathbf{u} + \sigma'(\langle\mathbf{w}_{j,r}^{(t)},y\mathbf{v}\rangle) \cdot y\mathbf{v}\Big).\label{eq: single data sgd}
\end{align}
In such a simplified setup, we aim to explain that, when SGD training belongs to certain \emph{oscillation} regime, which typically occurs under large learning rate $\eta$, the CNN is guaranteed to make progress in learning the weak signal $y\cdot\mathbf{v}$. 
Here by oscillation, we mean that the values of the CNN $f(\mathbf{x};\mathbf{W}^{(t)})$ keep oscillating around the label $y$ during training.
This phenomenon greatly contrasts with known results for feature learning when gradient descent training converges smoothly under relatively small learning rates \citep{allen2020towards,cao2022benign}, which we are going to review in the following. 

\paragraph{Review: small learning rate training regime.} 
Firstly, we make a review of what may happen when using SGD updates \eqref{eq: single data sgd} with a small learning rate $\eta$.
The following proposition proves that in this case the CNN can \textbf{not} make much progress in learning the weak signal $y\cdot\mathbf{v}$.

\begin{proposition}[Small learning rate training: single training data (informal)]\label{prop:one-data small-lr}
    Under mild conditions on $(d,m,\sigma_0,\norm{\ub}_2,\norm{\vb}_2)$, if we choose the learning rate $\eta\le m/(6\|\ub\|^2_2)$ small enough, then with high probability, the training loss can smoothly converge, during which
    \begin{equation*}
        \max_{j\in\{\pm 1\},r\in[m]} \left\{ |\langle\mathbf{w}_{j,r}^{(t)},j\mathbf{v}\rangle|\right\}\leq\widetilde{\cO}\left(\sigma_0\|\mathbf{v}\|_2\right).
    \end{equation*}
\end{proposition}

Please refer to Appendix~\ref{sec: one data small lr} for more details and proofs of Proposition~\ref{prop:one-data small-lr}.
This shows that in the small learning rate training regime, the CNN only learns the weak signal $y\cdot \mathbf{v}$ to the same scale as its initialization. 
CNN trained in this manner may fail to generalize to testing data without strong features (substituted by a noise patch $\boldsymbol{\xi})$, because it would make predictions relying mainly on the random noise.
On the contrary, in the following we intuitively explain that under certain large learning rate training regime, the CNN can learn the weak signal  $y\cdot \mathbf{v}$ up to a constant level higher than its initialization.
Such a phenomenon is depicted in Figure~\ref{fig: two way one data} on an $8$-neuron CNN trained by SGD with $\eta = 0.1$ and $\eta = 0.5$ respectively.

\begin{figure}[!t]
    \centering
    \subfigure{
            \includegraphics[width=0.26\textwidth]{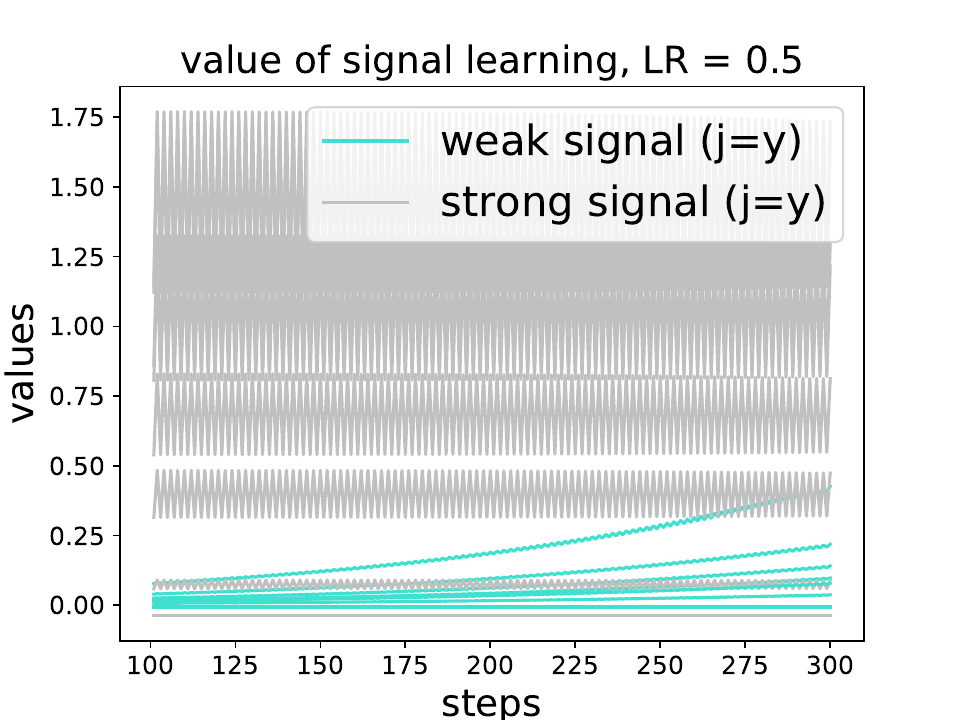}
            \label{fig: two way one data 1}
    }\hspace{-2em}
    \subfigure{
            \includegraphics[width=0.26\textwidth]{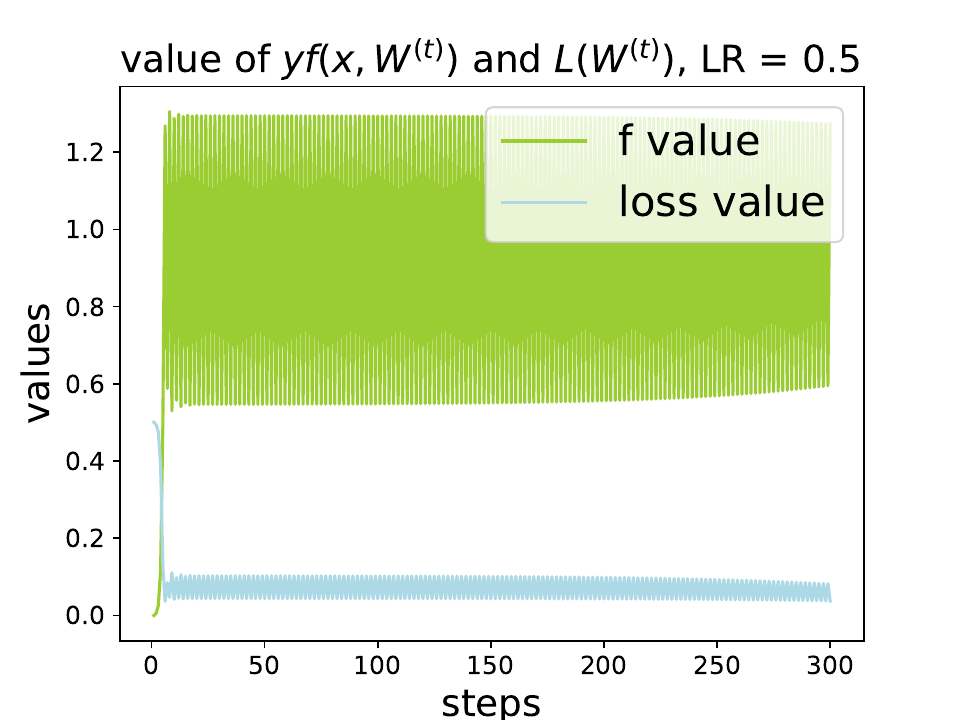}
            \label{fig: two way one data 2}
    }\hspace{-2em}
    \subfigure{
            \includegraphics[width=0.26\textwidth]{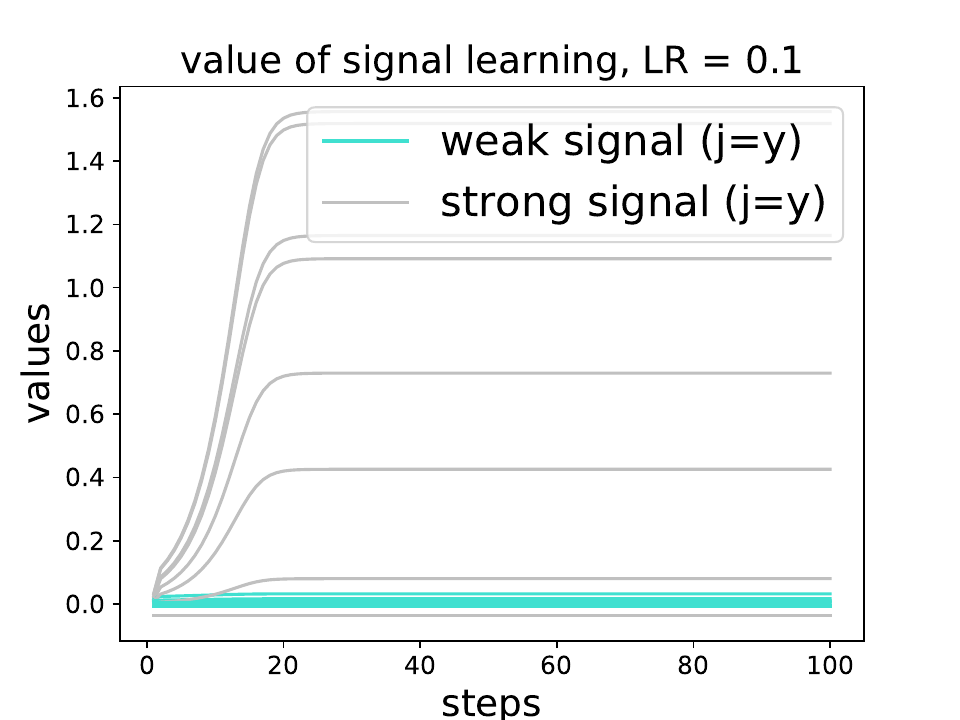}
            \label{fig: two way one data 3}
    }\hspace{-2em}
    \subfigure{
            \includegraphics[width=0.26\textwidth]{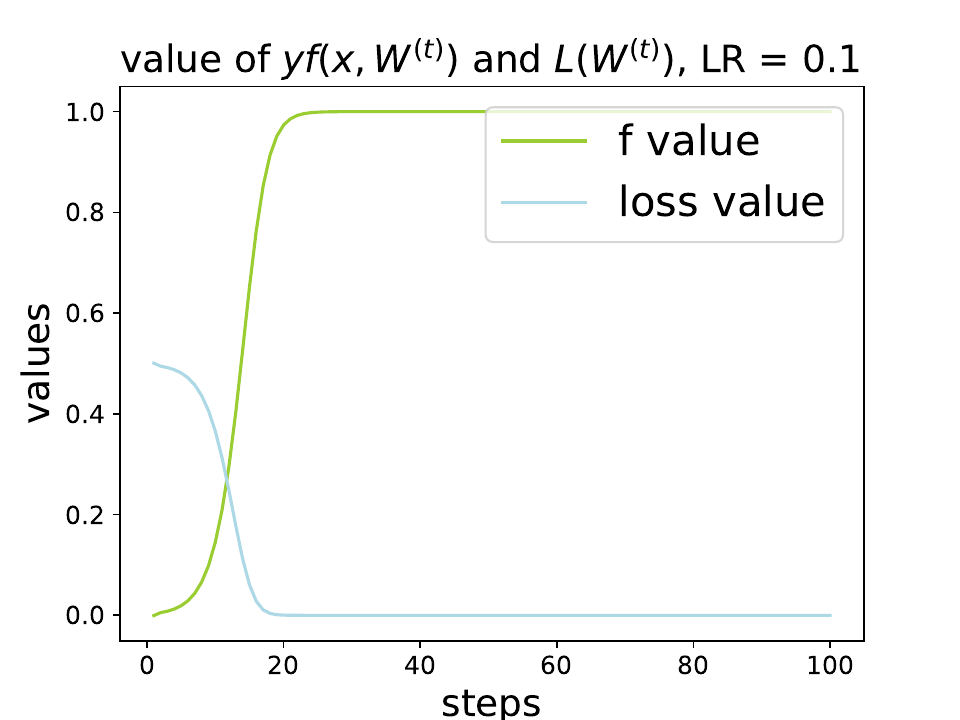}
            \label{fig: two way one data 4}
    }
    \vspace{-4mm}
    \caption{
    The progress of signal learning and the values of $yf(\mathbf{x};\mathbf{W}^{(t)})$ and $L(\mathbf{W}^{(t)})$ under different learning rates $\eta$. 
    The CNN in the first two figures is trained by SGD with $\eta = 0.5$ (large learning rate), while the CNN in the last two figures is trained by SGD with $\eta = 0.1$ (small learning rate).
    For signal learning (first and third figures), the gray lines depict the strong signal learning $\langle \mathbf{w}^{(t)}_{y,r},y\mathbf{u}\rangle$ by all neurons $r\in[m]$, and the light blue lines depict the weak signal learning $\langle \mathbf{w}^{(t)}_{y,r},y\mathbf{v}\rangle$ by all neurons $r\in[m]$. 
    As we can see, with a large learning rate, the value of CNN oscillates around $y$, and $\sum_t(1-yf(\mathbf{x};\mathbf{W}^{(t)}))$ is going to increase, which, as our theory indicates, incentivizes $\langle\mathbf{w}_{y,r}^{(t)},y\mathbf{v}\rangle$ to increase.
    In contrast, with a small learning rate, $\langle\mathbf{w}_{y,r}^{(t)},y\mathbf{v}\rangle$ would stay at the same scale as its initialization throughout the training process.}\label{fig: two way one data}
\end{figure}

\paragraph{Theoretical motivations: large learning rate regime and oscillation.} 
When using a large enough learning rate $\eta$ that exceeds the twice inversed smoothness, the weights of the CNN would keep oscillating, which makes the values of $f(\mathbf{x};\mathbf{W}^{(t)})$ fluctuate around $y$, or equivalently, $y\cdot f(\mathbf{x};\mathbf{W}^{(t)})$ fluctuate around $1$.
The key finding towards our theory is that the fluctuations of $y\cdot f(\mathbf{x};\mathbf{W}^{(t)})$ around $1$ would \emph{not} cancel with each other.
Instead, the oscillation accumulates linearly over time, as we could observe from Figure~\ref{fig: two way one data}.
This further serves as the engine driving the learning of the weak signal $y\cdot\mathbf{v}$.
In the sequel, we explain why the cancellation does not happen.

The core idea is that, with a reasonably large learning rate $\eta$, the CNN weights 
would be quickly enlarged from the learning of the strong signal $y\cdot\ub$ and then keep oscillating, but still stay well bounded. 
As a result of the SGD updates \eqref{eq: single data sgd}, the summation of the gradient terms is also well bounded.
More specifically, let's look carefully into the dynamics of learning the strong signal $y\cdot \mathbf{u}$.
For some time steps $t_0$, $t_1$, and certain neuron $r\in[m]$, it holds from \eqref{eq: single data sgd} that
\begin{align}\label{eq: intuition 1}
    \mathcal{O}(1) = \left|\langle\mathbf{w}_{y,r}^{(t_1+1)},y\mathbf{u} \rangle - \langle\mathbf{w}_{y,r}^{(t_0)},y\mathbf{u} \rangle\right| \approx  \Theta\Bigg(\sum_{s=t_0}^{t_1}\big(1 - yf(\mathbf{x};\mathbf{W}^{(s)})\big)\cdot  \langle\mathbf{w}_{y,r}^{(s)},y \mathbf{u} \rangle\Bigg).
\end{align}
Now we split the summation on the right hand side of \eqref{eq: intuition 1} into two parts: one part is $\mathcal{S}^+$ containing $s$ such that $yf(\mathbf{x};\mathbf{W}^{(s)})>1$ and the other part is $\mathcal{S}^-$ containing $s$ such that $yf(\mathbf{x};\mathbf{W}^{(s)})<1$.
It turns out that when the weak signal component of the CNN is relatively small compared with the strong signal component, the whole behavior of the CNN would be dominated by the dynamics of the strong signal component.
In other words, when $yf(\mathbf{x};\mathbf{W}^{(s)})>1$, the inner products $\langle\mathbf{w}_{y,r}^{(s)},y\mathbf{u}\rangle$ would also take a relatively large value.
Conversely, when $yf(\mathbf{x};\mathbf{W}^{(s)})<1$, the inner products $\langle\mathbf{w}_{y,r}^{(s)},y\mathbf{u}\rangle$ would also take a relatively small value.
Consequently, in view of \eqref{eq: intuition 1}, we can see that the total increases of $\langle\wb_{y,r}^{(s)},y\ub\rangle$ and decreases of $\langle\wb_{y,r}^{(s)},y\ub\rangle$
during the oscillation period are approximately balanced, i.e.,
\begin{align}\label{eq: intuition 2}
    \sum_{s\in\mathcal{S}^+}\underbrace{\big(yf(\mathbf{x};\mathbf{W}^{(s)}) - 1\big)}_{yf(\mathbf{x};\mathbf{W}^{(s)}) \textcolor{red}{>} 1}\cdot  \underbrace{\langle\mathbf{w}_{y,r}^{(s)},y\mathbf{u}\rangle}_{\text{relatively \textcolor{red}{large}}} \,\,\approx\,\,  \sum_{s\in\mathcal{S}^-}\underbrace{\big(1 - yf(\mathbf{x};\mathbf{W}^{(s)})\big)}_{{yf(\mathbf{x};\mathbf{W}^{(s)}) \textcolor{blue}{<}1}}\cdot  \underbrace{\langle\mathbf{w}_{y,r}^{(s)},y\mathbf{u}\rangle}_{\text{relatively \textcolor{blue}{small}}}.
\end{align}
Consequently, the summation of $1 - yf(\mathbf{x};\mathbf{W}^{(s)})$ over $s\in\mathcal{S}^-$ would take a larger value than the summation of $yf(\mathbf{x};\mathbf{W}^{(s)})-1$ over $s\in\mathcal{S}^+$.
This means that the whole summation 
\begin{align}\label{eq: intuition 3}
    \sum_{s\in\mathcal{S}^+\cup \mathcal{S}^-}\big(yf(\mathbf{x};\mathbf{W}^{(s)}) - 1\big) = \sum_{s\in\mathcal{S}^-}\big(1 - yf(\mathbf{x};\mathbf{W}^{(s)})\big) - \sum_{s\in\mathcal{S}^+}\big(yf(\mathbf{x};\mathbf{W}^{(s)} - 1)\big) \geq 0.
\end{align}
That is, the oscillation of $f(\mathbf{x};\mathbf{W}^{(s)})$ around the label $y$ over time does \textbf{not} tend to cancel with each other.
Instead, the summation of the fluctuations would have a determined sign.
Furthermore, if the CNN values are bounded away from the label by a uniform constant $\delta>0$ (i.e., the magnitude of the oscillation), we can further improve \eqref{eq: intuition 3} into
\begin{align}\label{eq: intuition 4}
    \sum_{s\in\mathcal{S}^+\cup \mathcal{S}^-}\big(yf(\mathbf{x};\mathbf{W}^{(s)}) - 1\big) = \Omega\big(\delta\cdot\max\{|\mathcal{S}^+|,|\mathcal{S}^-|\}\big) = \Omega\big(\delta\cdot(t_1 - t_0)\big).
\end{align}
This is the key observation behind our theory for studying the oscillating SGD. 
With \eqref{eq: intuition 4} in hand, we can further show that as long as the weak signal component of the CNN is still small, which means that the weak signal hasn't been well learned, the linear accumulation of the oscillation would incentivize the learning of the weak signal by a careful analysis of the updates \eqref{eq: single data sgd}.

\paragraph{Outcome of oscillation: effective weak signal learning.} 
Based on previous discussions, we can arrive at our main results for the simplified setup of this explanatory section: the oscillating SGD can indeed make progress in learning the weak signal $y\cdot\mathbf{v}$, which then helps the CNN to generalize to new data points which possibly lack the strong signal $y\cdot \mathbf{u}$. 
This is summarized in the following (informal) theorem.

\begin{theorem}[Large learning rate training: single training data case (informal)]\label{thm: one data}
    Under mild conditions on $(d,m,\sigma_0,\norm{\ub}_2,\norm{\vb}_2)$, if we choose the learning rate $\eta>m/(4\|\mathbf{u}\|_2^2)$ reasonably large such that the SGD training \eqref{eq: single data sgd} oscillates in the sense that $|yf(\mathbf{x};\mathbf{W}^{(t)})-1|\geq \delta$ for some constant $\delta>0$ and each $t\geq 0$, then with high probability there exists a $t^{\star}\leq T_{\max}$ with $T_{\max} \in \mathrm{poly}(d, m, \eta^{-1}, \delta^{-1}, \|\mathbf{u}\|_2^{-1}, \|\mathbf{v}\|_2^{-1})$ such that 
    \begin{align}
        \frac{1}{m}\sum_{r\in[m]}\sigma(\langle \mathbf{w}_{y,r}^{(t)},y\mathbf{v}\rangle) \geq \delta,\quad \forall t\geq t^{\star}.
    \end{align}
\end{theorem}


Please refer to Appendix~\ref{sec: one data case proof} for formal and detailed statement of Theorem \ref{thm: one data} and its proofs.
Theorem~\ref{thm: one data} shows that via oscillating SGD training, the CNN learns the weak signal $y\cdot\mathbf{v}$ up to a constant scale of $\delta$, which is typically much larger than the scale of its initialization, since
\begin{align}\label{eq: illustration delta}
    \frac{1}{m}\sum_{r\in[m]}\sigma(\langle\mathbf{w}_{y,r}^{(0)},y\mathbf{v}\rangle)\leq \widetilde{\mathcal{O}}(\sigma_0^2\|\mathbf{v}\|_2^2)\ll  \delta \leq  \frac{1}{m}\sum_{r\in[m]}\sigma(\langle \mathbf{w}_{y,r}^{(t^{\star})},y\mathbf{v}\rangle),
\end{align}
whenever the initialization of the CNN is small as in practice.
We remark that here we only consider neurons with $j=y$ since the CNN is trained only on a single data with label $y$.



\paragraph{Implications to the simple signal-noise model.} 
Our findings can also be adapted to explain the data model considered by \cite{cao2022benign}. 
In that case, each data point $\mathbf{x} = (y\cdot\boldsymbol{\mu},\boldsymbol{\xi})$ consists of a single signal patch $y\cdot \boldsymbol{\mu}$ and a single Gaussian noise patch $\boldsymbol{\xi}$. 
A trained neural network can generalize to new data points only when it learns the common signal vector $\boldsymbol{\mu}$.

As is shown by \cite{cao2022benign}, in small learning rate training regime, if the data model has a low \emph{signal-to-noise ratio} (SNR), that is, the strength of the noise patch is relatively stronger than the the strength of the signal patch, then overfitting the training data would result in poor generalization (harmful overfitting).
That is because the neural network would memorize the noise patch quickly so as to fit the training data, and consequently the signal patch is not well learned. 
In contrast, we could show that under the oscillating SGD training regime, the signal can also be well learned even with a low SNR.
The mechanism behind this is still that the oscillation during training would accumulate and incentivize the NN towards signal learning.

\section{Main Theory}\label{sec: main theory}

\newcommand{\ft}{\tilde{f}}

In this section, we present our main theory on benign oscillation of SGD with large learning rates based on the general setups introduced in Section \ref{sec: problem setup}.
We first introduce the key conditions and assumptions required by our theory in Section~\ref{subsec: keu conditions and assumptions}.
Then we present our theoretical results in Section~\ref{subsec: main results}, where we also compare large learning rate oscillating training to small learning rate training.  
Finally in Section~\ref{sec: experiments} we demonstrate our theoretical findings via numerical experiments.


\subsection{Key Conditions and Assumptions}\label{subsec: keu conditions and assumptions}

Before presenting our theoretical results, we outline the key conditions and assumptions needed on the model and the training dynamics. 
Firstly, our results are based upon the following conditions on the initialization scale $\sigma_0$, dimension $d$, datasize $n$, CNN width $m$, and signal and noise strength $\|\mathbf{u}\|_2$, $\|\mathbf{v}\|_2$,  and $\sigma_p$.

\begin{assumption}[Conditions on hyperparameters]\label{ass: conditions}
    Suppose that the following conditions hold: (i) the CNN weight initialization scale $\sigma_0 = \widetilde{\Theta}(\max\{\|\mathbf{u}\|_2,\|\mathbf{v}\|_2,\sigma_p\sqrt{d}\}^{-1}\cdot d^{-1/2})$; (ii) the dimension $d = \Omega(n^2,\mathrm{polylog}(m))$; (iii) the signal strength $\|\mathbf{v}\|_2\leq 0.01\cdot\|\mathbf{u}\|_2$, $\|\mathbf{u}\|_2^{-2} + \|\mathbf{v}\|_2^{-2}\leq \widetilde{\mathcal{O}}(n(\sigma_p^2d)^{-1})$.
    (iv) the learning rate $m/(4\|\mathbf{u}\|_2^2)\leq \eta \leq 2m/(5\|\mathbf{u}\|_2^2)$.
    (v) the weak data fraction $\rho \leq c$ for some small constant $c$.
\end{assumption}

We explain these conditions in Assumption~\ref{ass: conditions} one by one.
The conditions on the initialization scale $\sigma_0$ and the learning rate $\eta$ are to ensure that the whole training process is well bounded while oscillates (rather than converging smoothly).
Then the condition on the dimension $d$ puts us in the regime of high dimension for which independent Gaussian noise has small correlations. 
The conditions on the signal strength separate the strong signal $\mathbf{u}$ from the weak signal $\mathbf{v}$ by $\ell_2$-norms. 
Besides, we ensure that the data are not too noisy via restricting the variance of the Gaussian noise.
Finally, the condition on the fraction $\rho$ of weak data is for technical considerations in analyzing the training process, and \emph{can be relaxed for testing data population}.

The next assumption is on the training dynamics, which requires that the SGD oscillates.
For simplicity, we denote the index set of the weak training data points lacking the strong feature patch as $\mathcal{W}$.

\begin{assumption}[Oscillating SGD]\label{ass: oscillation}
    We assume that there exists some constant $\delta \in (0.2, 0.8)$, such that $|y_{i_t} f(\mathbf{x}_{i_t} ; \mathbf{W}^{(t)})- 1| \geq \delta$ holds for any $t\geq 0$ such that $i_t\notin\mathcal{W}$.
\end{assumption}

Through Assumption \ref{ass: oscillation}, we require that the value of $yf(\mathbf{x};\mathbf{W}^{(t)})$ on data points with strong features oscillates around the desired value, $1$, by a scale of $\delta \in(0.2,0.8)$, i.e., the magnitude of the
oscillation is at least $\delta$.
Here the range for $\delta$ is only for technical considerations to simplify the theoretical analyses.
 

We remark that in general the dynamics of the training process could be quite subtle when oscillation happens, and there exist other more complicated patterns of oscillations if one deliberately chooses a specific learning rate $\eta$. 
Our work focuses on a relatively simple but common pattern of oscillation.
It turns out that under the oscillation pattern in Assumption \ref{ass: oscillation}, we can show the benefits of oscillation on the generalization properties of the CNN.
Actually, we could also extend our analysis to a weakened version of Assumption~\ref{ass: oscillation} that the time average of $|y_{i_t} f(\mathbf{x}_{i_t}; \mathbf{W}^{(t)})- 1|$ is larger than $\delta$.

It is notable that Assumption~\ref{ass: oscillation} implicitly requires that the learning rate $\eta$ should be chosen properly. A large learning rate forces the training trajectories to escape from the regular region, while a small learning rate shall result in smooth convergence. In both cases the phenomenon described in Assumption~\ref{ass: oscillation} does not happen. We also remind readers that the $\eta$ condition in Assumption~\ref{ass: conditions} is only sufficient for the regularities such as boundedness and sign stability. Readers can refer to Appendix~\ref{subsec discuss} for a discussion of the necessary conditions on the learning rate $\eta$ implied by Assumption~\ref{ass: oscillation}.

Finally, we remark that we only assume the oscillation on strong data, since intuitively on weak data the CNN fits the label via the weak features and the noise (both have smaller strength than strong features) and may converge slower and more smoothly.  
Please see Section~\ref{sec: experiments} for experimental evidence.

\subsection{Main Theoretical Results}\label{subsec: main results}

Our main results are parallel to the single data setup in Section~\ref{sec: one data}: under previous conditions and assumptions on the hyperparameters and the training dynamics, the CNN can make enough progress in learning the weak signal $\mathbf{v}$ thanks to the oscillation happening during training. 
We refer to Appendix~\ref{sec: multiple data case proof} for a detailed proof of the following results.

\begin{theorem}[Weak signal learning: oscillating training with large learning rate]\label{thm: main}
    Under Assumptions \ref{ass: conditions} and \ref{ass: oscillation}, w.p. at least $1-1/\mathrm{poly}(d)$, there exists $t^{\star} \leq \mathrm{poly}(d,m,n,\delta^{-1},\eta^{-1},\sigma_p^{-1}, \|\mathbf{u}\|_2^{-1},\|\mathbf{v}\|_2^{-1})$ such that
    \begin{align}
       \max_{j\in\{\pm 1\}} \left\{ \frac{1}{m}\sum_{r\in[m]}\sigma(\langle\mathbf{w}_{j,r}^{(t^{\star})},j\mathbf{v}\rangle) - \frac{1}{m}\sum_{r\in[m]}\sigma(\langle\mathbf{w}_{-j,r}^{(t^{\star})},j\mathbf{v}\rangle) \right\}\geq \frac{\delta}{4}.
    \end{align}
\end{theorem}

\begin{remark}[Interpretations of Theorem~\ref{thm: main}]
    This theorem concludes that by large learning rate oscillating SGD training, the CNN learns the weak signal up to the scale of $\delta$ at certain time $t^{\star}$ and $j\in\{\pm 1\}$. 
    We can refine this result to any step after $t\geq t^{\star}$ and each $j\in\{\pm 1\}$ with some more intricate analysis regarding the signal learning dynamics, which we have done on the single training data case where the result holds for each step $t\geq t^{\star}$ (see Theorem~\ref{thm: one data}).
    Despite that, Theorem~\ref{thm: main} has already shown the power of large learning rate training for learning the weak signal in the multiple data setup, since as is shown in the following, the weak signal learning keeps at the same scale as its initialization under the small learning rate training regime.
\end{remark}



\paragraph{Small learning rate training regime.}
In contrast, under the small learning rate regime, the CNN would not learn the weak features, which is the following proposition with more details and proofs in Appendix~\ref{sec:multiple-data small lr}.

\begin{proposition}[Small learning rate training]\label{prop:multiple-data small-lr}
    Under Assumption \ref{cond:model_params_multiple_data small-lr} on $(d,m,n,\sigma_0,\norm{\ub}_2,\norm{\vb}_2,\sigma_p)$, if we choose learning rate $\eta\le m/(6\|\ub\|^2_2)$ small enough, then w.p. at least $1-1/\mathrm{poly}(d)$, the training loss can smoothly converge, during which
    \begin{align}
        \max_{j\in\{\pm 1\},r\in[m]} \left\{|\langle\mathbf{w}_{j,r}^{(t)},j\mathbf{v}\rangle|\right\}\le \widetilde{\cO}\left(\sigma_0\|\mathbf{v}\|_2\right).
    \end{align}
\end{proposition}

\paragraph{Division of generalization.} 
Suppose we are given a new testing data point $(\mathbf{x}^{\diamond}, y^{\diamond })$ with an input $\xb^\diamond=(\widetilde{\boldsymbol{\xi}}^\diamond, y^\diamond\vb,\boldsymbol{\xi}^\diamond)$ only consisting of the weak signal $y^\diamond\cdot\vb$. 
Then a reliable prediction can only count on utilizing the weak signal $y^\diamond\cdot\vb$. 
To see this and to be more specific, in the regime specified by Theorem~\ref{thm: main},
\begin{align}
    y^{\diamond}f(\mathbf{x}^{\diamond};\mathbf{W}^{(t^{\star})}) &= \underbrace{\frac{1}{m}\sum_{r\in[m]}\sigma(\langle\mathbf{w}_{y^{\diamond},r}^{(t^{\star})},y^{\diamond}\mathbf{v}\rangle) - \frac{1}{m}\sum_{r\in[m]}\sigma(\langle\mathbf{w}_{-y^{\diamond},r}^{(t^{\star})},y^{\diamond}\mathbf{v}\rangle)}_{\text{Weak signal component } \geq \delta / 4} \\
    &\qquad +\underbrace{\sum_{j\in\{\pm 1\}}\frac{j}{m}\sum_{r\in[m]} \sigma(\langle\mathbf{w}_{j,r}^{(t^{\star})},\boldsymbol{\xi}^{\diamond}\rangle) + \sigma(\langle\mathbf{w}_{j,r}^{(t^{\star})},\widetilde{\boldsymbol{\xi}}^{\diamond}\rangle)}_{\text{Noise component }\leq o(1)},
\end{align}
there holds\footnote{the noise patch term is of order $\widetilde{\mathcal{O}}(\sigma_0^2\sigma_p^2d)$ which is $\widetilde{\mathcal{O}}(d^{-1})$ under Assumption~\ref{ass: conditions}.} $y^{\diamond}\cdot f(\mathbf{x}^{\diamond};\mathbf{W}^{(t^{\star})})\ge\delta/4-o(1)>0$ which corresponds to correct prediction almost certainly. In contrast, when applying the small learning rate, as specified by Proposition~\ref{prop:multiple-data small-lr}, the trained NN fails to take advantage of the weak signal $y^\diamond\cdot\vb$ from data $\xb^{\diamond}$. 
Therefore, it would be likely to make the prediction based on a random guess (the randomness stems from the random initialization and the noise $\boldsymbol{\xi}^\diamond,\widetilde{\boldsymbol{\xi}}^\diamond$). 
Consequently, noting that the weak data takes up a $\rho$ fraction of the data distribution (see our data model in Section \ref{sec: problem setup}), SGD with large learning rates would achieve a $\Theta(\rho)$ higher test accuracy than SGD with small learning rates, demonstrating the benefit of oscillation and large learning training in terms of the generalization ability.




\begin{figure}[!t]
    \centering
    \subfigure{
            \includegraphics[width=0.26\textwidth]{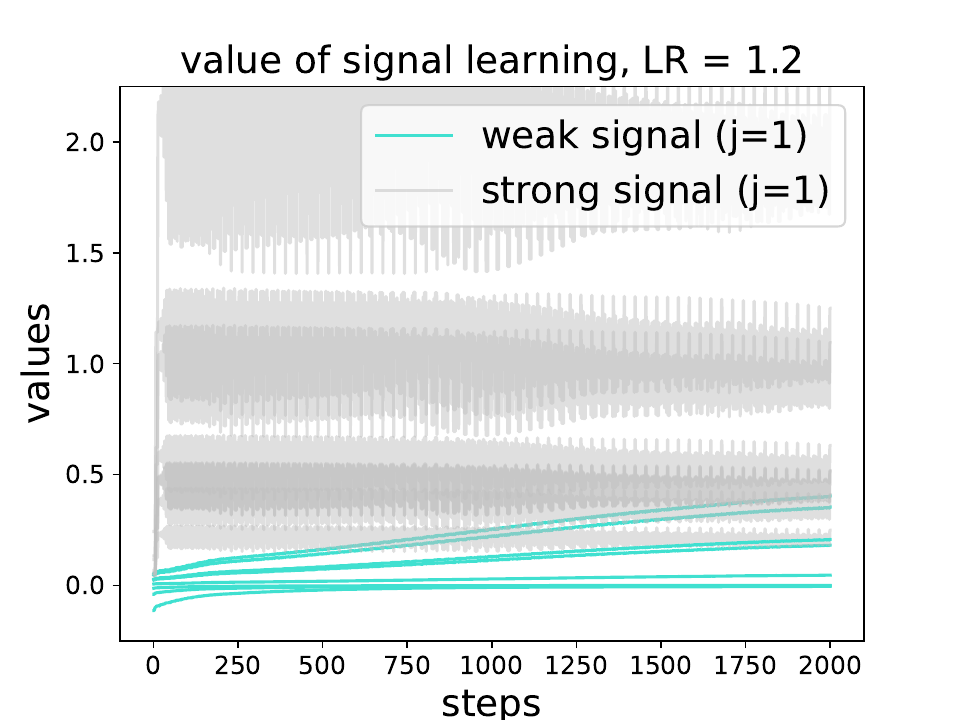}
    }\hspace{-2em}
    \subfigure{
            \includegraphics[width=0.26\textwidth]{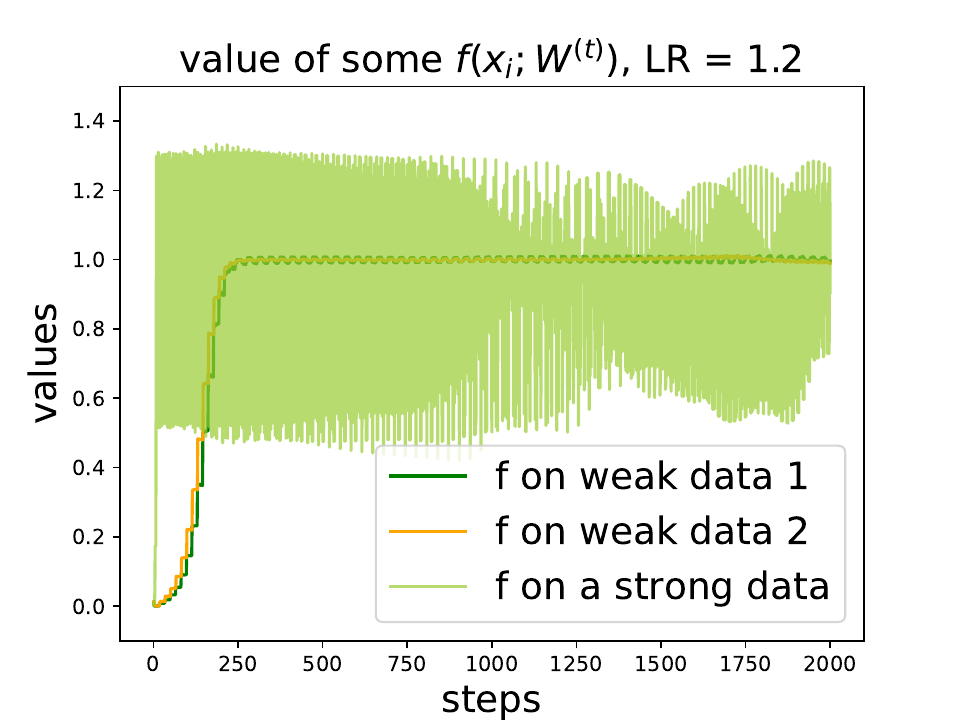}
    }\hspace{-2em}
    \subfigure{
            \includegraphics[width=0.26\textwidth]{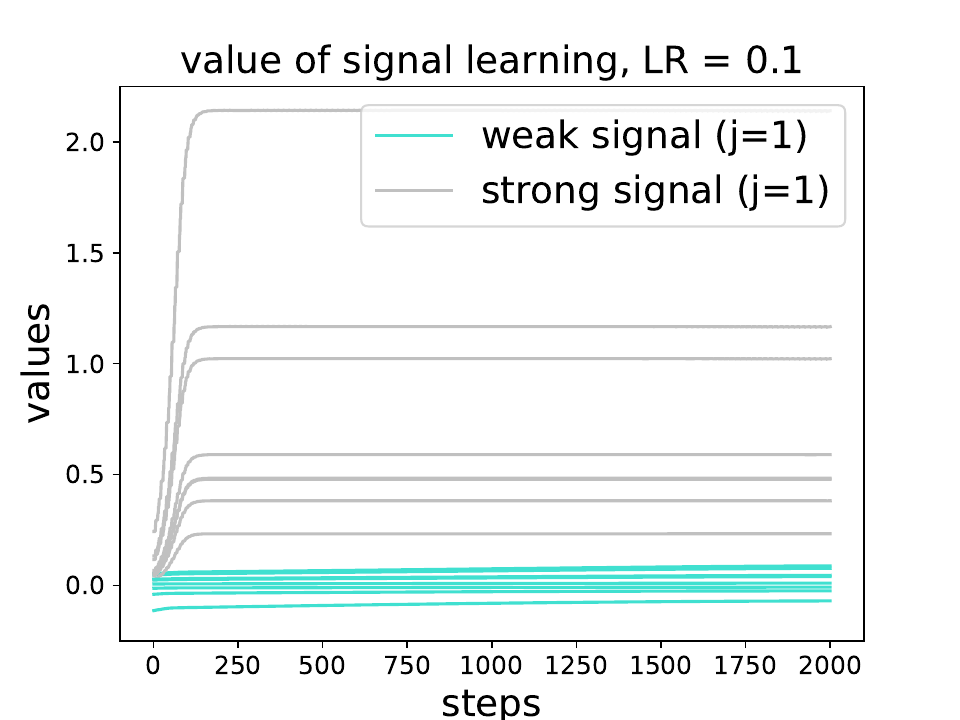}
    }\hspace{-2em}
    \subfigure{
            \includegraphics[width=0.26\textwidth]{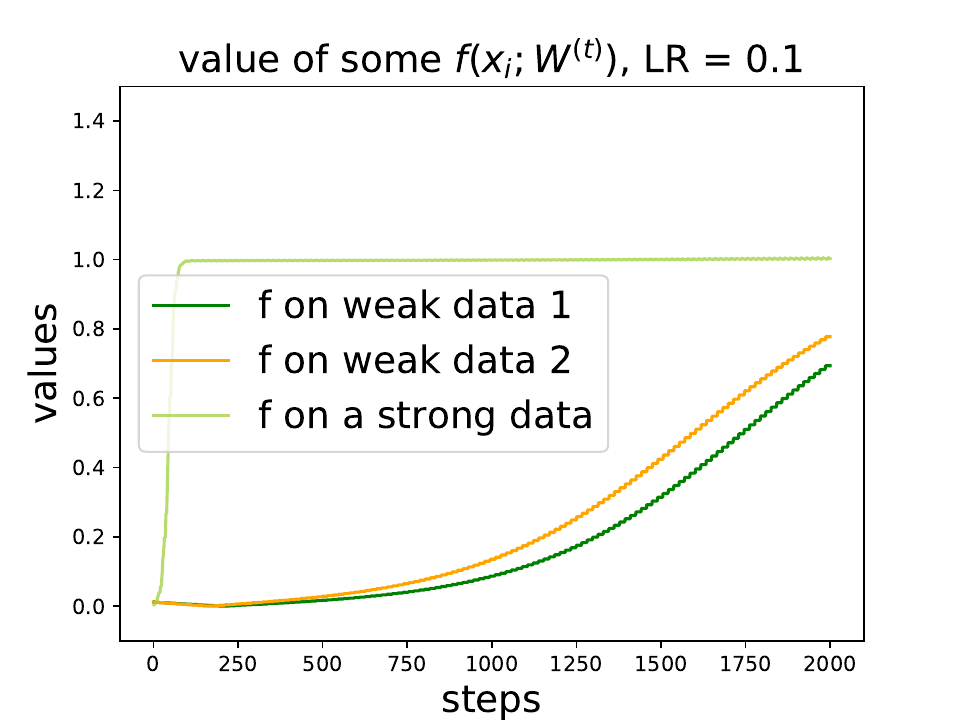}
    }
    \caption{
    The dynamics of signal learning under a large learning rate $\eta_{\mathrm{large}} = 1.2$ and a small learning rate  $\eta_{\mathrm{small}} = 0.1$. The values of signal learning are obtained by characterizing the inner products $\langle \wb_{j,r}^{(t)},j\ub\rangle$ and $\langle \wb_{j,r}^{(t)},j\vb\rangle$. 
    When using the large learning rate, strong signal learning as well as the NN outputs will oscillate, during which weak signal will be gradually learned. 
    When using the small learning rate, strong signal learning will converge quickly, and the weak signal learning will stay at the same scale as its initialization.
    }\label{fig: two way multiple data}
\end{figure}

\subsection{Numerical Experiments}\label{sec: experiments}

In this section, we conduct numerical experiments to demonstrate our findings on ``benign oscillation". 
We follow the same data generation model and optimization algorithm as we described as Section~\ref{sec: problem setup}.
Specifically, we consider a dataset with $n = 16$ and $|\mathcal{W}| = 2$, that is, $\rho \approx 0.125$.
The dimension is $d = 64$, and the number of neurons for each direction $j$ is $m = 8$.
We generate the data with strong signal $\|\mathbf{u}\|_2 = 2$, weak signal $\|\mathbf{v}\|_2 = 0.4$, and noise $\|\boldsymbol{\xi}\|_2\approx\sigma_pd^{1/2} = 0.8$.

\paragraph{Weak signal learning.} We run the SGD \eqref{eq: two way multiple data sgd} to train the CNN with two different scale of learning rates: a large learning rate $\eta_{\mathrm{large}} = 1.2$, a small learning rate $\eta_{\mathrm{small}} = 0.1$.
We plot the dynamics of signal learning for each neuron $r\in[m]$ from these two training regimes in Figure~\ref{fig: two way multiple data}.
As we can see from Figure~\ref{fig: two way multiple data}, with large learning rate SGD training, the CNN can effectively learn the weak signal to a scale much larger than the initialization.
On the contrary, by small learning rate SGD training, the CNN does not learn the weak signal since it just remains at the same level as the initialization.
This demonstrates our theory in Section~\ref{subsec: main results}.

Furthermore, in Figure~\ref{fig: two way multiple data}, we plot the values of $y_{i} f(\mathbf{x}_i;\mathbf{W}^{(t)})$ on certain data points $i\in[n]$ and the value of $L(\mathbf{W}^{(t)})$ for those two training regimes.
Specifically, we plot the values of $y_i f(\mathbf{x}_i;\mathbf{W}^{(t)})$ on a strong data $i\notin \mathcal{W}$ (randomly sampled) and the values of $y_i f(\mathbf{x}_i;\mathbf{W}^{(t)})$ on the two weak data $i\in\mathcal{W}$.
As we can observe from the large learning rate training case, the values of $f$ on the strong data oscillate while the values of $f$ on the weak data do not. 
This matches our Assumption~\ref{ass: oscillation} that the oscillation in $f$ value only happens for strong data.
Also, for the small learning rate training case, the values of $f$ on the weak data converge slower than those on the strong data.
This is because on the weak data the CNN mainly  utilizes the noise to fit the target which is of lower strength than the strong signals on strong data. 
But still, the noise outweights the weak signals and consequently the CNN makes no progress in learning the weak signal if trained smoothly.

\paragraph{Generalization properties.}
Finally, we test the CNN trained by two different learning rates on the new testing data generated in the same way as the training data. 
The testing data size is $32$ with $4$ weak data points. 
We repeat the testing evaluation over $5$ random seeds and take the average.
The result is that for the CNN trained by $\eta_{\mathrm{large}}$ the testing accuracy is $99.38\%$, and for the CNN trained by $\eta_{\mathrm{small}}$ the testing accuracy is $93.75\%$, matching our theoretical insights that large learning rate training benefits NN generalization.
For the CNN trained by $\eta_{\mathrm{small}}$, it misclassifies certain weak data points. 
As we previously discussed, on the data without the strong signal, the CNN approximately uses a random guess.

\section{Conclusions}

This work theoretically investigated the NN training with large learning rates and established a theoretical framework to understand the oscillation phenomenon. 
We revealed the benefits of oscillation training to the NN generalization,  which we summarize as the phenomenon of ``\emph{benign oscillation}". 
Our theory demystified the phenomenon based on a feature learning perspective and showed that the oscillation can drive the learning of weak but important patterns from data that are crucial to generalization. 
Our theory sheds light on the understanding of large learning rate NN training and provided useful guidance towards the optimization analysis when smooth convergence is not guaranteed.

\bibliography{reference}

\begin{thebibliography}{32}
\expandafter\ifx\csname natexlab\endcsname\relax\def\natexlab#1{#1}\fi
\expandafter\ifx\csname url\endcsname\relax
  \def\url#1{\texttt{#1}}\fi
\expandafter\ifx\csname urlprefix\endcsname\relax\def\urlprefix{URL }\fi

\bibitem[{Ahn et~al.(2022)Ahn, Zhang and Sra}]{ahn2022understanding}
\textsc{Ahn, K.}, \textsc{Zhang, J.} and \textsc{Sra, S.} (2022).
\newblock Understanding the unstable convergence of gradient descent.
\newblock In \textit{International Conference on Machine Learning}. PMLR.

\bibitem[{Allen-Zhu and Li(2022)}]{allen2020towards}
\textsc{Allen-Zhu, Z.} and \textsc{Li, Y.} (2022).
\newblock Towards understanding ensemble, knowledge distillation and
  self-distillation in deep learning.
\newblock In \textit{The Eleventh International Conference on Learning
  Representations}.

\bibitem[{Andriushchenko et~al.(2023)Andriushchenko, Varre, Pillaud-Vivien and
  Flammarion}]{andriushchenko2023sgd}
\textsc{Andriushchenko, M.}, \textsc{Varre, A.~V.}, \textsc{Pillaud-Vivien, L.}
  and \textsc{Flammarion, N.} (2023).
\newblock Sgd with large step sizes learns sparse features.
\newblock In \textit{International Conference on Machine Learning}. PMLR.

\bibitem[{Arora et~al.(2022)Arora, Li and Panigrahi}]{arora2022understanding}
\textsc{Arora, S.}, \textsc{Li, Z.} and \textsc{Panigrahi, A.} (2022).
\newblock Understanding gradient descent on the edge of stability in deep
  learning.
\newblock In \textit{International Conference on Machine Learning}. PMLR.

\bibitem[{Cao et~al.(2022)Cao, Chen, Belkin and Gu}]{cao2022benign}
\textsc{Cao, Y.}, \textsc{Chen, Z.}, \textsc{Belkin, M.} and \textsc{Gu, Q.}
  (2022).
\newblock Benign overfitting in two-layer convolutional neural networks.
\newblock \textit{Advances in neural information processing systems}
  \textbf{35} 25237--25250.

\bibitem[{Chen and Bruna(2022)}]{chen2022gradient}
\textsc{Chen, L.} and \textsc{Bruna, J.} (2022).
\newblock On gradient descent convergence beyond the edge of stability.
\newblock \textit{arXiv preprint arXiv:2206.04172} .

\bibitem[{Chen et~al.(2022)Chen, Deng, Wu, Gu and Li}]{chen2022towards}
\textsc{Chen, Z.}, \textsc{Deng, Y.}, \textsc{Wu, Y.}, \textsc{Gu, Q.} and
  \textsc{Li, Y.} (2022).
\newblock Towards understanding the mixture-of-experts layer in deep learning.
\newblock \textit{Advances in neural information processing systems}
  \textbf{35} 23049--23062.

\bibitem[{Cohen et~al.(2020)Cohen, Kaur, Li, Kolter and
  Talwalkar}]{cohen2020gradient}
\textsc{Cohen, J.}, \textsc{Kaur, S.}, \textsc{Li, Y.}, \textsc{Kolter, J.~Z.}
  and \textsc{Talwalkar, A.} (2020).
\newblock Gradient descent on neural networks typically occurs at the edge of
  stability.
\newblock In \textit{International Conference on Learning Representations}.

\bibitem[{Cohen et~al.(2022)Cohen, Ghorbani, Krishnan, Agarwal, Medapati,
  Badura, Suo, Cardoze, Nado, Dahl et~al.}]{cohen2022adaptive}
\textsc{Cohen, J.~M.}, \textsc{Ghorbani, B.}, \textsc{Krishnan, S.},
  \textsc{Agarwal, N.}, \textsc{Medapati, S.}, \textsc{Badura, M.},
  \textsc{Suo, D.}, \textsc{Cardoze, D.}, \textsc{Nado, Z.}, \textsc{Dahl,
  G.~E.} \textsc{et~al.} (2022).
\newblock Adaptive gradient methods at the edge of stability.
\newblock \textit{arXiv preprint arXiv:2207.14484} .

\bibitem[{Damian et~al.(2022)Damian, Nichani and Lee}]{damian2021label}
\textsc{Damian, A.}, \textsc{Nichani, E.} and \textsc{Lee, J.~D.} (2022).
\newblock Self-stabilization: The implicit bias of gradient descent at the edge
  of stability.
\newblock In \textit{The Eleventh International Conference on Learning
  Representations}.

\bibitem[{Frankle et~al.(2019)Frankle, Schwab and Morcos}]{frankle2019early}
\textsc{Frankle, J.}, \textsc{Schwab, D.~J.} and \textsc{Morcos, A.~S.} (2019).
\newblock The early phase of neural network training.
\newblock In \textit{International Conference on Learning Representations}.

\bibitem[{He et~al.(2016)He, Zhang, Ren and Sun}]{he2016deep}
\textsc{He, K.}, \textsc{Zhang, X.}, \textsc{Ren, S.} and \textsc{Sun, J.}
  (2016).
\newblock Deep residual learning for image recognition.
\newblock In \textit{Proceedings of the IEEE conference on computer vision and
  pattern recognition}.

\bibitem[{Huang et~al.(2023)Huang, Cao, Wang, Cao and Suzuki}]{huang2023graph}
\textsc{Huang, W.}, \textsc{Cao, Y.}, \textsc{Wang, H.}, \textsc{Cao, X.} and
  \textsc{Suzuki, T.} (2023).
\newblock Graph neural networks provably benefit from structural information: A
  feature learning perspective.
\newblock \textit{arXiv preprint arXiv:2306.13926} .

\bibitem[{Hui(2020)}]{hui2020evaluation}
\textsc{Hui, L.} (2020).
\newblock Evaluation of neural architectures trained with square loss vs
  cross-entropy in classification tasks.
\newblock In \textit{The Ninth International Conference on Learning
  Representations (ICLR 2021)}.

\bibitem[{Jastrzebski et~al.(2021)Jastrzebski, Arpit, Astrand, Kerg, Wang,
  Xiong, Socher, Cho and Geras}]{jastrzebski2021catastrophic}
\textsc{Jastrzebski, S.}, \textsc{Arpit, D.}, \textsc{Astrand, O.},
  \textsc{Kerg, G.~B.}, \textsc{Wang, H.}, \textsc{Xiong, C.}, \textsc{Socher,
  R.}, \textsc{Cho, K.} and \textsc{Geras, K.~J.} (2021).
\newblock Catastrophic fisher explosion: Early phase fisher matrix impacts
  generalization.
\newblock In \textit{International Conference on Machine Learning}. PMLR.

\bibitem[{Kalra and Barkeshli(2023)}]{kalra2023phase}
\textsc{Kalra, D.~S.} and \textsc{Barkeshli, M.} (2023).
\newblock Phase diagram of training dynamics in deep neural networks: effect of
  learning rate, depth, and width.
\newblock \textit{arXiv preprint arXiv:2302.12250} .

\bibitem[{Kaur et~al.(2023)Kaur, Cohen and Lipton}]{kaur2023maximum}
\textsc{Kaur, S.}, \textsc{Cohen, J.} and \textsc{Lipton, Z.~C.} (2023).
\newblock On the maximum hessian eigenvalue and generalization.
\newblock In \textit{Proceedings on}. PMLR.

\bibitem[{Lewkowycz et~al.(2020)Lewkowycz, Bahri, Dyer, Sohl-Dickstein and
  Gur-Ari}]{lewkowycz2020large}
\textsc{Lewkowycz, A.}, \textsc{Bahri, Y.}, \textsc{Dyer, E.},
  \textsc{Sohl-Dickstein, J.} and \textsc{Gur-Ari, G.} (2020).
\newblock The large learning rate phase of deep learning: the catapult
  mechanism.
\newblock \textit{arXiv preprint arXiv:2003.02218} .

\bibitem[{Li et~al.(2019)Li, Wei and Ma}]{li2019towards}
\textsc{Li, Y.}, \textsc{Wei, C.} and \textsc{Ma, T.} (2019).
\newblock Towards explaining the regularization effect of initial large
  learning rate in training neural networks.
\newblock \textit{Advances in Neural Information Processing Systems}
  \textbf{32}.

\bibitem[{Smith and Topin(2019)}]{smith2019super}
\textsc{Smith, L.~N.} and \textsc{Topin, N.} (2019).
\newblock Super-convergence: Very fast training of neural networks using large
  learning rates.
\newblock In \textit{Artificial intelligence and machine learning for
  multi-domain operations applications}, vol. 11006. SPIE.

\bibitem[{Wang et~al.(2022)Wang, Li and Li}]{wang2022analyzing}
\textsc{Wang, Z.}, \textsc{Li, Z.} and \textsc{Li, J.} (2022).
\newblock Analyzing sharpness along gd trajectory: Progressive sharpening and
  edge of stability.
\newblock \textit{Advances in Neural Information Processing Systems}
  \textbf{35} 9983--9994.

\bibitem[{Wen and Li(2021)}]{wen2021toward}
\textsc{Wen, Z.} and \textsc{Li, Y.} (2021).
\newblock Toward understanding the feature learning process of self-supervised
  contrastive learning.
\newblock In \textit{International Conference on Machine Learning}. PMLR.

\bibitem[{Wu et~al.(2023)Wu, Braverman and Lee}]{wu2023implicit}
\textsc{Wu, J.}, \textsc{Braverman, V.} and \textsc{Lee, J.~D.} (2023).
\newblock Implicit bias of gradient descent for logistic regression at the edge
  of stability.
\newblock \textit{arXiv preprint arXiv:2305.11788} .

\bibitem[{Wu et~al.(2021)Wu, Zou, Braverman and Gu}]{wu2021direction}
\textsc{Wu, J.}, \textsc{Zou, D.}, \textsc{Braverman, V.} and \textsc{Gu, Q.}
  (2021).
\newblock Direction matters: On the implicit bias of stochastic gradient
  descent with moderate learning rate.
\newblock In \textit{International Conference on Learning Representation
  (ICLR)}.

\bibitem[{Xing et~al.(2018)Xing, Arpit, Tsirigotis and Bengio}]{xing2018walk}
\textsc{Xing, C.}, \textsc{Arpit, D.}, \textsc{Tsirigotis, C.} and
  \textsc{Bengio, Y.} (2018).
\newblock A walk with sgd.
\newblock \textit{arXiv preprint arXiv:1802.08770} .

\bibitem[{Yang et~al.(2023)Yang, Tang and Tu}]{yang2023stochastic}
\textsc{Yang, N.}, \textsc{Tang, C.} and \textsc{Tu, Y.} (2023).
\newblock Stochastic gradient descent introduces an effective
  landscape-dependent regularization favoring flat solutions.
\newblock \textit{Physical Review Letters} \textbf{130} 237101.

\bibitem[{Zhang et~al.(2021)Zhang, Bengio, Hardt, Recht and
  Vinyals}]{zhang2021understanding}
\textsc{Zhang, C.}, \textsc{Bengio, S.}, \textsc{Hardt, M.}, \textsc{Recht, B.}
  and \textsc{Vinyals, O.} (2021).
\newblock Understanding deep learning (still) requires rethinking
  generalization.
\newblock \textit{Communications of the ACM} \textbf{64} 107--115.

\bibitem[{Zhu et~al.(2022{\natexlab{a}})Zhu, Liu, Radhakrishnan and
  Belkin}]{zhu2022quadratic}
\textsc{Zhu, L.}, \textsc{Liu, C.}, \textsc{Radhakrishnan, A.} and
  \textsc{Belkin, M.} (2022{\natexlab{a}}).
\newblock Quadratic models for understanding neural network dynamics.
\newblock \textit{arXiv preprint arXiv:2205.11787} .

\bibitem[{Zhu et~al.(2023)Zhu, Liu, Radhakrishnan and
  Belkin}]{zhu2023catapults}
\textsc{Zhu, L.}, \textsc{Liu, C.}, \textsc{Radhakrishnan, A.} and
  \textsc{Belkin, M.} (2023).
\newblock Catapults in sgd: spikes in the training loss and their impact on
  generalization through feature learning.
\newblock \textit{arXiv preprint arXiv:2306.04815} .

\bibitem[{Zhu et~al.(2022{\natexlab{b}})Zhu, Wang, Wang, Zhou and
  Ge}]{zhu2022understanding}
\textsc{Zhu, X.}, \textsc{Wang, Z.}, \textsc{Wang, X.}, \textsc{Zhou, M.} and
  \textsc{Ge, R.} (2022{\natexlab{b}}).
\newblock Understanding edge-of-stability training dynamics with a minimalist
  example.
\newblock In \textit{The Eleventh International Conference on Learning
  Representations}.

\bibitem[{Zou et~al.(2022)Zou, Cao, Li and Gu}]{zou2021understanding}
\textsc{Zou, D.}, \textsc{Cao, Y.}, \textsc{Li, Y.} and \textsc{Gu, Q.} (2022).
\newblock Understanding the generalization of adam in learning neural networks
  with proper regularization.
\newblock In \textit{The Eleventh International Conference on Learning
  Representations}.

\bibitem[{Zou et~al.(2023)Zou, Cao, Li and Gu}]{zou2023benefits}
\textsc{Zou, D.}, \textsc{Cao, Y.}, \textsc{Li, Y.} and \textsc{Gu, Q.} (2023).
\newblock The benefits of mixup for feature learning.
\newblock In \textit{Proceedings of the 40th International Conference on
  Machine Learning}.

\end{thebibliography}

\bibliographystyle{ims}

\newpage

\appendix

\tableofcontents

\section{Preliminary Lemmas on Concentration}\label{sec: preliminary lemmas}

In this section, we give finite-sample concentration results to characterize the high-probability concentration properties of the random elements involved in our problem.
Throughout this paper, we fix a small constant failure probability $ p = 1/\mathrm{poly}(d)$.

\begin{lemma}\label{lem: balanced sample}
Suppose that $n \ge 8\log(4/p)$, then with probability at least $1-p$, we have that
\begin{align}
    |\{i\in[n]:y_i = 1\}|\wedge |\{i\in[n]:y_i = -1\}|  \ge \frac{n}{4}. \label{eqn: balanced sample}
\end{align}
\end{lemma}

\begin{proof}[Proof of Lemma~\ref{lem: balanced sample}]
    See Lemma B.1 in \cite{cao2022benign} for a proof.
\end{proof}

\begin{lemma}\label{lem: balanced weak data}
Suppose that $n \ge 8\log(4/p)$, then with probability at least $1-p$, we have that
\begin{align}
    |\mathcal{W}| \leq \frac{2\rho}{n}. \label{eqn: balanced weak data}
\end{align}
\end{lemma}

\begin{proof}[Proof of Lemma~\ref{lem: balanced weak data}]
    This follows from the same proof as Lemma~\ref{lem: balanced sample}.
\end{proof}

\begin{lemma}\label{lem: noise norm and correlation}
Suppose that $d = \Omega(\log (4n/p))$, then with probability $1-p$, we have that 
\begin{gather}
    \sigma_p^2d /2 \le \norm{\bxi}_2^2  \le 3\sigma_p^2d/2, 
    \quad |\dotp{\bxi_i}{\bxi_{i^\prime}} | \le 2\sigma_p^2 \cdot \sqrt{d\log(2n/p)}\label{eq: noise corrletation ub}
\end{gather}
hold for all $i,i'\in [n]$.
\end{lemma}
\begin{proof}[Proof of Lemma~\ref{lem: noise norm and correlation}]
    See Lemma B.2 in \cite{cao2022benign} for a proof.
\end{proof}



\begin{lemma}\label{lem: initialization}
Suppose that $d \ge \Omega(\log(mn/p))$. Then with probability at least $1-p$, we have that 
\begin{align}
    \sigma_0 \norm{\ub}_2/2\le \max_{j\in\{\pm 1\},r\in[m]} \dotp{\wb_{j,r}^{(0)}}{j\ub} &\le \sqrt{2\log(16 m/p)}\cdot \sigma_0 \norm{\ub}_2,\label{eq: LUBPosPartIPU}\\
\min_{j\in\{\pm 1\},r\in[m]} \dotp{\wb_{j,r}^{(0)}}{j\ub} &\ge -\sqrt{2\log(16 m/p)}\cdot  \sigma_0 \norm{\ub}_2,\label{eq: LBNegPartIPU}\\
    \sigma_0 \norm{\vb}_2/2\le \max_{j\in\{\pm 1\},r\in[m]} \dotp{\wb_{j,r}^{(0)}}{j\vb} &\le \sqrt{2\log(16 m/p)}\cdot \sigma_0 \norm{\vb}_2, \label{eq: LUBPosPartIPV}\\
\min_{j\in\{\pm 1\},r\in[m]} \dotp{\wb_{j,r}^{(0)}}{j\vb} &\ge -\sqrt{2\log(16 m/p)}\cdot  \sigma_0 \norm{\vb}_2.\label{eq: LBNegPartIPV}\\
\sigma_0\sigma_p\sqrt{d}/4 \leq\max_{r\in[m]} j\cdot\langle \mathbf{w}_{j,r}^{(0)},\boldsymbol{\xi}_i\rangle &\leq 2\sqrt{\log(16mn/p)}\cdot\sigma_0\sigma_p\sqrt{d},\quad \forall i\in[n],
\end{align}
\end{lemma}

\begin{proof}[Proof of Lemma~\ref{lem: initialization}]
    See Lemma B.3 in \cite{cao2022benign} for a proof.
\end{proof}
\newcommand{\MaxInitIPU}{{\sqrt{2\log(16 m/p)}\cdot \sigma_0 \norm{\ub}_2}}
\newcommand{\MaxInitIPV}{{\sqrt{2\log(16 m/p)}\cdot \sigma_0 \norm{\vb}_2}}

\section{Proofs for Single Training Data Case (Section~\ref{sec: one data})}\label{sec: one data case proof}

In this section, we give a formal statement and a detailed proof for our main results on the single noiseless training data setup, Theorem~\ref{thm: one data}.

We begin with the formal statement on the conditions and assumptions required for Theorem~\ref{thm: one data}. 
Firstly, we put requirements on the data model, initialization, and learning rates.

\begin{assumption}[Conditions on hyperparameters]\label{cond:model_params_one_data}
    Suppose that the following holds:
    \begin{enumerate}[leftmargin = 0.3in]
        \item The learning rate $m/(4\|\mathbf{u}\|_2^2)\leq \eta \leq 2m/(5\|\mathbf{u}\|_2^2)$;
        \item The weight initialization scale $\sigma_0 = \widetilde{\Theta}(\max\{\|\mathbf{u}\|_2, \|\mathbf{v}\|_2 \}^{-1}\cdot d^{-1/2})$;
        \item The signal strength $\norm{\vb}_2< 0.01\cdot \norm{\ub}_2$;
        \item The dimension $d$ satisfies $d = \Omega(\mathrm{polylog}(m))$.
    \end{enumerate}
\end{assumption} 

These conditions are the simplified version of those conditions from Assumption~\ref{ass: conditions} for the multiple data setup.
We refer the readers to Section~\ref{subsec: keu conditions and assumptions} for an explanation of these conditions.

The next assumption is on the training process, which requires that the SGD oscillates.

\begin{assumption}[Oscillations SGD: single training data case]\label{ass:oscilation_one_data}
    We assume that there exists some constant $\delta\in(0.2,0.8)$, such that $|yf(\mathbf{x};\mathbf{W}^{(t)}) - 1|\geq \delta$ for any $t\geq 0$, where $(\mathbf{x}, y)$ denotes the single data point, $\mathbf{W}^{(t)}$ denotes the weights found by SGD \eqref{eq: single data sgd}.
\end{assumption}

Again, we refer the readers to Section~\ref{subsec: keu conditions and assumptions} for an explanation of Assumption~\ref{ass:oscilation_one_data}.
With Assumptions~\ref{cond:model_params_one_data} and \ref{ass:oscilation_one_data}, our formal statement of Theorem~\ref{thm: one data} is the following theorem.

\begin{theorem}[Restatement of Theorem~\ref{thm: one data}]\label{thm:one_data_big_eta}
    Under Assumptions~\ref{cond:model_params_one_data} and \ref{ass:oscilation_one_data}, with probability at least $1-1/\mathrm{poly}(d)$, there exists a step $T_{(\vb)}$ such that for $j=y$ and any $t\ge T_{(\vb)}$, it holds that
    \begin{align}\label{eq: weak signal learning one data}
        \frac{1}{m}\sum_{r\in[m]}\sigma(\langle\mathbf{w}_{j,r}^{(t)},j\mathbf{v}\rangle) - \frac{1}{m}\sum_{r\in[m]}\sigma(\langle\mathbf{w}_{-j,r}^{(t)},j\mathbf{v}\rangle) \geq \frac{\delta}{4},
    \end{align}
    where $\delta>0$ is specified in Assumption \ref{ass:oscilation_one_data} and $T_{(\vb)} \le  \widetilde{\Theta}(m\cdot \eta^{-1} \cdot \norm{\vb}^{-2}_2 \cdot \delta^{-1}\cdot \log (m\delta\sigma_0^{-1} \norm{\vb}_2^{-1}))$.  
\end{theorem}

\begin{proof}[Proof of Theorem~\ref{thm:one_data_big_eta}]
    Please refer to Appendix~\ref{subsec: proof one data} for a detailed proof.
\end{proof}

The following of this section is organized as following.
Appendix~\ref{subsec: basic props one data} presents important properties of the whole training dynamics and the CNN, which serve as the basis for all the following proofs. 
Appendix~\ref{subsec: fundamental} presents the fundamental step that allows for proving weak signal learning. 
Based on that, we prove the main theorem in Appendix~\ref{subsec: proof one data}.
Finally, the remaining subsections collect the proof for all the lemmas and technical results  involved in Appendix~\ref{sec: one data case proof}.

\subsection{Basic Properties of Training Dynamics and the Two-Layer CNN}\label{subsec: basic props one data}

\paragraph{Properties of training dynamics.} 
We first define some neuron subsets. For $j\in \{ \pm 1\}$, we define that
$$
\cU_{j,+}^{(t)} = \left\{r\in [m]: \dotp{\wb_{y,r}^{(t)}}{y\ub}>0\right\},\quad 
\cU_{j,-}^{(t)} = \left\{r\in [m]: \dotp{\wb_{y,r}^{(t)}}{y\ub}\le 0\right\}.
$$ 
and 
$$
\cV_{j,+}^{(t)} = \left\{r\in [m]: \dotp{\wb_{y,r}^{(t)}}{y\vb}>0\right\},\quad 
\cV_{j,-}^{(t)} = \left\{r\in [m]: \dotp{\wb_{y,r}^{(t)}}{y\vb}\le 0\right\}.
$$ 
By the update formula \eqref{eq: single data sgd}, we know that the gradient descent iterates the inner products $\{\dotp{\wb_{j,r}^{(t)}}{j \ub}\}_{t\geq 0}$ as follows:
\begin{align}
\dotp{\wb_{y,r}^{(t+1)}}{y\ub} &= \dotp{\wb_{y,r}^{(t)}}{y\ub} + \frac{\eta\norm{\ub}_2 ^2}{m} \cdot \bigl( 1-yf(\xb;\Wb^{(t)})\bigr)\cdot \sigma'\bigl(\dotp{\wb_{y,r}^{(t)}}{y\ub}\bigr), \label{eq: ip positive gd one data}\\
\dotp{\wb_{-y,r}^{(t+1)}}{-y\ub} &= \dotp{\wb_{-y,r}^{(t)}}{-y\ub} + \frac{\eta\norm{\ub}_2 ^2}{m} \cdot \bigl(1-yf(\xb;\Wb^{(t)})\bigr)\cdot \sigma'\bigl(-\dotp{\wb_{-y,r}^{(t)}}{-y\ub}\bigr). \label{eq: ip negative gd one data} 
\end{align}
Analogously, we have that 
\begin{align}
\dotp{\wb_{y,r}^{(t+1)}}{y\vb} &= \dotp{\wb_{y,r}^{(t)}}{y\vb} + \frac{\eta\norm{\vb}_2 ^2}{m} \cdot\bigl(1-yf(\xb;\Wb^{(t)})\bigr)\cdot \sigma'\bigl(\dotp{\wb_{y,r}^{(t)}}{y\vb}\bigr), \label{eq: ipv positive gd one data}\\
\dotp{\wb_{-y,r}^{(t+1)}}{-y\vb} &= \dotp{\wb_{-y,r}^{(t)}}{-y\vb} + \frac{\eta\norm{\vb}_2 ^2}{m} \cdot\bigl(1-yf(\xb;\Wb^{(t)})\bigr)\cdot \sigma'\bigl(-\dotp{\wb_{-y,r}^{(t)}}{-y\vb}\bigr). \label{eq: ipv negative gd one data} 
\end{align} 
Then we invite readers to some facts that helps to understand the behavior of the inner products during the training processes. 
First, we note that by \eqref{eq: ip positive gd one data}, for any $r\in \cU_{y,-}^{(0)}$, 
\begin{align}
    \left\{\dotp{\wb_{y,r}^{(t)}}{y\ub}\right\}_{t\geq 0}
\end{align}
stays fixed at its initialization, thus automatically keeps a fixed sign. 
The same phenomenon that the inner products are fixed can be verified on following neuron sets,
\begin{align}
    \left\{\dotp{\wb_{-y,r}^{(t)}}{-y\ub}\right\}_{t\geq 0},\,\, \forall r\in\cU_{-y,+}^{(0)}; \qquad \left\{\dotp{\wb_{y,r}^{(t)}}{y\vb}\right\}_{t\geq 0},\,\, \forall r\in\cV_{y,-}^{(0)}; \qquad \left\{\dotp{\wb_{-y,r}^{(t)}}{-y\vb}\right\}_{t\geq 0},\,\, \forall r\in\cV_{-y,+}^{(0)}.
\end{align}
The benefit of this phenomenon is when looking at the prediction $f(\cdot; \Wb^{(t)})$ on the single data with label $y$,
\begin{align}\label{eq: f expression one data}
     f(\xb;\Wb^{(t)})&= \frac{y}{m} \sum_{r\in [m]} \sigma\bigl( \dotp{\wb_{y,r}^{(t)}}{y \ub}\bigr) + \sigma\bigl( \dotp{\wb_{y,r}^{(t)}}{y \vb}\bigr) -\sigma\bigl( \dotp{\wb_{-y,r}^{(t)}}{y \ub}\bigr) - \sigma\bigl( \dotp{\wb_{-y,r}^{(t)}}{- y \vb}\bigr).
\end{align}
By splitting the summation over $r\in[m]$ according to our defined neuron sets, we can simplify \eqref{eq: f expression one data} as
\begin{align}
    f(\xb;\Wb^{(t)})& =  \frac{y}{m} \sum_{r\in \cU ^{(t)}_{y,+}}\sigma\bigl( \dotp{\wb_{y,r}^{(t)}}{y \ub}\bigr) + \frac{y}{m}\sum_{r\in \cV^{(t)}_{y,+}}\sigma\bigl( \dotp{\wb_{y,r}^{(t)}}{y \vb}\bigr)\\
     &\qquad    -\frac{y}{m}\sum_{r\in \cU^{(t)}_{-y,-}}\sigma\bigl(-\dotp{\wb_{-y,r}^{(t)}}{-y \ub}\bigr) - \frac{y}{m}\sum_{r\in \cV^{(t)}_{-y,-}}\sigma\bigl(-\dotp{\wb_{-y,r}^{(t)}}{-y\vb}\bigr),\label{eq: f expression one data simplified}
\end{align}
which only involves $\wb_{y,r}^{{(t)}}$ with $r\in \cU^{(t)}_{y,+} \cup  \cV^{(t)}_{y,+}$  and  $\wb_{-y,r}^{{(t)}}$ with $r\in \cU^{(t)}_{-y,-} \cup  \cV^{(t)}_{-y,-}$. 
Combined with previous observations, we only need to track the training dynamics of the neurons appearing in \eqref{eq: f expression one data simplified}.

Thanks to these facts and the signal strength regime in Assumption~\ref{ass:oscilation_one_data}, we can then retreat  to tracking the movements of 
\begin{align}
    \left\{\dotp{\wb_{y,r}^{(t)}}{y\ub}\right\}_{t\geq 0},\,\, \forall r\in\cU_{y,+}^{(0)} \quad\text{and}\quad \left\{\dotp{\wb_{-y,r}^{(t)}}{-y\ub}\right\}_{t\geq 0},\,\, \forall r\in\cU_{-y,-}^{(0)}
\end{align}
whenever the weak signal part is not yet effectively learned and the strong signal part still dominates. 
Ideally, only the inner products $\dotp{\wb_{y,r}^{(t)}}{y\ub}, r\in \cU^{(t)}_{y,+}$ take the lead.

Then a natural and crucial question is the boundedness of these inner products, which turns out to be the cornerstone for the subsequent analysis. 
A straightforward but helpful lemma indicates that the inner products that are initialized to be the maximal (resp. minimal) among all inner products continue to be the maximal (resp. minimal) throughout the training process. 
To put formally, we have the following.

\begin{lemma}[Maximum and minimum neurons] \label{lem: dominate one data} 
Suppose that the signs of all the related inner products do not change throughout $[t_1, t_2]$, i.e., $\cU_{y,+}^{(t)} = \cU_{y,+}^{(t_1)}$ and $\cU_{-y,-}^{(t)} = \cU_{-y,-}^{(t_1)}$ for all $t\in[t_1,t_2]$.
Then we have that
\begin{align}
    \argmax_{r\in [m]}\dotp{\wb_{y,r}^{(t_1)}}{y\ub} &= \argmax_{r\in [m]}\dotp{\wb_{y,r}^{(t)}}{y\ub},\\
    \argmin_{r\in [m]}\dotp{\wb_{-y,r}^{(t_1)}}{-y\ub} &= \argmin_{r\in [m]} \dotp{\wb_{-y,r}^{(t)}}{-y\ub}
\end{align}
hold for all $t\in [t_1, t_2]$.
\end{lemma}
\begin{proof}[Proof of Lemma~\ref{lem: dominate one data}]
    See Appendix~\ref{subsubsec: proof: dominant one data} for a detailed proof.
\end{proof}
A direct profit of this lemma is that when the signs keep invariant, it suffices to track two specific indices in $\cU_{y,+}^{(t)}$ and $\cU_{-y,-}^{(t)}$, the maximum and the minimum, to analyze upper and lower bounds for all neurons.

\paragraph{Single neuron behaves similarly to CNN.} The proof of boundedness utilizes another property of two-layer CNN defined in Equation~\eqref{eq: cnn} that exhibits the connections between the behavior of inner products $\dotp{\wb_{y,r}}{y\ub}$ and the outcome of the model $f(\cdot\,;\Wb)$. We state it as follows.

\begin{lemma}[Single neuron imitates entire CNN]\label{lem: single neuron behaves similarly}\label{lem: similar behavior}
Define the major part of $y\cdot f(\xb;\Wb)$ as 
\begin{align}
    g(\xb,y;\Wb) = \frac{1}{m}\sum_{r\in [m]} \sigma\big(\dotp{\wb_{y,r}}{ y \ub}\big).
\end{align}
Suppose that there is $t_1< t_2$ such that $\cU_{y,+}^{(t)} = \cU_{y,+}^{(t_1)}$ for all $t\in [t_1, t_2]$. Then, for any $c>0$, $g(\xb,y;\Wb^{(t)}) \ge c$ implies that for all $t\in[t_1,t_2]$,
\begin{align}
    \max_{r\in [m]}\dotp{\wb_{y,r}^{(t)}}{y\ub} \ge  ({\beta^*_{\ub}}^{(t_1)} m c)^{1/2}. 
\end{align}
On the other hand, $g(\xb;\Wb^{(t)}) \le c$ implies that for all $t\in[t_1,t_2]$,
\begin{align}
    \max_{r\in [m]} \dotp{\wb_{y,r}^{(t)}}{y\ub} \le ({\beta^*_{\ub}}^{(t_1)} m c)^{1/2}. 
\end{align}
Here ${\beta^*_{\ub}}^{(t)}$ is defined as 
\begin{align}
    {\beta^*_{\ub}}^{(t)} =  \frac{\max_{r\in [m ]} \sigma (\dotp{\wb_{y,r}^{(t)}}{y\ub} )}{ \sum_{r\in [m]}\sigma(\dotp{\wb_{y,r}^{(t)}}{y\ub})}.
\end{align}
\end{lemma}

\begin{proof}[Proof of Lemma~\ref{lem: single neuron behaves similarly}]
    See Appendix~\ref{subsubsec: proof: similarly one data} for a detailed proof. 
    For multiple training data setting, the proof is in Appendix~\ref{subsubsec: proof similar behavior multiple data setting}.
\end{proof}

Being a subtle condition required in the previous two lemmas, whether the signs of these inner products are invariant throughout the process remains unknown, making the behaviors of these inner products complicated. The answer to this question is affirmative, as we are able to prove that, under proper conditions, the signs of these inner products are fixed throughout the process. 
Using Lemmas~\ref{lem: dominate one data} and \ref{lem: single neuron behaves similarly}, we prove the boundedness and the sign stability simultaneously through a sophisticated inductive argument. 
The formal statement of this result is as follows. 

We first define a stopping time. Let 
\begin{align}\label{eq: stopping time}
    T_{(\vb)} =& \min_{t\geq 0} \left\{t: \frac{1}{m} \sum_{r\in[m]}\sigma\bigl(\dotp{\wb_{y,r}^{(t)}}{y\vb}\bigr)> \delta\right\}.
\end{align}
Such a stopping time helps to control the non-dominating terms in the CNN. 
Moreover, once the process reaches $T_{(\vb)}$, the conclusion in Theorem~\ref{thm:one_data_big_eta} is nearly achieved. 
Our results are summarized as follows.

\begin{lemma}[Boundedness and sign stability: single training data case]\label{lem: boundedness oscillating one data}
Suppose that Assumptions~\ref{cond:model_params_one_data} and \ref{ass:oscilation_one_data} hold.
With probability at least $1 - p =1-1/\mathrm{poly}(d)$, we have the following bounds:
\begin{align}
    \max_{r} \dotp{\wb_{y,r}^{(t)}}{y\ub}\le 1.5 \cdot \big(1.05 \beta^*_{\ub } m\big)^{1/2},\quad &\forall t\le T_{(\vb)}, \label{eqn:max_ub_one_data_oscil}\\
    \min_{r} \dotp{-\wb_{-y,r}^{(t)}}{-y\ub} \ge - \MaxInitIPU, \quad &\forall t\le T_{(\vb)}, \label{eqn:min_lb_one_data_oscil}\\
    \min_{r} \dotp{-\wb_{-y,r}^{(t)}}{-y\vb} \ge - \MaxInitIPV, \quad &\forall t\le T_{(\vb)}, \label{eqn:min_wv_lb_one_data_oscil}\\
    0 \le yf(\xb;\Wb^{(t)}) \le  3,  \quad &\forall t\le T_{(\vb)} .\label{eqn:1-f bound_one_data_oscil}
\end{align}
Here $\beta^*_\ub :={\beta^*_\ub}^{(0)}$ is defined in Lemma~\ref{lem: similar behavior}. 
Besides, the sign stability holds before $T_{(\vb)}$, i.e. $\cU^{(t)}_{ \pm y, \pm}$ and $\cV^{(t)}_{ \pm y, \pm}$ remain invariant in $t$, and the superscript $(t)$ can be dropped. 
\end{lemma}
\begin{proof}[Proof of Lemma~\ref{lem: boundedness oscillating one data}]
    See Appendix~\ref{subsubsec: proof: boundedness oscillating one data} for a detailed proof.
\end{proof}


Thanks to the stopping time defined in \eqref{eq: stopping time} which puts controls on the scale of $\dotp{\wb_{j,r}^{(t)}}{j\vb}$, the major part function $g$ defined in the previous lemma dominates the entire CNN, as the negative parts $\dotp{\wb_{-y,r}^{(t)}}{-y \ub}$, $\dotp{\wb_{-y,r}^{(t)}}{-y \vb}$ can be lower bounded in Lemma~\ref{lem: boundedness oscillating one data} through a delicate analysis of the training dynamics.

Lemmas~\ref{lem: dominate one data}, \ref{lem: single neuron behaves similarly}, and \ref{lem: boundedness oscillating one data} reveal the key properties of the training dynamics and the two-layer CNN. 
Several remarks are put here again. Lemmas~\ref{lem: dominate one data} and~\ref{lem: single neuron behaves similarly} are temporarily \emph{local}, with the condition on the local sign stability. 
They are not informative until we are able to extend the sign stability to a wider sense, which is achieved by Lemma~\ref{lem: boundedness oscillating one data}. 
Nevertheless, these two local lemmas are used frequently throughout the subsequent analysis, so we single them out here to make the proof more readable. 

\subsection{Fundamental Reasoning towards the Weak Signal Learning}\label{subsec: fundamental}

The previous section presents several basic properties of the training dynamics and the CNN. 
However, they are insufficient in interpreting the driving force of the weak signal learning during oscillation. 
In the lemma below, we discover a quantitative interpretation towards the increasing on $\dotp{\wb_{y,r}^{(t)}}{y\vb}$, which formalizes the illustration of the function of oscillation in Section~\ref{sec: one data}.

\begin{lemma}[Weak signal learning: single training data case]\label{lem: linear increment one data}
    Under Assumptions~\ref{cond:model_params_one_data} and \ref{ass:oscilation_one_data}, suppose that there exists $t_0\leq t_1$, such that:
    \begin{enumerate}[leftmargin = 0.3in]
        \item The sign stability holds on $[t_0,t_1]$, i.e.,  i.e. $\cU^{(t)}_{ \pm y, \pm}$ and $\cV^{(t)}_{ \pm y, \pm}$ remain invariant in $t$;
        \item $\max_{r\in [m] }\dotp{\wb_{y,r}^{(t)}}{y\ub} < B \cdot (\beta^*_\ub m)^{1/2},\; \forall t\in [t_0,t_1]$ for some $B>0$;
        \item $\min_{r\in [m] } \dotp{\wb_{-y,r}^{(t)}}{-y\ub} > -0.1$ and $\min_{r\in [m] } \dotp{\wb_{-y,r}^{(t)}}{-y\vb} > -0.1,\; \forall t\in [t_0,t_1]$;
        \item $\frac{1}{m}\sum_{r\in [m]}\sigma (\dotp{\wb_{y,r}^{(t)}}{y\vb}) < \delta,\; \forall t\in [t_0,t_1]$;
        \item $-2\le 1-yf(\xb;\Wb^{(t)}) \le 1,\; \forall t\in [t_0,t_1]$.
    \end{enumerate}
    Then we have that
    \begin{align}
        \sum_{s=t_0}^{t_1-1}\big(1 - yf(\mathbf{x};\mathbf{W}^{(s)})\big) \geq 2\epsilon \cdot (t_1 - t_0) - \frac{mB}{\eta \norm{\ub}_2^2 \sqrt{1.05-\delta}}. \label{eq: sum 1 - yf lower bound}
    \end{align}
    Consequently, we further have that for any $r\in\cV_{y,+} = \{r\in[m]:  \langle \mathbf{w}_{y,r}^{(t_0)},y\mathbf{v}\rangle >0\}$, it holds that 
    \begin{align}
        \langle\mathbf{w}_{y,r}^{(t_1)},y\mathbf{v}\rangle  \geq  \langle\mathbf{w}_{y,r}^{(t_0)},y\mathbf{v}\rangle \cdot \exp \left\{  \frac{\eta \norm{\vb}_2^2 \epsilon}{m}\cdot(t_1 - t_0) - \frac{\norm{\vb}_2^2}{ \norm{\ub}_2^2} \cdot \frac{B}{(1.05-\delta)^{1/2}}\right\}.\label{eq: exp rate increase}
    \end{align}
    Here $\epsilon = (\delta - \delta(1.05-\delta)^{1/2})/4$ with $\delta$ specified in Assumption~\ref{ass:oscilation_one_data}.
\end{lemma}

\begin{proof}[Proof of Lemma~\ref{lem: linear increment one data}]
    See Appendix~\ref{subsec: proof increment} for a detailed proof.
\end{proof}

This lemma asserts that, stable oscillation in a bounded and favorable training area leads to a linear increasing lower bound for $\sum_{t} (1-yf_t)$. This is the first part of Lemma~\ref{lem: linear increment one data}. 
The second part of this lemma relates the increasing speed of $\dotp{\wb_{y,r}^{(t)}}{y\vb}$ to the summation of $1-yf_t$. The derivation of this part relies on the fact that $\alpha = \norm{\vb}^2_2/\norm{\ub}_2^2\in (0,1)$ and is close to $0$. This observation motivates us to approximate the ratio with a first-order Taylor expansion,
\begin{align}
    \frac{\dotp{\wb_{y,r}^{(t_1)}}{y\vb}}{\dotp{\wb_{y,r}^{(t_0)}}{y\vb}} &= \prod_{t'=t_0}^{t_1-1} \left\{1+\alpha\et\big(1-yf(\xb;\Wb^{(t')})\big)\right\}\\
    &\approx 1+\alpha \et \sum_{t'=t_0}^{t_1-1} \big(1-yf(\xb;\Wb^{(t')})\big).
\end{align}
The proof of the second part of this lemma justifies this intuition formally with more delicate analysis. With all these collected results, we are ready to present the proof of the main theorem.

\subsection{Proof of Theorem \ref{thm:one_data_big_eta}}\label{subsec: proof one data}

\begin{proof}[Proof of Theorem \ref{thm:one_data_big_eta}]
We prove Theorem \ref{thm:one_data_big_eta} by contradiction. 
Recall that in the previous section we have defined that 
\begin{align}
    T_{(\vb)} &= \min_{t\geq 0} \left\{t: \frac{1}{m}\sum_{r\in[m]}\sigma\bigl(\dotp{\wb_{y,r}^{(t)}}{y\vb}\bigr)> \delta\right\}. \label{eq: Tv +}
\end{align}
With this definition, our first goal is to prove that $T_{(\vb)}$ is bounded by a finite time with explicit expression. 
To put it precisely, we are going to prove by contradiction that 
\begin{align}
    T_{(\vb)} < T_{0} := \frac{m}{\eta \norm{\vb}_2^2 \epsilon} \cdot \left\{  \log\left( \frac{2\sqrt{m\delta}}{ \sigma_0 \norm{\vb}_2 } \right)+ 1.5\cdot\frac{\norm{\vb}_2^2}{\norm{\ub}_2^2} \cdot \sqrt{\frac{ 1.05}{1-\delta}} \right\},
\end{align}
where $\epsilon$ is specified in Lemma~\ref{lem: linear increment one data} and $\delta$ is specified in Assumption~\ref{ass:oscilation_one_data}. 
Suppose otherwise that $T_{(\vb)}\ge T_0$. 
By Lemma~\ref{lem: boundedness oscillating one data} we can see that for $t\in[0,T_0]$ all the conditions of Lemma~\ref{lem: linear increment one data} are satisfied.
Then Lemma~\ref{lem: linear increment one data} implies that,
\begin{align}
    \dotp{\wb_{y,r^*}^{(T_0)}}{y\vb} &\geq  \dotp{\wb_{y,r^*}^{(0)}}{y\vb} \cdot \exp\left\{ \frac{\eta \norm{\vb}_2^2\epsilon}{m} \cdot T_0  - 1.5\cdot \frac{ \norm{\vb}_2^2}{\norm{\ub}_2^2} \cdot \sqrt{\frac{1.05}{1.05-\delta }}\right\} \hfill \\
    &\ge \frac{1}{2}\sigma_0 \norm{\vb}_2\cdot  \frac{2\sqrt{m\delta}}{\sigma_0 \norm{\vb}_2}\ge \sqrt{m\delta}.
\end{align}
where $r^* = \argmax_{r\in[m]}\dotp{\wb_{y,r}^{(t)}}{y\vb}$ are fixed throughout $t\in[0,T_0]$ (see Lemma~\ref{lem: dominate one data}) and in the second inequality we apply Lemma~\ref{lem: initialization} to lower bound the initialization.
This leads to the following,
\begin{align}
    \frac{1}{m} \sum_{r\in [m]} \sigma\big(\dotp{\wb_{y,r}^{(T_0)}}{y\vb}\big) \ge  \frac{1}{m} \sigma\big(\dotp{\wb_{y,r^*}^{(T_0)}}{y\vb}\big) \ge \delta,
\end{align}
which contradicts the definition of $T_{(\vb)}$, and therefore $T_{(\vb)} < T_0 $. 

The rest part of the proof is to show that the sequence 
\begin{align}\label{eq: sequence}
    \left\{\frac{1}{m} \sum_{r\in[m]}  \sigma\big(\dotp{\wb_{y,r}^{(t)}}{y\vb}\big)\right\}_{t\ge T_{(\vb)}}
\end{align}
does not fall below $\delta/2$. 
Intuitively, as long as the sequence above falls below $\delta$, the sequence would have the incentive to increase, as the results in Lemma~\ref{lem: linear increment one data} are revived again based on another boundedness argument.
The analysis resembles the proof of Lemma~\ref{lem: boundedness oscillating one data} with slight differences. We provide the following proposition with the proofs delayed to Appendix~\ref{subsec: proof of technical results single data}. 
\begin{proposition}[Weak signal memorization]\label{prop: doesn't decrease}
It holds that
\begin{align}
    \frac{1}{m}\sum_{r\in [m]} \sigma \big(\dotp{\wb_{y,r}^{(t)}}{y\vb}\big) >\frac{\delta}{2}, \quad \forall t\ge T_{(\vb)}.
\end{align}
\end{proposition}
\begin{proof}[Proof of Proposition~\ref{prop: doesn't decrease}]
    See Appendix~\ref{subsubsec: proof: doesn't decrease} for a detailed proof.
\end{proof}
This proposition finalizes the proof of Theorem~\ref{thm:one_data_big_eta}, and we are done. 
\end{proof}

\subsection{Proof of Lemmas in Appendix~\ref{subsec: basic props one data}}\label{subsec: proof of key lemmas}


\subsubsection{Proof of Lemma~\ref{lem: dominate one data}}\label{subsubsec: proof: dominant one data}

\begin{proof}[Proof of Lemma~\ref{lem: dominate one data}]
By our assumption that the signs of the inner products does not change throughout $[t_1, t_2]$, it is straightforward that $\cU_{y,+}^{(t)}=\cU_{y,+}^{(t_1)}$ for every $t \in [t_1, t_2]$. 
The same is true for $\cU_{-y,-}^{(t)}$. Therefore, we are able to drop the superscript $(t)$ temporarily, as the attention is restricted to a local interval $[t_1, t_2]$. 

Regarding the maximal index, introducing $r'\neq r\in \cU_{y,+}$, we have following relation
\begin{align}
\argmax_{r\in [m]} \dotp{\wb_{y,r}^{(t)}}{y\ub} &= \argmax_{r\in \cU_{y,+}} \dotp{\wb_{y,r}^{(t)}}{y\ub} , \\
 &= \argmax_{r\in \cU_{y,+}} \dotp{\wb_{y,r}^{(t)}}{y\ub} /  \dotp{\wb_{y,r'}^{(t)}}{y\ub} \\
 &= \argmax_{r\in \cU_{y,+}}\frac{\dotp{\wb_{y,r}^{(t_1)}}{y\ub}\cdot \prod_{t' = t_1}^{t-1} \Big(1+ \frac{2\eta\norm{\ub}_2^2}{m}\cdot \big( 1- yf( \xb; \Wb^{(t)})\big)\Big)}{  \dotp{\wb_{y,r'}^{(t_1)}}{y\ub}\cdot \prod_{t' = t_1}^{t-1} \Big(1+ \frac{2\eta\norm{\ub}_2^2}{m} \cdot\big( 1- yf( \xb; \Wb^{(t)})\big)\Big)} \\ 
  &= \argmax_{r\in \cU_{y,+}} \dotp{\wb_{y,r}^{(t_1)}}{y\ub} /  \dotp{\wb_{y,r'}^{(t_1)}}{y\ub} \\
  &= \argmax_{r\in \cU_{y,+}} \dotp{\wb_{y,r}^{(t_1)}}{y\ub}.
\end{align}
The same relation can be verified for $\dotp{\wb_{-y,r}^{(t}}{y\ub}$ with $r\in \cU_{-y,-}$, finishing the proof.
\end{proof}

\subsubsection{Proof of Lemma~\ref{lem: single neuron behaves similarly}}\label{subsubsec: proof: similarly one data}

\begin{proof}[Proof of Lemma~\ref{lem: single neuron behaves similarly}]\label{proof: similarly one data}
Again, the local sign stability assumption ensures that each inner product grows proportionally and the superscript $(t)$ in the neuron index sets can be dropped, with
\begin{align}
    g(\xb,y;\Wb^{(t)}) &= \frac{1}{m} \sum_{r\in [m]} \sigma\big(\dotp{\wb_{y,r}^{(t)}}{y \ub}\big) \\ 
     &= \frac{1}{m} \sum_{r\in \cU_{y,+}} \sigma\big( \dotp{\wb_{y,r}^{(t_1 )}}{y \ub}\big) \cdot \prod_{t'=t_1}^{t-1} \left(1+ \frac{2\eta\norm{\ub}_2^2}{m}\cdot \big( 1- yf( \xb; \Wb^{(t')})\big)\right)^2 \\ 
     &= \frac{\max_{r\in [m]}\sigma\big( \dotp{\wb_{y,r}^{(t_1)}}{y \ub} \big)}{m\beta^{*,(t_1)}_{\ub}}   \prod_{t'=t_1}^{t-1} \left(1+ \frac{2\eta\norm{\ub}_2^2}{m}\cdot \big( 1- yf( \xb; \Wb^{(t')})\big)\right)^2 \\
     &= \frac{\sigma\big(\max_{r\in [m]} \dotp{\wb_{y,r}^{(t)}}{y \ub} \big)}{m\beta^{*,(t_1)}_{\ub}}.
     \label{eq: g(x,y) one data}
\end{align}
Here the second line and the last equality is true because \eqref{eq: ip positive gd one data} implies that all the positive $\dotp{\wb_{y,r}^{(t)}}{y\ub}$ iterates by sequentially multiplying the same factor 
\begin{align}
    1+ \frac{2\eta\norm{\ub}_2^2}{m} \cdot \big( 1- yf( \xb; \Wb^{(t')})\big).  
\end{align}
The third equality comes from the definition of $\beta^{*,(t_1)}_{\ub}$ in Lemma~\ref{lem: single neuron behaves similarly}. 
Thus, $g(\xb,y;\Wb^{(t)})\ge c$ implies that
\begin{align}
    \sigma\left(\max_{r\in [m]} \dotp{\wb_{y,r}^{(t)}}{y \ub}\right)>\beta^{*,(t_1)}_{\ub}\cdot mc
\end{align}
and the desired lower bound follows. 
The upper bound can be proved analogously and is omitted here. 
\end{proof}

\subsubsection{Proof of Lemma~\ref{lem: boundedness oscillating one data}}\label{subsubsec: proof: boundedness oscillating one data}

\begin{proof}[Proof of Lemma~\ref{lem: boundedness oscillating one data}]\label{proof: boundedness oscillating one data}

A roadmap is provided to help understand how every single step is achieved so that the readers can skip the details without leaving the key ideas behind. 
    
\paragraph{Recap on notations.} 
Recall that, $\cU_{j,+}^{(t)}$ is the set of indices $r\in [m]$ such that $\dotp{\wb_{j,r}^{(t)}}{j\ub} > 0$ and $\cU_{j,-}^{(t)}$ is the set of indices $r\in [m]$ such that $\dotp{\wb_{j,r}^{(t)}}{j\ub} \le 0$. 
Specially, let $\cU_{y,+} = \cU_{y,+}^{(0)}$ and $\cU_{-y,-} = \cU_{-y,-}^{(0)}$. With probability one, it holds that $\cU_{j,-} \cup \cU_{j,+}=[m]$. 
Let $r^*\coloneqq \argmax_{r\in [m]} \dotp{\wb_{y,r}^{(0)}}{y\ub}$ and $r_* \coloneqq \argmin_{r\in [m]} \dotp{\wb_{-y,r}^{(0)}}{-y\ub}$, i.e., $r^*$ (resp. $r_*$) denote the index of the maximum (resp. minimum) throughout the process as we are able to extend the results in Lemma~\ref{lem: dominate one data} globally.

We also introduce several supplemental notations here to facilitate the proof. 
Recursively, we define 
\begin{align}\label{eq: T bar}
    \bar{T}_{k}\coloneqq \min_{t\geq 0}\left\{t:t>\bar{T}_{k-1},\,\,yf(\xb;\Wb^{(t)})\ge 1 \text{ and } yf(\xb;\Wb^{(t-1)})<1\right\},
\end{align}
with $\bar{T}_0 = 0$. Similarly, we define that 
\begin{align}
\munderbar{T}_{k}\coloneqq \min_{t\geq 0}\left\{t:t>\munderbar{T}_{k-1},\,\,yf(\xb;\Wb^{(t)})< 1 \text{ and } yf(\xb;\Wb^{(t-1)})\ge 1 \right\},
\end{align}
and $\munderbar{T}_0 = 0$. 
Intuitively, $\bar{T}_{k}$ captures the times that  $yf(\xb;\Wb^{(t)})$ just exceeds $1$, and similarly for $\munderbar{T}_{k}$.

\paragraph{Roadmap.} From a high level, three steps are required to establish the full proof:
\vspace{-1mm}
\begin{enumerate}[leftmargin = 0.3in]
    \item  We verify that the lower bound in Inequality~\eqref{eqn:min_lb_one_data_oscil} in Lemma~\ref{lem: boundedness oscillating one data} holds for all $t\in [0,\bar{T}_1]$ with a direct monotonicity argument. 
    Additionally, we can prove that the upper bound in Inequality~\eqref{eqn:max_ub_one_data_oscil} holds for $t\in [0,\bar{T}_1)$ by using Lemma~\ref{lem: single neuron behaves similarly} and the definition of $\bar{T}_1$. The signs do not change in this stage, as shown in the details below.
    \item We extend the results in Lemma~\ref{lem: boundedness oscillating one data} to $t\in [\bar{T}_1, \bar{T}_2)$ with repeated use of Lemma~\ref{lem: single neuron behaves similarly}. The sign stability is guaranteed from an intermediate upper bound on $|1-yf(\xb;\Wb^{(t)})|$.
    \item Note that the condition on $[0,\bar{T}_1)$ (which is proved in the first step) required for the proof of the second step is again true for $t\in [0,\bar{T}_2)$, which is a consequence of the second step. 
    Therefore we can repeat the second step to extend the results in Lemma~\ref{lem: boundedness oscillating one data} to $t\in [\bar{T}_2, \bar{T}_3)$, and so on. So the results are true for all $t\le T_{(\vb)}$.
\end{enumerate}
The first and the last step above are relatively straightforward. 
However, the second step requires a delicate break-down analysis. 
Here we provide a more detailed roadmap for the second step.  
The goal is to prove the results in Lemma~\ref{lem: boundedness oscillating one data}, restricted to $t\in [\bar{T}_1, \bar{T}_2)$. 
This would be achieved in four split steps:
\vspace{-1mm}
\begin{enumerate}[label=2.\arabic*, leftmargin = 0.3in]
    \item Firstly, we prove the upper bound in Inequality~\eqref{eqn:max_ub_one_data_oscil} for $t=\bar{T}_1$ by tracking one-step gradient descent and the upper bound for $\dotp{\wb_{y,r^*}^{(\bar{T}_1 - 1)}}{y\ub}$. 
    Then with a monotonicity argument, 
    \begin{align}
        \dotp{\wb_{y,r^*}^{(t)}}{y\ub}\le\dotp{\wb_{y,r^*}^{(\bar{T}_1)}}{y\ub},\quad \forall t\in [\bar{T}_1, \munderbar{T}_1 ].
    \end{align}
    Besides, for $t\in[\bar{T}_1, \munderbar{T}_1)$ we can give a lower bound on $\dotp{\wb_{y,r^*}^{(t)}}{y\ub}$ with the help of Lemma~\ref{lem: single neuron behaves similarly}. 
    \item Based on the previous lower and upper bounds on $\dotp{\wb_{y,r^*}^{(\bar{T}_1)}}{y\ub}$, we can derive a lower bound on $\dotp{\wb_{y,r^*}^{(\munderbar{T}_1)}}{y\ub}$ by tracking one-step gradient descent. 
    We apply this worst-case tight lower bound to conclude the sign stability. 
    We note that this step is free of the lower bound on $\dotp{\wb_{-y,r_*}^{(t)}}{- y\ub}$ for $t\in (\bar{T}_1, \munderbar{T}_2]$, which we have not yet proved to be true. 
    \item Now we give lower bounds on $\dotp{\wb_{-y,r_*}^{(\munderbar{T}_1)}}{-y\ub}$ and $\dotp{\wb_{-y,r_*}^{(\munderbar{T}_1)}}{-y\vb}$. 
    We achieve this by a delicate usage of the lower bound on $\dotp{\wb_{y,r^*}^{(\munderbar{T}_1)}}{y\ub}$ (which we have proved in Step 2.2) plus an inequality that connects the relative increment of $\dotp{\wb_{y,r^*}^{(t)}}{y \ub}$ and $\dotp{\wb_{-y, r_*}^{(t)}}{-y\ub}$ (or $\dotp{\wb_{-y,r_*}^{(t)}}{-y\vb}$). 
    Thus we prove Inequalities~\eqref{eqn:min_lb_one_data_oscil} and \eqref{eqn:min_wv_lb_one_data_oscil} for $t=\munderbar{T}_1$, and this can be further extended to the entire $[\bar{T}_1,\bar{T}_2]$ by another monotonicity argument since $\munderbar{T}_1$ is the local minima.
    \item The remaining to is upper bound $\dotp{\wb_{y,r^*}^{(t)}}{y\ub}$ for $t\in (\munderbar{T}_1,\bar{T}_2)$. This is again a consequence of Lemma~\ref{lem: single neuron behaves similarly}, as exactly what has been done to upper bound $\dotp{\wb_{y,r^*}^{(t)}}{y\ub},t\le \bar{T}_1$ in Step 1. 
\end{enumerate}

Now with the roadmap in mind, we are ready to dive into the details of every step. 
\paragraph{Step 1: Pre-$\bar{T}_1$ Analysis.} Lemma~\ref{lem: initialization} 
indicates that the lower bound in Inequality~\eqref{eqn:min_lb_one_data_oscil} holds at initialization $t=0$ under Assumption~\ref{cond:model_params_one_data}. 
Moreover, the upper bounds on the maximal initial inner products in Lemma~\ref{lem: initialization} with Assumption~\ref{cond:model_params_one_data} indicate that 
\begin{align}\label{eq: |f| upper bound at 0}
    |yf(\xb;\Wb^{(0)})| &\le 2 \cdot \max_{r\in [m]} \dotp{\wb_{y,r}^{(0)}}{y\ub}^2 \vee  \max_{r\in [m]} \dotp{\wb_{-y,r}^{(0)}}{-y\ub}^2 \\
    &= \widetilde{\mathcal{O}}\big( \sigma_0^2 \norm{\ub}_2^2\big) \ll 1.
\end{align}
From this we know that $\bar{T}_1 \ge 1 $ and the upper bound in Inequality~\eqref{eqn:max_ub_one_data_oscil} is true at $t=0$.

Then the first step to do is to extend the lower on $\dotp{\wb_{-y,r}^{(t)}}{-y\ub}$ to $[1, \bar{T}_1]$. Definition of $\bar{T}_1$ implies that $yf(\xb; \Wb^{(t)})\le 1$ for $t\in [0,\bar{T}_1)$. 
Therefore for $r\in \cU_{-y,-}^{(0)}$, Equation~\eqref{eq: ip negative gd one data} gives that
\begin{align}
    \dotp{\wb_{-y,r}^{(t+1)}}{-y\ub} &= \dotp{\wb_{-y,r}^{(t)}}{-y\ub} + \frac{\eta\norm{\ub}_2^2}{m} \cdot \big(1 - y f(\xb; \Wb^{(t)})\big)\cdot \sigma'\big(- \dotp{\wb_{-y,r}^{(t)}}{-y\ub}\big) \\ 
    &= \dotp{\wb_{-y,r}^{(t)}}{-y\ub}  -  \frac{2\eta\norm{\ub}_2^2}{m} \cdot \big(1 - y f(\xb; \Wb^{(t)})\big)\cdot \dotp{\wb_{-y,r}^{(t)}}{-y\ub} \label{eq: ip neg iter expanded}\\
    &\ge \dotp{\wb_{-y,r}^{(t)}}{-y\ub}.
\end{align}
And furthermore 
\begin{align}
    \dotp{\wb_{-y,r}^{(\bar{T}_1)}}{-y\ub} &\ge \dotp{\wb_{-y,r}^{(0)}}{-y\ub} \ge -\MaxInitIPU.
\end{align}
Taking minimum with respect to $r\in[m]$ gives the result. Same argument can be applied to $\dotp{\wb_{-y,r}^{(\bar{T}_1)}}{-y\vb}$. So the lower bounds in Inequalities~\eqref{eqn:min_lb_one_data_oscil} and~\eqref{eqn:min_wv_lb_one_data_oscil} hold for $t\in [1, \bar{T}_1]$. 

Also, \eqref{eq: ip positive gd one data} and \eqref{eq: ipv positive gd one data} imply that $\dotp{\wb_{y,r}^{(t)}}{y\ub},r\in \cU^{(0)}_{y,+}$ and $\dotp{\wb_{y,r}^{(t)}}{y\vb},r\in \cV^{(0)}_{y,+}$ increase for all $t<\bar{T}_1$.
A natural consequence is that $yf(\xb;\Wb^{(t)})$ is non-decreasing in $t$ in this stage, since every summand (possibly with the negative sign before) in the summation is non-decreasing. 

Now we can prove the sign stability. 
For $r\in \cU_{y,+}^{(0)}$, we know from \eqref{eq: ip positive gd one data} that $\dotp{\wb_{y,r}^{(t)}}{y\ub}>\dotp{\wb_{y,r}^{(0)}}{y\ub}>0$, hence $r\in \cU_{y,+}^{(t)}$ is non-vanishing in $t$ by induction. 
Additionally, for $r\in\cU_{y,-}^{(0)}$, $\dotp{\wb_{y,r}^{(t)}}{y\ub}$ stays fixed in $t$, as mentioned in Appendix~\ref{subsec: basic props one data}. 
Two points together ensure that 
$\cU_{y,+}^{(t)}=\cU_{y,+}^{(0)}$ for $t\in [1, \bar{T}_1]$. 

For $r\in \cU_{-y,-}^{(0)}$, let's take a closer look at Equation~\eqref{eq: ip neg iter expanded} with $t=0$:
\begin{align}
    \dotp{\wb_{-y,r}^{(1)}}{-y\ub}  &= \dotp{\wb_{-y,r}^{(0)}}{-y\ub}  -  \frac{2\eta\norm{\ub}_2^2}{m} \cdot \big(1 - y f(\xb; \Wb^{(0)})\big)\cdot \dotp{\wb_{-y,r}^{(0)}}{-y\ub} \\
    &=\dotp{\wb_{-y,r}^{(0)}}{-y\ub}\cdot \underbrace{\left( 1 - \frac{2\eta\norm{\ub}_2^2}{m} \cdot \big(1 - y f(\xb; \Wb^{(0)})\big)\right) }_{>0\text{ if } \eta<(1-o(1))/2\cdot m\norm{\ub}_2^{-2}}.
\end{align}
Assumption~\ref{cond:model_params_one_data} ensures that $\eta < 0.4 m\norm{\ub}_2^{-2} < (1-o(1)) /2 \cdot m\norm{\ub}_2^{-2}$ so the sign change does not happen at $t=0$. As mentioned before, $yf(\xb;\Wb^{(t)})\le 1$ is non-decreasing in $t$ at this stage, and putting them together we know that
\begin{align}
     1 - \frac{2\eta\norm{\ub}_2^2}{m} \big(1 - y f(\xb; \Wb^{(t)})\big)\ge 0 ,\quad \forall t\in [1, \bar{T}_1).
\end{align}
The same sign stability can be verified for $\dotp{\wb_{-y,r}^{(t)}}{-y\vb}$, and therefore, the sign stability for $\mathcal{U}_{-y,-}^{(t)}$ and $\mathcal{V}_{-y,-}^{(t)}$ is ensured for $t\in [0, \bar{T}_1]$ and the lower bound in Inequality~\eqref{eqn:min_lb_one_data_oscil} holds for $t\in [1, \bar{T}_1]$. 

Now we turn to prove the upper bound \eqref{eqn:max_ub_one_data_oscil}. Note that $yf(\xb;\Wb^{(t)})\le 1$ for $t< \bar{T}_1$, and the definition of $T_{(\vb)}$ then implies that
\begin{align}
    g(\xb,y;\Wb^{(t)})\le 1+ 2\big(\MaxInitIPU\big)^2 \le 1.05.
\end{align}
Now that the sign stability holds for $t\in [0, \bar{T}_1)$, Lemma~\ref{lem: single neuron behaves similarly} gives that
\begin{align}
    \max_{r\in[m]} \dotp{\wb_{y,r}^{(t)}}{y\ub} \le \big(1.05\beta_{\ub}^* m\big)^{1/2}. \label{eq: max ub pre bar T_1}
\end{align}
Here $\beta^*_{\ub}$ is defined in \ref{lem: single neuron behaves similarly} with the superscript $(0)$ dropped.

\paragraph{Step 2.1: Bounding $\max_r\dotp{\wb_{y,r}^{(t)}}{y\ub}$ for $t\in [\bar{T}_1, \munderbar{T}_1)$.} In order for the upper bound to be tight, we need a lower bound on $yf(\xb;\Wb^{(\bar{T}_1-1)})$. 
Let $\tilde\eta = 2\eta \norm{\ub}_2^2 / m>1/2$, we have the following result.
\begin{proposition}\label{prop: f Tk-1 lower bound}
 For every $k\ge 1$, suppose that 
\begin{align}
     E_{\bar{T}_k - 1} \coloneqq \frac{1}{m} \sum_{r\in [m]}\sigma\big(-\dotp{\wb_{-y,r}^{(\bar{T}_k - 1)}}{-y\ub}\big)+\sigma\big(-\dotp{\wb_{-y,r}^{(\bar{T}_k - 1)}}{-y\vb}\big) < \frac{\delta}{2}.
\end{align}
Then we have that
\begin{align}
    yf(\xb;\Wb^{(\bar{T}_k - 1)}) \ge \frac{2+ \tilde{\eta}-\sqrt{\tilde{\eta}^2+4\tilde{\eta}}}{2\tilde{\eta}}.\label{eq: f Tk-1 lower bound}
\end{align}
Moreover, it holds that
\begin{align}
    yf(\xb;\Wb^{(\bar{T}_k )}) \le  yf(\xb;\Wb^{(\bar{T}_k - 1)})\cdot \left(1+ \tilde\eta\cdot  \big( 1- yf( \xb; \Wb^{(\bar{T}_{k}- 1)})\big)\right)^2 + 2 E_{\bar{T}_k - 1 }. \label{eq: f Tk upper bound E_t}
\end{align}
Also, it is notable that the result here still holds for $\bar{T}_{k}\ge  T_{(\vb)}$. 
\end{proposition}
\begin{proof}[Proof of Proposition~\ref{prop: f Tk-1 lower bound}]
    See Appendix~\ref{subsubsec: proof: f Tk-1 lower bound} for a detailed proof.
\end{proof}

Clearly, the conditions required for Proposition~\ref{prop: f Tk-1 lower bound} is true for before $\bar{T}_{1}$. 
So consider a one-step gradient descent at $t= \bar{T}_1 -1$, by Proposition~\ref{prop: f Tk-1 lower bound} and Equation~\eqref{eq: ip positive gd one data},
\begin{align}
    \dotp{\wb_{y,r^*}^{(\bar{T}_1)}}{y\ub} &= \dotp{\wb_{y,r^*}^{(\bar{T}_1 - 1)}}{y\ub} + \frac{\eta\norm{\ub}_2^2}{m}\cdot \big(1 - y f(\xb; \Wb^{(\bar{T}_1 - 1)})\big)\cdot \sigma'\big(\dotp{\wb_{y,r^*}^{(\bar{T}_1 - 1)}}{y\ub}\big) \\
    &\le \dotp{\wb_{y,r^*}^{(\bar{T}_1 - 1)}}{y\ub}\cdot \left(1+ \et \left(1-  \frac{\et+2-\sqrt{\et^2+4\et}}{2\et}\right)\right)  \\
    &\le  1.5 \cdot \big(1.05\beta_{\ub}^* m\big )^{1/2}.\label{eq: max ub bar T_1} 
\end{align}
Here the first inequality above comes from Inequality~\eqref{eq: f Tk-1 lower bound}, and the second inequality is derived from taking the suprema $\et = 1/2$ and Inequality~\eqref{eq: max ub pre bar T_1}. 
Additionally, we can further deliver an upper bound on $yf(\xb;\Wb^{(\bar{T}_1)})$. 
Inequality~\eqref{eq min wu negative lower bound} and Inequality~\eqref{eq min wv negative lower bound} that have been proved to be true on $[0, \bar{T}_{1}]$ indicate that $E_{\bar{T}_1}=\widetilde{\cO}(\sigma_0^2 \norm{\ub}_2^2)\ll 1 $. 
Combining with Inequality~\eqref{eq: f Tk upper bound E_t}, it holds that
\begin{align}
    yf(\xb;\Wb^{(\bar{T}_1)}) &\le  yf(\xb;\Wb^{(\bar{T}_k - 1)})\cdot \left(1+ \tilde\eta\cdot \big( 1- yf( \xb; \Wb^{(\bar{T}_{k}- 1)})\big)\right)^2 + o(1)\\
    &\le \left(1+ \tilde\eta \cdot\big( 1- yf( \xb; \Wb^{(\bar{T}_{k}- 1)})\big)\right)^2 + o(1)\\
    &\le \left(1+ \big( \et - 1-\et/2  + \sqrt{\et^2/4+\et} \big)\right)^2 +o(1)\\
    &=  \bigl(\et/2+\sqrt{\et^2/4+\et}\bigr)^2 +o(1).\label{eq: f Tk upper bound} 
\end{align}
Since $\et<1$, we know that $yf(\xb;\Wb^{(\bar{T}_1)})\le3$ and $ |yf(\xb;\Wb^{(\bar{T}_1)})-1|\le 2$. 
(We keep $\et$ in the upper bound above for deriving a sufficient condition on $\et$ for the sign stability later.) 

Now we look to the lower bound. 
The definition of $\bar{T}_1$, $T_{(\vb)}$ and Assumption~\ref{ass:oscilation_one_data} implies that $yf(\xb;\Wb^{(t)})>1+\delta$, hence $g(\xb,y;\Wb^{(t )})>1+\delta-\delta =1 $ for $t\in [\bar{T}_1, \munderbar{T}_1)$. 
Thus Lemma~\ref{lem: single neuron behaves similarly} implies that
\begin{align}
    \dotp{\wb_{y,r^*}^{(t)}}{y\ub} > (\beta^*_\ub m)^{1/2},\quad t\in [\bar{T}_1, \munderbar{T}_1).
\end{align}

\paragraph{Step 2.2: Lower bounding $\dotp{\wb_{y,r}^{(\munderbar{T}_1)}}{y \ub}$.}

Note that $yf(\xb;\Wb^{(t)})\le yf(\xb;\Wb^{(\bar{T}_1)})$ from the local monotonicity. Consider doing one-step gradient descent with Equation~\eqref{eq: ip positive gd one data}:
\begin{align}
\dotp{\wb_{y,r^*}^{(\munderbar{T}_1)}}{y\ub} &= \dotp{\wb_{y,r^*}^{(\munderbar{T}_1 - 1)}} {y\ub} \cdot \left( 1- \et \cdot \big(yf(\xb;\Wb^{(\munderbar{T}_1 -1)}) - 1 \big)\right)\\ 
&\ge  \dotp{\wb_{y,r^*}^{(\munderbar{T}_1 - 1)}} {y\ub} \cdot \left( 1- \et \cdot \big(yf(\xb;\Wb^{(\bar{T}_1 -1)}) - 1 \big)\right)\\
&\ge (\beta^*_\ub m)^{1/2}\cdot \underbrace{\left(1- \et \Big(\big(\et/2+ \sqrt{\et^2/4+\et}\big)^2 -1 \Big)-o(1)\right)}_{>0 \text{ with }\et < 4/5}  \label{eq: w t1 lower bound}.
\end{align}
Here the last inequality is a consequence of Inequality~\eqref{eq: f Tk upper bound} and the deterministic estimation that 
\begin{align}
    \min_{\et \in [1/2,4/5]}\Big\{1- \et \Big(\big(\et/2+ \sqrt{\et^2/4+\et}\big)^2 -1 \Big)\Big\}>1/4. \label{eq: et sign lower bound}
\end{align}
We note that every step above is free of the lower bound in Inequality~\eqref{eqn:min_lb_one_data_oscil}. Therefore, the sign stability is true on $[\bar{T}_1, \munderbar{T}_1]$, because all the inner products are non-decreasing in $t\in [\munderbar{T}_1, \bar{T}_2]$

\paragraph{Step 2.3: Lower bounding 
$\dotp{\wb_{-y,r}^{(\munderbar{T}_1)}}{-y \ub}$ and $\dotp{\wb_{-y,r}^{(\munderbar{T}_1)}}{-y \vb}$.}

The key is to notice that the sign stability on $[0, \munderbar{T}_1]$ guarantees that for every $t\le \munderbar{T}_1 - 1$, 
\begin{align}
    1\pm \et \cdot  \big(1- yf(\xb; \Wb^{(t)})\big)>0,\quad 1\pm \et \cdot \frac{\norm{\vb}_2^2}{\norm{\ub}_2^2} \cdot \big(1- yf(\xb; \Wb^{(t)})\big)>0.
\end{align}
From Equation~\eqref{eq: ip negative gd one data}, we have that
\begin{align}
    \dotp{\wb_{-y,r}^{(\munderbar{T}_1)}}{-y \ub} &= \dotp{\wb_{-y,}^{(0)}}{-y \ub}\cdot\prod_{t=0}^{\munderbar{T}_1 - 1} \left(1-\et \big(1- yf(\xb; \Wb^{(t)})\big)\right) \\
    &\ge\dotp{\wb_{-y,}^{(0)}}{-y \ub}\cdot\exp\left\{ -\et\sum_{t=0}^{\munderbar{T}_1 - 1}  \big(1- yf(\xb; \Wb^{(t)})\big)\right\}.  \label{eq: <w,-u>  lower bound}
\end{align}
On the other hand,
\begin{align}
\frac{\dotp{\wb_{y,r^*}^{(\munderbar{T}_1)}}{y\ub}}{\dotp{\wb_{y,r^*}^{(0)}}{y\ub}} & = \prod_{t=0}^{\munderbar{T}_1 - 1} \left(1+\et\cdot \big(1- yf(\xb; \Wb^{(t)})\big)\right)\le \exp\left\{\sum_{t=0}^{\munderbar{T}_1 - 1}\et \cdot \big(1- yf(\xb; \Wb^{(t)})\big) \right\}. \label{eq: 1-yf sum pos}
\end{align}
However, $\dotp{\wb_{y,r^*}^{(0)}}{y\ub} \le   (\beta^*_\ub m )^{1/2}\cdot\MaxInitIPU$ and Inequality~\eqref{eq: w t1 lower bound} imply that 
\begin{align}
    \frac{\dotp{\wb_{y,r^*}^{(\munderbar{T}_1)}}{y\ub}}{\dotp{\wb_{y,r^*}^{(0)}}{y\ub}} \ge \frac{(\beta^*_\ub m )^{1/2}\cdot (1/4 - o(1))}{ (\beta^*_\ub m )^{1/2}\cdot\MaxInitIPU} \ge 1. \label{eq: ratio lower bound}
\end{align}
Combining Inequalities~\eqref{eq: 1-yf sum pos} and~\eqref{eq: ratio lower bound}, we can obtain that $\sum_{t=0}^{\munderbar{T}_1 - 1 } \big(1- yf(\xb;\Wb^{(t)})\big) \ge 0$. Together with Inequality~\eqref{eq: <w,-u>  lower bound}, this in turns leads to
\begin{align}
     0> \min_{r\in[m]} \dotp{\wb_{-y,r}^{(\munderbar{T}_1)}}{-y \ub}&= \min_{r\in[m]}\dotp{\wb_{-y,r}^{(0)}}{-y \ub}\cdot\exp\left\{ -\et\sum_{t=0}^{\munderbar{T}_1 - 1}  \big(1- yf(\xb; \Wb^{(t)})\big)\right\}\\
    &\ge \min_{r\in[m]}  \dotp{\wb_{-y,r}^{(0)}}{-y \ub} \ge -\MaxInitIPU.
\end{align}
Analogously we have that 
\begin{align}
     0> \min_{r\in[m]} \dotp{\wb_{-y,r}^{(\munderbar{T}_1)}}{-y \vb}&=\dotp{\wb_{-y,r}^{(0)}}{-y \vb}\cdot\exp\left\{ -\et\cdot\frac{\|\mathbf{v}\|_2^2}{\|\mathbf{u}\|_2^2}\cdot \sum_{t=0}^{\munderbar{T}_1 - 1}  \big(1- yf(\xb; \Wb^{(t)})\big)\right\}.\\
    &\ge \min_{r\in[m]} \dotp{\wb_{-y,r}^{(0)}}{-y \vb}\ge -\MaxInitIPV.
\end{align}
In conclusion, we have that 
\begin{align}
\max_{r\in[m]} \dotp{\wb_{y,r}^{(t)}}{y\ub} \ge 0.2 (\beta^*_\ub m)^{1/2},\quad &t\in [\bar{T}_1, \bar{T}_2],\\\min_{r\in[m]} \dotp{\wb_{-y,r}^{(t)}}{-y\ub}\ge -\MaxInitIPU,\quad & t\in [\bar{T}_1, \bar{T}_2],\\
\min_{r\in[m]} \dotp{\wb_{-y,r}^{(t)}}{-y\vb}\ge -\MaxInitIPV,\quad & t\in [\bar{T}_1, \bar{T}_2].
\end{align}
Additionally, the lower bound on $\max_{r\in[m]} \dotp{\wb_{y,r}^{(t)}}{y\ub} $ indicates that the sign stability holds on $[\bar{T}_1, \bar{T}_2)$, since the last drop step does not change the sign of $\dotp{\wb^{(t)}_{y,r}}{y\ub}$. 
Furthermore, once the $\ub$-sign stability holds, the $\vb$-sign stability can be easily derived. 
To see this, we note that the $\ub$-sign stability implies that, for any $t\in [0,T_{(\vb)}]$, it holds that $1\pm   2\eta \norm{\ub}_2^2 \cdot (1-yf(\xb;\Wb^{(t)}))/m > 0$. Now that $\norm{\ub}_2> \norm{\vb}_2$, one clearly sees that $1\pm   2\eta \norm{\vb}_2^2  \cdot (1-yf(\xb;\Wb^{(t)}))/m > 0$ and the $\vb$-sign stability holds. 
\paragraph{Step 2.4: Upper bounding $\dotp{\wb_{y,r}^{(t)}}{y \ub}$ for $t\in (\munderbar{T}_1, \bar{T}_2)$.} This is exactly the same as the proof of Inequality~\eqref{eq: max ub pre bar T_1} in \textbf{Step 1}, and is thus omitted.

\paragraph{Step 3. Finalizing proofs.} At this point, all the results in Lemma~\ref{lem: boundedness oscillating one data} have been proved to be true on $t\in [\bar{T}_1, \bar{T}_2)$. It is important to note that the only inductive hypothesis used for the local extension on $[\bar{T}_1, \bar{T}_2)$ is the lower bound on $\dotp{\wb_{-y,r}^{(t)}}{-y\ub}$ for $t\le \bar{T}_1$. 
The rest part merely comes from the definition of $\bar{T}_1$ and $\munderbar{T}_2$ and Assumption~\ref{ass:oscilation_one_data}. 
So repeating the same argument extends previous steps to all $t\leq T_{(\mathbf{v})}$.
\end{proof}

\begin{remark}\label{rmk: finite period} In the proof, we implicitly utilize the fact that $\bar{T}_k,\munderbar{T}_k<+\infty, \; \forall k \ge 0$. One may conjecture that there could be cases that for some $k$, $yf(\xb; \Wb^{(t)})> 1$ (resp. $<1$) for all $t \ge \bar{T}_{k}$ (resp. $\munderbar{T}_k$). However, Assumption~\ref{ass:oscilation_one_data} (guaranteed by tuning a proper $\eta$) indicates that this cannot happen. One can argue that once $yf(\xb; \Wb^{(t)})> 1$, the Assumption~\ref{ass:oscilation_one_data} enables the dynamic to bounce back towards $1$ with at least exponential rate. Therefore, $yf(\xb; \Wb^{(t)})$ falls below $1+\delta$ within a few steps and Assumption~\ref{ass:oscilation_one_data} forces $yf(\xb; \Wb^{(t)})<1-\delta$, and $\munderbar{T}_{k}<+\infty$. Same argument can be used to prove that $\bar{T}_{k}<+\infty$. 
\end{remark}

\subsection{Proof of Fundamental Reasoning (Lemma~\ref{lem: linear increment one data})}\label{subsec: proof increment}

\begin{proof}[Proof of Lemma~\ref{lem: linear increment one data}]
Let $r^* := \argmax_{r\in[m]} \dotp{\wb_{y,r}^{(t_0)}}{y\ub}$. For the step $t_0< t_1$, \eqref{eq: ip positive gd one data} implies that
\begin{align}
\dotp{\wb_{y,r^*}^{(t_1)}}{y\ub} &= \dotp{\wb_{y,r^*}^{(t_0 )}}{y\ub}  + \frac{2\eta \norm{\ub}_2^2}{m} \cdot \sum_{\substack{s\in[t_0,t_1-1]:\\ yf(\xb;\Wb^{(s)})\ge 1}} \big(1-yf(\xb;\Wb^{(s)})\big)\cdot \dotp{\wb_{y,r^*}^{(s)}}{ y\ub } \\
&\qquad + \frac{2\eta \norm{\ub}_2^2}{m} \cdot \sum_{\substack{s\in[t_0,t_1-1]:\\ yf(\xb;\Wb^{(s)})<1}} \big(1-yf(\xb;\Wb^{(s)})\big)\cdot \dotp{\wb_{y,r^*}^{(s)}}{ y\ub } . \label{eq: diff expand}
\end{align}
By Condition 4 in Lemma~\ref{lem: linear increment one data}, for all $s\in [t_0,t_1]$, $\frac{1}{m} \sum_{r\in[m]}  \sigma(\dotp{\wb_{y,r}^{(s)}}{y\vb})<\delta$. 
Therefore by Assumption~\ref{ass:oscilation_one_data}, for $s$ such that $yf(\xb;\Wb^{(s)}) \ge 1$ (and thus $>1+\delta$), it holds from \eqref{eq: f expression one data simplified} that 
\begin{align}
    g(\xb,y;\Wb^{(t)}) = \frac{1}{m}\sum_{r\in[m]}\sigma\big(\dotp{\wb_{y,r}^{(s)}}{y\ub}\big)\geq 1+\delta - \delta = 1.
\end{align}
Hence by Lemma~\ref{lem: similar behavior}, we have that
\begin{align}
    \dotp{\wb_{y,r^*}^{(s)}}{y\ub} \ge (\beta^*_\ub m)^{1/2 }.  \label{eq: ws lower}
\end{align}
On the other hand, by Conditions 3 in Lemma~\ref{lem: linear increment one data}, for all $s\in [t_0,t_1]$, it holds that $\min_{r\in[m]} \dotp{\wb_{-y,r}^{(s)}}{-y \ub}>-0.1$ and $\min_{r\in[m]} \dotp{\wb_{-y,r}^{(s)}}{-y \vb}>-0.1$, which lead to
\begin{align}
    \frac{1}{m}\sum_{r\in [m]}  \sigma\big(-\dotp{\wb_{-y,r}^{(s)}}{-y\ub }\big) + \sigma\big(-\dotp{\wb_{-y,r}^{(s)}}{-y\vb }\big)\le 2\times 0.1^2 < 0.05.
\end{align}
Therefore by Assumption~\ref{ass:oscilation_one_data}, for $s$ such that $yf(\xb;\Wb^{(s)}) \le 1$ (and thus $<1-\delta$), it holds from \eqref{eq: f expression one data simplified} that 
\begin{align}
     g(\xb,y;\Wb^{(t)}) = \frac{1}{m}\sum_{r\in[m]}\sigma\big(\dotp{\wb_{y,r}^{(s)}}{y\ub}\big)\leq 1 - \delta + 0.05.
\end{align}
Hence by Lemma~\ref{lem: similar behavior}, we know that
\begin{align}
    \dotp{\wb_{y,r^*}^{(s)}}{y\ub} \le (1.05-\delta)^{1/2} \cdot(\beta^*_\ub m)^{1/2 }. \label{eq: ws upper}
\end{align}
Meanwhile, Condition 1 ($\ub$-sign stability) and Condition 2 (boundedness) in Lemma~\ref{lem: linear increment one data} imply that
\begin{align}
    \left|\dotp{\wb_{y,r^*}^{(t_1)}}{y\ub}  -  \dotp{\wb_{y,r^*}^{(t_0)}}{y\ub}  \right| \le  \left|\dotp{\wb_{y,r^*}^{(t_1)}}{y\ub}\right| \vee\left| \dotp{\wb_{y,r^*}^{(t_0)}}{y\ub} \right|\le B\cdot (\beta^*_\ub m)^{1/2}.\label{eq: diff upper}
\end{align}
Putting Inequality~\eqref{eq: ws lower}, \eqref{eq: ws upper}, \eqref{eq: diff upper} and Equation~\eqref{eq: diff expand} together, we have that
\begin{align}
B\cdot (\beta^*_\ub m)^{1/2} &\ge  \left|  \frac{2\eta \norm{\ub}_2^2}{m} \cdot \sum_{\substack{s\in[t_0,t_1-1]:\\ yf(\xb;\Wb^{(s)})\ge 1}} \big(1-yf(\xb;\Wb^{(s)})\big)\cdot \dotp{\wb_{y,r^*}^{(s)}}{ y\ub } \right.\\
&\qquad \left.+ \frac{2\eta \norm{\ub}_2^2}{m} \cdot \sum_{\substack{s\in[t_0,t_1-1]:\\ yf(\xb;\Wb^{(s)})<1}} \big(1-yf(\xb;\Wb^{(s)})\big)\cdot \dotp{\wb_{y,r^*}^{(s)}}{ y\ub } \right| \\
& \ge   \frac{2\eta \norm{\ub}_2^2}{m} \cdot \sum_{\substack{s\in[t_0,t_1-1]:\\ yf(\xb;\Wb^{(s)})\ge 1}} \big(yf(\xb;\Wb^{(s)})-1\big)\cdot \dotp{\wb_{y,r^*}^{(s)}}{ y\ub } \\
&\qquad - \frac{2\eta \norm{\ub}_2^2}{m} \cdot \sum_{\substack{s\in[t_0,t_1-1]:\\ yf(\xb;\Wb^{(s)})<1}} \big(1-yf(\xb;\Wb^{(s)})\big)\cdot \dotp{\wb_{y,r^*}^{(s)}}{ y\ub } \\
&\ge   \frac{2\eta \norm{\ub}_2^2 }{m}\cdot (\beta^*_\ub m)^{1/2} \cdot \sum_{\substack{s\in[t_0,t_1-1]:\\ yf(\xb;\Wb^{(s)})\ge 1}} \big(yf(\xb;\Wb^{(s)})-1\big) \\
&\qquad-\frac{2\eta \norm{\ub}_2^2 }{m}\cdot (\beta^*_\ub m)^{1/2} \cdot(1.05-\delta)^{1/2} \cdot\sum_{\substack{s\in[t_0,t_1-1]:\\ yf(\xb;\Wb^{(s)})<1}} \big(1-yf(\xb;\Wb^{(s)})\big).
\end{align}
This is also equivalent to
\begin{align}
\sum_{\substack{s\in[t_0,t_1-1]:\\ yf(\xb;\Wb^{(s)})<1}} \big(1-yf(\xb;\Wb^{(s)})\big)\cdot \dotp{\wb_{y,r^*}^{(s)}}{ y\ub }& \ge \sum_{\substack{s\in[t_0,t_1-1]:\\ yf(\xb;\Wb^{(s)})\ge 1}}(1.05-\delta)^{-1/2}\cdot \big(yf(\xb;\Wb^{(s)})-1\big)  \\
& \qquad - \underbrace{\frac{m B}{2\eta \norm{\ub}_2^2 \sqrt{1.05-\delta}}}_{\coloneqq \Delta(B)}.\label{eq: balanced summation}
\end{align}
Now we are ready to lower bound the summation that we are interested in.
Since 
\begin{align}
    |\{s\in[t_0,t_1-1]:yf(\mathbf{x};\mathbf{W}^{(s)})<1\}| + |\{s\in[t_0,t_1-1]:yf(\mathbf{x};\mathbf{W}^{(s)})>1\}| = t_1 - t_0,
\end{align}
we have either $|\{s\in[t_0,t_1-1]:yf(\mathbf{x};\mathbf{W}^{(s)})>1\}|> (t_1 - t_0) / 2$, which by \eqref{eq: balanced summation} implies that 
\allowdisplaybreaks
\begin{align}
    \sum_{s=t_0}^{t_1-1}\big(1 - yf(\mathbf{x};\mathbf{W}^{(s)})\big) &= \sum_{\substack{s\in[t_0,t_1-1]:\\ yf(\xb;\Wb^{(s)})< 1}}\big(1 - yf(\mathbf{x};\mathbf{W}^{(s)})\big) - \sum_{\substack{s\in[t_0,t_1-1]:\\ yf(\xb;\Wb^{(s)})> 1}}\big(yf(\mathbf{x};\mathbf{W}^{(s)}) - 1\big)  \\ 
    &\geq \big((1.05-\delta)^{-1/2} - 1\big)\cdot \sum_{\substack{s\in[t_0,t_1-1]:\\ yf(\xb;\Wb^{(s)})> 1}}\big(yf(\mathbf{x};\mathbf{W}^{(s)}) - 1\big)  - \Delta(B)  \\
    &\geq \frac{1}{2}\big(\delta(1.05-\delta)^{-1/2} - \delta\big)\cdot (t_1 - t_0) - \Delta(B), \label{eq: main 11}
\end{align}
or $|\{s\in[t_0,t_1-1]:yf(\mathbf{x};\mathbf{W}^{(s)})<1\}| > (t - t_0 ) / 2$, which by \eqref{eq: balanced summation} implies that 
\begin{align}
    \sum_{s=t_0}^{t_1-1}\big(1 - yf(\mathbf{x};\mathbf{W}^{(s)})\big) &= \sum_{\substack{s\in[t_0,t_1-1]:\\ yf(\xb;\Wb^{(s)})< 1}}\big(yf(\mathbf{x};\mathbf{W}^{(s)}) - 1\big)  - \sum_{\substack{s\in[t_0,t_1-1]:\\ yf(\xb;\Wb^{(s)})> 1}}\big(1 - yf(\mathbf{x};\mathbf{W}^{(s)})\big) \notag \\ 
    &\geq \big(1 - (1.05-\delta)^{1/2}\big)\cdot \sum_{\substack{s\in[t_0,t_1-1]:\\ yf(\xb;\Wb^{(s)})< 1}}\big(1 - yf(\mathbf{x};\mathbf{W}^{(s)})\big)  - (1.05-\delta)^{1/2}\cdot \Delta(B) \notag \\
    &\geq \frac{1}{2}\big(\delta - \delta(1.05-\delta)^{1/2}\big)\cdot (t_1 - t_0) - (1.05-\delta)^{1/2}\cdot \Delta(B). \label{eq: main 12}
\end{align}
In both two cases we have used Assumption \ref{ass:oscilation_one_data} to bound $yf(\mathbf{x};\mathbf{W}^{(s)})$ from $1$.
Combining \eqref{eq: main 11} and \eqref{eq: main 12}, we have that
\begin{align}\label{eq: main 13}
    \sum_{s=t_0}^{t_1-1}\big(1 - yf(\mathbf{x};\mathbf{W}^{(s)})\big) \geq \frac{1}{2}\big(\delta - \delta(1.05-\delta)^{1/2}\big)\cdot (t_1 - t_0) - \Delta(B).
\end{align}
Now plugging in the definition of $\Delta(B)$ in \eqref{eq: balanced summation}, we have proved the linear increasing lower bound for the summation, which is the first conclusion of Lemma~\ref{lem: linear increment one data}.

Then we turn to prove the second conclusion of Lemma~\ref{lem: linear increment one data}. 
For simplicity, we denote $\alpha := \norm{\vb}^2_2 / \norm{\ub}^2_2$ and $\epsilon = (\delta - \delta(1.05-\delta)^{1/2})/4$. 
Note that from the $\vb$-sign stability (Condition 1) and \eqref{eq: ipv positive gd one data}, we have that
\begin{align}
1+ \frac{2\eta \norm{\vb}_2^2}{m} \cdot \big(1-yf(\xb;\Wb^{(t)})\big) > 0, \quad \forall t\in [t_0,t_1-1].
\end{align}
Therefore, we can lower bound the logarithmic ratio $\dotp{\wb_{y,r}^{(t)}}{y\vb}/ \dotp{\wb_{y,r}^{(0)}}{y\vb}$ for $r\in \cV_{y,+}$ as 
(recall that $\et = 2\eta\|\ub\|_2^2 / m$)
\begin{align}
    &\sum_{t=t_0}^{t_1-1} \log \Big(1+\alpha \et \big(1- yf(\xb;\Wb^{(t)})\big)\Big)\\
    &\qquad = \sum_{t=t_0}^{t_1-1} \int_0^\alpha \frac{\et\big(1- yf(\xb;\Wb^{(t)})\big) }{1+\et z \big(1- yf(\xb;\Wb^{(t)})\big)} dz \\
    &\qquad = \sum_{t=t_0}^{t_1-1} \int_0^\alpha \frac{\et\cdot\Big(\big(1- yf(\xb;\Wb^{(t)})\big) + 2\Big) }{1+\et z \big(1- yf(\xb;\Wb^{(t)})\big)} dz 
    -  \sum_{t=t_0}^{t_1-1} \int_0^\alpha \frac{2\et }{1+\et z \big(1- yf(\xb;\Wb^{(t)})\big)} dz. \label{eq: lower bound weak signal increasing}
\end{align}
Note that by Condition 5 in Lemma~\ref{lemma:multiple-data, lower bound fitting residual}, $-2\le 1-yf(\xb;\Wb^{(t)})\le 1$, which further lower bounds \eqref{eq: lower bound weak signal increasing} as 
\begin{align}
    \eqref{eq: lower bound weak signal increasing}&\ge  \sum_{t=t_0}^{t_1-1} \int_0^\alpha \frac{\et\cdot\Big(\big(1- yf(\xb;\Wb^{(t)})\big) + 2\Big) }{1+\et z} dz -  \sum_{t=t_0}^{t_1-1} \int_0^\alpha \frac{2\et }{1-2\et z} dz \\ 
    &= \int_0^\alpha \frac{\et\cdot\Big(\sum_{t=t_0}^{t_1-1}\big(1- yf(\xb;\Wb^{(t)})\big) + 2(t_1-t_0)\Big) }{1+\et z} dz -  \int_0^\alpha \frac{2\et (t_1 - t_0 )}{1-2\et z} dz.\label{eq: lower bound weak signal increasing 2}
\end{align}
Now applying \eqref{eq: main 13} to the summation in \eqref{eq: lower bound weak signal increasing 2}, we can arrive at 
\begin{align}
    \eqref{eq: lower bound weak signal increasing 2}&\ge \int_0^\alpha \frac{\et\cdot\Big( 2\epsilon (t_1 - t_0) - \Delta(B) + 2(t_1-t_0)\Big) }{1+\et z} dz -  \int_0^\alpha \frac{2\et (t_1 - t_0 )}{1-2\et z} dz \\ 
    &\ge \Big( 2\epsilon(t_1 -t_0 ) - \Delta(B)  + 2(t_1-t_0)\Big)\cdot\log(1+\alpha\et) + (t_1-t_0)\cdot\log(1-2\alpha\et )\\
    &\ge \big(2\epsilon(t_1 -t_0 ) - \Delta(B) \big)\cdot\log(1+\alpha\et)  + (t_1-t_0)\cdot\log\Big((1+\alpha\et)^2\cdot(1-2\alpha\et )\Big).\label{eq: lower bound weak signal increasing 3}
\end{align}
To further lower bound \eqref{eq: lower bound weak signal increasing 3}, we note that $\alpha\et >\log(1+\alpha\et)\ge \frac{1}{2}\alpha\et$ since by Assumption~\ref{cond:model_params_one_data}, $0<\alpha \et < 1$.  
Furthermore, Assumption~\ref{cond:model_params_one_data} guarantees that $\alpha < \epsilon/2 = (\delta - \delta \cdot (1.05-\delta)^{1/2})/8$ and $2\alpha^3\et^3+3\alpha^2 \et^2 \le 4 \alpha^2 < 1/5 \wedge 2\epsilon/5$, so we have that
\begin{align}
    \log\Big((1+\alpha\et)^2\cdot(1-2\alpha\et )\Big) &= \log \big(1- 3\alpha^2\et^2 -2\alpha^3\et^3\big)\\
    &= - \log \left(1+ \frac{3\alpha^2\et^2 +2\alpha^3\et^3}{1-3\alpha^2\et^2 -2\alpha^3\et^3}\right)\\ 
    &\ge \frac{-3\alpha^2\et^2 -2\alpha^3\et^3 }{1-3\alpha^2\et^2 -2\alpha^3\et^3}\\
    &\ge -5\alpha^2.\label{eq: lower bound weak signal increasing 4}
\end{align}
Putting \eqref{eq: lower bound weak signal increasing 3} and \eqref{eq: lower bound weak signal increasing 4} together, we have that
\begin{align}
    \sum_{t=t_0}^{t_1-1} \log \Big(1+\alpha \et \big(1- yf(\xb;\Wb^{(t)})\big)\Big)&\ge \big(\alpha \et \epsilon -5\alpha^2 \big)\cdot(t_1-t_0)-\Delta(B) \cdot \log(1+\alpha \et) \\
    &\ge \frac{1}{2}\alpha \et \epsilon\cdot (t_1- t_0 )-\Delta(B) \alpha \et\label{eq: lower bound weak signal increasing 5}
\end{align}
And consequently we have the following result, for all $r\in\mathcal{V}_{y,+}$,
\begin{align}
    \dotp{\wb_{y,r}^{(t_1)}}{y\vb} &=\dotp{\wb_{y,r}^{(t_0)}}{y\vb} \cdot\prod _{t=t_0}^{t_1 - 1} \Big(1+\alpha\et\big(1- yf(\xb;\Wb^{(t)})\big)\Big) \\
    & = \dotp{\wb_{y,r}^{(t_0)}}{y\vb} \cdot\exp\left\{\sum_{t = t_0}^{t_1 -1 }\log\Big(1+\alpha\et\big(1- yf(\xb;\Wb^{(t)})\big)\Big)\right\}\\
    &\ge\dotp{\wb_{y,r}^{(t_0)}}{y\vb}\cdot  \exp\left\{ \frac{1}{2}\alpha\et\epsilon\cdot (t_1 - t_0) - \Delta(B) \alpha \et \right\},\label{eq: lower bound weak signal increasing 6}
\end{align}
where in the last inequality we use \eqref{eq: lower bound weak signal increasing 5}.
Plugging in the expression of $\epsilon$, $\alpha$, $\et$, and $\Delta(B)$, we have proved the second conclusion of Lemma~\ref{lem: linear increment one data}.
This concludes the proof of Lemma~\ref{lem: linear increment one data}.
\end{proof}

\subsection{Proof of Technical Results}\label{subsec: proof of technical results single data}

\subsubsection{Proof of Proposition~\ref{prop: doesn't decrease}}\label{subsubsec: proof: doesn't decrease}

\begin{proof}[Proof of Proposition~\ref{prop: doesn't decrease}]

We want to track the sequence \eqref{eq: sequence} after it falls below $\delta$. To this end, we define two stopping times
\begin{align}
T_{(\vb), \delta, L} &= \min\left\{t \ge T_{(\vb)}:\frac{1}{m} \sum_{r\in[m]}  \sigma\big(\dotp{\wb_{y,r}^{(t)}}{y\vb}\big) < \delta \right\},\\
T_{(\vb)}^{+,2 } &= \min\left\{t \ge T_{(\vb), \delta, L}:\frac{1}{m} \sum_{r\in[m]}  \sigma\big(\dotp{\wb_{y,r}^{(t)}}{y\vb}\big) \ge  \delta \right\}.
\end{align}
If $T_{(\vb), \delta, L}=+\infty$, then the proof is over. 
Otherwise we prove that, before $T_{(\vb)}^{+,2 } \le +\infty $ (possibly equal), the sequence never falls below $\delta/2$.

Let's take a closer look at the controls over the negative parts while the weak signal remain learned. Note that for $t\in [T_{(\vb)},T_{(\vb), \delta, L}]$ and $r\in \cV_{y,+}$, we have that
\begin{align}
    \dotp{\wb_{y,r}^{(t)}}{y\vb} &= \dotp{\wb_{y,r}^{(0)}}{ y\vb} \cdot \prod_{s=0}^{t-1} \left(1+\frac{2\eta \norm{\vb}}{m} \cdot \big(1- yf(\xb;\Wb^{(s)})\big)\right) \\ &\le \dotp{\wb_{y,r}^{(0)}}{ y\vb} \cdot  \exp \left\{ \frac{2\eta \norm{\vb}_2^2}{m}\cdot\sum_{s=0}^{t-1} \big(1-yf(\xb;\Wb^{(s)})\big)\right\}.
\end{align}
Meanwhile, for $t\in [T_{(\vb)},T_{(\vb), \delta, L}]$, we have that $\dotp{\wb_{y,r}^{(t)}}{y\vb}  > (\beta_\vb^* m\delta )^{1/2} \gg  \dotp{\wb_{y,r}^{(0)}}{ y\vb}$. Hence 
\begin{align}
    \exp \left\{ \frac{2\eta \norm{\vb}_2^2}{m}\cdot\sum_{s=0}^{t-1} \big(1-yf(\xb;\Wb^{(s)})\big)\right\}  >1,  \quad \forall t\in [T_{(\vb )},T_{(\vb), \delta, L}-1],
\end{align}
and in consequence we have that 
\begin{align}
    \sum_{s=0}^{t-1} \big(1-yf(\xb;\Wb^{(s)})\big) > 0, \quad \forall t\in [T_{(\vb )},T_{(\vb), \delta, L}-1].\label{eq: doesn't decrease 1}
\end{align}
On the other hand, for $r\in \cV_{-y,-}$, it holds that $\dotp{\wb_{-y,r}^{(0)}}{-y\vb}<0$ and thus by \eqref{eq: doesn't decrease 1} we have that 
\begin{align}
    \dotp{\wb_{-y,r}^{(t)}}{-y\vb} &= \dotp{\wb_{-y,r}^{(0)}}{-y\vb} \cdot \prod_{s=0}^{t-1} \left(1- \frac{2\eta \norm{\vb}_2^2 }{m}\cdot \big(1-yf(\xb;\Wb^{(s)})\big)\right) \\
    & \ge \dotp{\wb_{-y,r}^{(0)}}{-y\vb } \cdot \exp \left\{ -\frac{2\eta \norm{\vb}_2^2 }{m}\cdot \sum_{s=0}^{t-1} \big(1-yf(\xb;\Wb^{(s)})\big)\right\}\\
    & \ge \dotp{\wb_{-y,r}^{(0)}}{-y\vb } \ge -\MaxInitIPV, \label{eq: negative part lower bound}
\end{align}
for all $t\in [T_{(\vb )},T_{(\vb), \delta, L}]$. Analogously, we also have that 
\begin{align}
    \dotp{\wb_{-y,r}^{(t)}}{-y\ub}  \ge \dotp{\wb_{-y,r}^{(0)}}{-y\ub}\ge -\MaxInitIPU\label{eq: negative part lower bound u}
\end{align}
for all $t\in [T_{(\vb )},T_{(\vb), \delta, L}]$. Therefore, we have that for all $t\in [T_{(\vb )},T_{(\vb), \delta, L}]$,
\begin{align}
    E_t&\coloneqq \frac{1}{m} \sum_{r\in [m]} \sigma\big(-\dotp{\wb_{-y,r}^{(t)}}{-y\ub}\big) + \sigma\big(-\dotp{\wb_{-y,r}^{(t)}}{-y\vb}\big) \\ 
& \le  2  \cdot \max_{r \in [m]}\left\{\sigma\big(-\dotp{\wb_{-y,r}^{(t)}}{-y\ub}\big) \vee \sigma\big(-\dotp{\wb_{-y,r}^{(t)}}{-y\vb}\big)\right\} \\ 
&\le 2\big(\MaxInitIPU\big)^2  \\ 
&\ll \frac{\delta}{2}. \label{eq: Et upper bound}
\end{align}
This allows us to leverage Proposition~\ref{prop: f Tk-1 lower bound} to upper bound $yf(\xb;\Wb^{(t)})$ for $t\in [T_{(\vb )},T_{(\vb), \delta, L}]$. 
Specifically, 
we locate the last step before $T_{(\vb), \delta, L}$ when $yf$ just bounces up over $1$, which is, 
\begin{align}
    \bar{T}_{k^*}\coloneqq \max\left\{\bar{T}_k: \bar{T}_k\le T_{(\vb), \delta, L}\right\},
\end{align}
where $\bar{T}_k$ is defined in \eqref{eq: T bar}.
Then Proposition~\ref{prop: f Tk-1 lower bound} with Inequality~\eqref{eq: Et upper bound} implies that
\begin{align}
    yf(\xb;\Wb^{(\bar{T}_{k^*})}) &\le yf(\xb;\Wb^{(\bar{T}_{k^*} - 1)})\cdot \Big(1+ \tilde\eta \big( 1- yf( \xb; \Wb^{(\bar{T}_{k}- 1)})\big)\Big)^2 + 2 E_{\bar{T}_{k^*} - 1 }\\
    &\le\Big( 1+ \et \big(1- yf(\xb;\Wb^{(\bar{T}_{k^*}-1)})\big)\Big)^2 + o(1)\\
    &\le\Big( {1+ \et \Big(1- \frac{1}{2\et}\big(\et + 2 - \sqrt{\et^2 +4\et}\big)\Big)} \Big)^2 + o(1) \\
    & \le 3, \label{eq: Et upper bound 2}
\end{align}
where we have applied the fact that $\et<1$. 
On the other hand, we have that 
\begin{align}
    \frac{1}{m} \sum_{r\in [m]} \sigma\big(\dotp{\wb_{y,r}^{(\bar{T}_{k^*} - 1)}}{ y\ub}\big) \le 1.05,
\end{align}
for which Lemma~\ref{lem: similar behavior} indicates that $\max_{r\in [m]} \dotp{\wb_{y,r}^{(\bar{T}_{k^*} - 1)}}{ y\ub} \le (1.05 \beta^*_\ub m)^{1/2}$. 
As in Inequality~\eqref{eq: max ub bar T_1}, one step gradient descent then gives that
\begin{align}
    \dotp{\wb_{y,r}^{(\bar{T}_{k^*} )}}{ y\ub}  \le \dotp{\wb_{y,r}^{(\bar{T}_{k^*} - 1)}}{ y\ub} \cdot  \left(1+ \frac{2\eta \norm{\ub}_2^2}{m} \cdot\big(1- yf(\xb;\Wb^{(\bar{T}_{k^*} - 1)})\big)\right) \le 1.5\cdot \big(1.05\beta^*_\ub m\big)^{1/2}. 
\end{align}
Now we consider the scale of these inner products right at $T_{(\vb), \delta, L}$. 
From the definitions of $T_{(\vb), \delta, L}$ and $\bar{T}_{k^*}$ we know that $yf_t>1$ between these steps (otherwise the sequence wouldn't fall below $\delta$). 
Thus by \eqref{eq: Et upper bound 2},
\begin{align}
    1< yf(\xb;\Wb^{(t)})<yf(\xb;\Wb^{(\bar{T}_{k^*})})\le 3,\quad \forall t\in [\bar{T}_{k^*},T_{(\vb), \delta, L}].
\end{align}
We first state that $\frac{1}{m}\sum_{r\in [m]} \sigma(\dotp{\wb_{y,r}^{(T_{(\vb), \delta, L})}}{y\vb})$ is not far away from $\delta$. 
Note that by \eqref{eq: doesn't decrease 1},
\begin{align}
\frac{1}{m}\sum_{r\in [m]} \sigma\big(\dotp{\wb_{y,r}^{(T_{(\vb), \delta, L})}}{y\vb}\big)&=  \frac{1}{m} \sum_{r\in [m]} \sigma\big(\dotp{\wb_{y,r}^{(T_{(\vb), \delta, L}-1)}}{y\vb}\big) \cdot \left(1+ \frac{2\eta\norm{\vb}_2^2}{m} \cdot \big(1-yf(\xb;\Wb^{(T_{(\vb), \delta, L} - 1)})\big)\right)\\ 
 & \ge \delta\cdot   \left(1+ \frac{2\eta\norm{\vb}_2^2}{m} \cdot(1-3)\right)\\
  & \ge \delta \cdot \left(1 - 2  \cdot\frac{\norm{\vb}_2^2}{ \norm{\ub}_2^2}\right) \\ 
  &\ge \frac{3}{4}\delta. \label{eq: doesn't decrease 2}
\end{align}
Last inequality uses the fact that $\norm{\vb}_2^2 /\norm{\ub}_2^2\le 1/4$. 
This shows that the sequence does not fall below $3\delta / 4$ at the step $T_{(\vb), \delta, L}$.
In the sequel, we study the behavior of the sequence after step $T_{(\vb), \delta, L}$.

Firstly, by Inequalities~\eqref{eq: negative part lower bound} and \eqref{eq: negative part lower bound u} with $t =  T_{(\vb), \delta, L}$, we obtain that the negative parts satisfy
\begin{align}
    \min_{r\in [m]}  \dotp{\wb_{-y,r}^{(T_{(\vb), \delta, L})}}{-y\ub } \ge -\MaxInitIPU,\label{eq: E_t exit 1}\\
    \min_{r\in [m]}  \dotp{\wb_{-y,r}^{(T_{(\vb), \delta, L})}}{-y\vb } \ge -\MaxInitIPV. \label{eq: E_t exit 2}
\end{align}
which by definition of $E_t$ in \eqref{eq: Et upper bound} implies that 
\begin{align}
    E_{T_{(\vb), \delta, L}} \le 2\left(\MaxInitIPU\right)^2\ll 1.\label{eq: E_t exit 3}
\end{align} 
For the positive part, we have that 
\begin{align}
    \max_{r\in[m]} \dotp{\wb_{y,r}^{(T_{(\vb), \delta, L})}}{y\ub} <\max_{r\in[m]} \dotp{\wb_{y,r}^{(\bar{T}_{k^*})}}{y\ub}< 1.5\cdot (1.05\beta^*_{\ub}m)^{1/2},\label{eq: E_t exit 4}
\end{align}
and by the same argument as Inequality~\eqref{eq: w t1 lower bound} we can obtain the lower bound of $\max_{r\in [m]} \dotp{\wb_{y,r}^{(T_{(\vb), \delta, L})}}{y\ub}$, 
\begin{align}
    \max_{r\in [m]} \dotp{\wb_{y,r}^{(T_{(\vb), \delta, L})}}{y\ub} &\ge \max_{r\in [m]} \dotp{\wb_{y,r}^{(T_{(\vb), \delta, L}-1)}}{y\ub}  \left(1+ \frac{2\eta \norm{\ub}_2^2}{m} \big(1- yf(\xb;\Wb^{(T_{(\vb), \delta, L} - 1)})\big)\right)\\
    &\ge 0.2 \cdot \big(\beta^*_\ub m\big)^{1/2},\label{eq: E_t exit 5}
\end{align}
and finally for $t\in [T_{(\vb), \delta, L}, T_{(\vb)}^{+,2 }]$, we have that 
\begin{align}
    \frac{1}{m}\sum_r \sigma(\dotp{\wb_{y,r}^{(t)}}{y\vb})<\delta.\label{eq: E_t exit 6}
\end{align}
With all these initial conditions \eqref{eq: E_t exit 3}, \eqref{eq: E_t exit 4}, \eqref{eq: E_t exit 5}, and \eqref{eq: E_t exit 6}, one can then consider all the steps $\munderbar{T}_{k^{\prime}}$, $\bar{T}_{k^{\prime}}> T_{(\vb), \delta, L}$ and apply an inductive argument exactly the same as in the proof of Lemma~\ref{lem: boundedness oscillating one data} to conclude that for all $t\in [T_{(\vb), \delta, L}, T_{(\vb)}^{+,2 }]$ it holds that 
\begin{align}
    0.2 \cdot (\beta^*_{\ub }m)^{1/2}&\leq \max_{r} \dotp{\wb_{y,r}^{(t)}}{y\ub}\le 1.5 \cdot (1.05\beta^*_{\ub } m)^{1/2}, \\
    \min_{r} \dotp{-\wb_{-y,r}^{(t)}}{-y\ub} &> - \MaxInitIPU,\\
    \min_{r} \dotp{-\wb_{-y,r}^{(t)}}{-y\vb} &> - \MaxInitIPV,\\
    0&\le yf(\xb;\Wb^{(t)})\le 3,
\end{align}
and the sign stability is also true throughout $ t\in [T_{(\vb), \delta, L}, T_{(\vb)}^{+,2 }]$.
Therefore, Lemma~\ref{lem: linear increment one data} and \eqref{eq: doesn't decrease 2} implies that for any $ t\in [T_{(\vb), \delta, L}, T_{(\vb)}^{+,2 }]$, it holds that 
\begin{align}
\frac{1}{m}\sum _{r} \sigma\big(\dotp{\wb_{y,v}^{(t)}}{y\vb} \big)\ge&  \frac{1}{m}\sum _{r} \sigma\big(\dotp{\wb_{y,v}^{(T_{(\vb), \delta, L})}}{y\vb}\big)\cdot \exp\left\{- 2\cdot\frac{\norm{\vb}_2^2}{\norm{\ub}_2^2}\cdot \frac{1.5\sqrt{1.05}}{\sqrt{1.05-\delta}}\right\} \ge \frac{1}{2}\delta. 
\end{align}
In conclusion, we have obtained that 
\begin{align}
    \frac{1}{m}\sum_{r\in [m]} \sigma\big(\dotp{\wb_{y,r}}{y\vb}\big)\ge \delta/2,\quad \forall t\in [T_{(\vb), \delta, L}, T_{(\vb)}^{+,2 }].
\end{align}
If $T_{(\vb)}^{+,2 }\neq +\infty$, repeat the above argument and we can finish the proof of Proposition~\ref{prop: doesn't decrease}.
\end{proof}

\subsubsection{Proof of Proposition~\ref{prop: f Tk-1 lower bound}}\label{subsubsec: proof: f Tk-1 lower bound}

\begin{proof}[Proof of Proposition~\ref{prop: f Tk-1 lower bound}]
We continue with the notation $\tilde{\eta} = 2\eta \norm{\ub}_2^2 /m$. Define the function 
\begin{align}
    h_{\tilde{\eta}}(z) \coloneqq  \big(1+ \et (1-z)\big)^2\cdot z.
\end{align}
Note that $\norm{\vb}_2^2\le \norm{\ub}_2^2$ and that 
\begin{align}
    \frac{2\eta \norm{\ub}_2^2}{m} \cdot \big(1-yf(\xb;
\Wb^{(\bar{T}_k-1)})\big)< \frac{2\eta \norm{\ub}_2^2}{m} <1,
\end{align}
from the definition of $\bar{T}_{k}-1$.
Thus we have the following,
\allowdisplaybreaks
\begin{align}
    &yf(\xb,y;\Wb^{(\bar{T}_k )})\\
    &\qquad = \frac{1}{m} \sum_{r\in [m]} \sigma\big(\dotp{\wb_{y,r}^{(\bar{T}_k - 1)}}{y\ub}\big) \cdot \left(1+ \frac{2\eta\norm{\ub}_2^2}{m} \cdot\big( 1- yf( \xb; \Wb^{(\bar{T}_{k}- 1)})\big)\right)^2 \\
    &\qquad\qquad +\frac{1}{m} \sum_{r\in [m]} \sigma\big(\dotp{\wb_{y,r}^{(\bar{T}_k - 1)}}{y\vb}\big) \cdot \left(1+ \frac{2\eta\norm{\vb}_2^2}{m}\cdot \big( 1- yf( \xb; \Wb^{(\bar{T}_{k}- 1)})\big)\right)^2 \\ 
    &\qquad\qquad -\frac{1}{m} \sum_{r\in [m]} \sigma\big(-\dotp{\wb_{-y,r}^{(\bar{T}_k - 1)}}{-y\ub}\big) \cdot \left(1-\frac{2\eta\norm{\ub}_2^2}{m} \cdot\big( 1- yf( \xb; \Wb^{(\bar{T}_{k}- 1)})\big)\right)^2 \\ 
    &\qquad\qquad -\frac{1}{m} \sum_{r\in [m]} \sigma\big(-\dotp{\wb_{-y,r}^{(\bar{T}_k - 1)}}{-y\vb}\big) \cdot \left(1- \frac{2\eta\norm{\vb}_2^2}{m}\cdot \big( 1- yf( \xb; \Wb^{(\bar{T}_{k}- 1)})\big)\right)^2 \\ 
    &\qquad\le  \frac{1}{m} \sum_{r\in [m]} \Big(\sigma\big(\dotp{\wb_{y,r}^{(\bar{T}_k - 1)}}{y\ub}\big)+  \sigma\big(\dotp{\wb_{y,r}^{(\bar{T}_k - 1)}}{y\vb}\big)\Big) \cdot \left(1+ \frac{2\eta\norm{\ub}_2^2}{m} \cdot \big( 1- yf( \xb; \Wb^{(\bar{T}_{k}- 1)})\big)\right)^2 \\
    &\qquad\qquad -\frac{1}{m} \sum_{r\in [m]} \Big(\sigma\big(-\dotp{\wb_{-y,r}^{(\bar{T}_k - 1)}}{-y\ub}\big)   +\sigma\big(-\dotp{\wb_{-y,r}^{(\bar{T}_k - 1)}}{-y\vb}\big)\Big) \\
    &\qquad \qquad \qquad \cdot \left(1-\frac{2\eta\norm{\ub}_2^2}{m} \cdot\big( 1- yf( \xb; \Wb^{(\bar{T}_{k}- 1)})\big)\right)^2 \\
    &\qquad= yf(\xb;\Wb^{(\bar{T}_k - 1)})\cdot \left(1+ \tilde\eta \big( 1- yf( \xb; \Wb^{(\bar{T}_{k}- 1)})\big)\right)^2  \\
    &\qquad\qquad + \frac{1}{m} \sum_{r\in [m]} \Big(\sigma\big(-\dotp{\wb_{-y,r}^{(\bar{T}_k - 1)}}{-y\ub}\big) + \sigma\big(-\dotp{\wb_{-y,r}^{(\bar{T}_k - 1)}}{-y\vb}\big)\Big) \cdot \underbrace{\left(2\et\big( 1-yf( \xb; \Wb^{(\bar{T}_{k}- 1)})\big)\right)^2}_{<2\et <2 } \\
    &\qquad\le  yf(\xb;\Wb^{(\bar{T}_k - 1)})\cdot \left(1+ \tilde\eta \big( 1- yf( \xb; \Wb^{(\bar{T}_{k}- 1)})\big)\right)^2 + 2E_{\bar{T}_k-1}.    \label{eq: f Tk expand -} 
\end{align}
Thus by \eqref{eq: f Tk expand -} we have proved the second conclusion of Proposition~\ref{prop: f Tk-1 lower bound}. 
In the following, we prove the first conclusion of Proposition~\ref{prop: f Tk-1 lower bound}. 

By Assumption~\ref{ass:oscilation_one_data}, we know that $yf(\xb;\Wb^{(\bar{T}_k)})>1+\delta$. 
The definition of $T_{(\vb)}$ along with with Inequality~\eqref{eq: f Tk expand -} imply that
\begin{align}
    h_{\tilde{\eta}}\big(f(\xb;\Wb^{(\bar{T}_k-1)})\big) =yf(\xb;\Wb^{(\bar{T}_k - 1)})\cdot \left(1+ \tilde\eta \big( 1- yf( \xb; \Wb^{(\bar{T}_{k}- 1)})\big)\right)^2\ge 1+\delta-\delta= 1 . \label{eq: h(f Tk-1) lower bound}
\end{align}
Now it suffices to consider the equation $h_{\tilde\eta}-1 =0$, and one can easily verify that it has three roots
\begin{align}
    z_1 &= 1,\\
    z_2 &=\frac{\et+2- \sqrt{\et^2+4\et}}{2\et},\\
    z_3 &=\frac{\et+2 + \sqrt{\et^2+4\et}}{2\et} > \frac{2\et+2}{2\et}>1.
\end{align}
And the second root $z_2<1 $ if and only if $\et >1/2 $. 

Now if $0<\et \le 1/2$, then $z_1 \le z_2<z_3$, and then $h(z)<1$ for $z< z_1 = 1$. Therefore Inequality~\eqref{eq: h(f Tk-1) lower bound} implies that $yf(\xb;\Wb^{(\bar{T}_k - 1)})\ge z_1 = 1$, and recursively implies that $yf(\xb;\Wb^{(0)})\ge 1$, which contradicts with the fact that $yf(\xb;\Wb^{(0)})\le 2\MaxInitIPU$. 
So in order for $\bar{T}_1 <+\infty$, we must have $\et > 1/2$. In conclusion, one necessary condition for the stable oscillation Assumption~\ref{ass:oscilation_one_data} is that $\et>1/2$, and therefore
\begin{align}
    yf(\xb;\Wb^{(\bar{T}-1)}) \ge z_2 =  \frac{\et +2 - \sqrt{\et^2+4\et}}{2\et}.
\end{align}
This finishes the proof of Proposition~\ref{prop: f Tk-1 lower bound}.
\end{proof}

\subsection{Discussion: Necessary Condition for \texorpdfstring{$\delta$}--Oscillation}\label{subsec discuss}

Inequality~\eqref{eq: f Tk upper bound} provides an upper bound involving $\et$, which should be 
compatible with Assumption~\ref{ass:oscilation_one_data}, hence
\begin{align}
    \big(\et/2 + \sqrt{\et^2/4+\et}\big)^2 >1+\delta \quad \Leftrightarrow \quad \et> (1+\delta^{-1}) \big(\sqrt{1+\delta} -1\big). 
\end{align}
One can verify with software that RHS is monotonically increasing in $\delta \in [0,1]$ with minimal value $0.5$, when $\et = 0$, which is in line with the weakest oscillation condition discovered in the last part. And the maximal value taken at $\delta = 1$ is less than $0.83$. 

On the other hand, Inequality~\eqref{eq: f Tk-1 lower bound} should also be compatible with the Assumption~\ref{ass: oscillation}, which indicates 
\begin{align}
 1-\delta > yf(\xb; \Wb^{(\bar{T}_{k-1}})) > \frac{2+\et - \sqrt{\et^2 + 4\et }}{2\et}. 
\end{align}
The readers can see that it is equivalent to $\et > \delta^{-1} ((1-\delta)^{-1/2}-1)$. Furthermore, we have that $\delta^{-1} ((1-\delta)^{-1/2}-1) > (1+\delta^{-1}) (\sqrt{1+\delta} -1)$ thus it is a stronger requirement on $\eta$.

\section{Single Training Data Case: Small Learning Rate Regime}\label{sec: one data small lr}

This section focuses on training our model with single noiseless data point $(\mathbf{x},y)$, where $\mathbf{x}=(y\mathbf{u},y\mathbf{v})$ contains two signal patches with $\mathbf{u}$ much stronger than $\mathbf{v}$. Therefore, the whole objective can be rearranged by
\begin{align}
    L(\mathbf{W}) = \frac{1}{2}\big(f(\mathbf{x};\mathbf{W}) - y\big)^2=\frac{1}{2}\left(\sum_{j\in\{\pm 1\}} \frac{j}{m} \sum_{r\in[m]} \sigma(\langle\mathbf{w}_{j,r},y\mathbf{u}\rangle)+\sigma(\langle\mathbf{w}_{j,r},y\mathbf{v}\rangle)-y\right)^2.
\end{align}
In this simplified setting, each weight vector is updated by
\begin{align}
    \mathbf{w}_{j,r}^{(t+1)} = \mathbf{w}_{j,r}^{(t)} - \frac{jy\eta}{m}\cdot \big(f(\mathbf{x};\mathbf{W}^{(t)}) - y\big)\cdot\left( \sigma^{\prime}(\langle\mathbf{w}_{j,r}^{(t)},y\mathbf{u}\rangle)\mathbf{u}+\sigma^{\prime}(\langle\mathbf{w}_{j,r}^{(t)},y\mathbf{v}\rangle)\mathbf{v}\right).
\end{align}
Then we can directly obtain the following updating rules of the inner products,
\begin{align}
\langle\mathbf{w}_{j,r}^{(t+1)},\mathbf{u}\rangle=\langle\mathbf{w}_{j,r}^{(t)},\mathbf{u}\rangle-\frac{jy\eta}{m}\cdot \big(f(\mathbf{x};\mathbf{W}^{(t)}) - y\big)\cdot\sigma^{\prime}(\langle\mathbf{w}_{j,r}^{(t)},y\mathbf{u}\rangle)\cdot\|\mathbf{u}\|_2^2,\\
    \langle\mathbf{w}_{j,r}^{(t+1)},\mathbf{v}\rangle=\langle\mathbf{w}_{j,r}^{(t)},\mathbf{u}\rangle-\frac{jy\eta}{m}\cdot \big(f(\mathbf{x};\mathbf{W}^{(t)}) - y\big)\cdot\sigma^{\prime}(\langle\mathbf{w}_{j,r}^{(t)},y\mathbf{v}\rangle)\cdot\|\mathbf{v}\|_2^2,
\end{align}
since $\mathbf{u}$ and $\mathbf{v}$ are assumed to be orthogonal in our data generation model.
In this section, we also denote the fitting residual at iteration $t$ as $\ell^{(t)}=f(\mathbf{x};\mathbf{W}^{(t)}) - y$ for convenience.

To better prepare for the analysis of this section, we single out the following concentration results from Appendix~\ref{sec: preliminary lemmas} which provides a high-probability bound on the initialization $\mathbf{W}^{(0)}$.
\begin{lemma}[Initialization]
    Suppose that $d= \Omega(\log(m/p))$ and $m=\Omega(\log(1/p))$. Then with probability at least $1-p$, there holds
    \begin{align}
        \sigma_0\|\mathbf{u}\|/2\le& \max_{r\in[m]}\langle\mathbf{w}^{(t)}_{j,r},j\mathbf{u}\rangle\le \sqrt{2\log(16m/p)} \sigma_0\|\mathbf{u}\|,\\
        \sigma_0\|\mathbf{v}\|/2\le& \max_{r\in[m]}\langle\mathbf{w}^{(t)}_{j,r},j\mathbf{v}\rangle\le \sqrt{2\log(16m/p)} \sigma_0\|\mathbf{v}\|,
    \end{align}
    for all $j\in\{\pm 1\}$.
\end{lemma}

The result of this section relies on the following conditions on the data model and the initialization.

\begin{assumption}[Conditions on hyperparameters]\label{cond:model_params_one_data small-lr}
    Suppose that the following holds:
    \begin{enumerate}[leftmargin = 0.3in]
        \item The weight initialization scale $\sigma_0 = \widetilde{\Theta}(\|\mathbf{u}\|^{-1}_2)$;
        \item The signal strength $\norm{\ub}_2> \widetilde{\Omega}(m^2)\cdot \norm{\vb}_2$;
        \item The dimension $d$ satisfies $d = \Omega(\mathrm{polylog}(m))$.
    \end{enumerate}
\end{assumption}

\begin{theorem}[Restatement of Proposition~\ref{prop:one-data small-lr}]\label{thm:one-data small-lr formal}
    Under Assumption~\ref{cond:model_params_one_data small-lr}, choosing the learning rate $\eta\le m/6\|\ub\|_2^2$ small enough and $\epsilon=0.01$, then with probability at least $1 - p = 1-1/\mathrm{poly}(d)$, there exist
    \begin{equation*}
        T^\dagger=\frac{m}{\eta(1-\tau)\|\mathbf{u}\|^2_2}\log\left(\frac{2\iota}{\sigma_0\|\mathbf{u}\|_2}\right),\quad T=T^\dagger+\left\lfloor\frac{Cm^3}{2\eta\epsilon\norm{\ub}^2}\right\rfloor,
    \end{equation*}
    with $\tau$, $\iota$ defined in \eqref{eq: tau}, \eqref{eq: iota},
    such that: (i) the average loss over iterations $[T^\dagger,T]$ decreased to $2\epsilon$, i.e. 
    \begin{align}
        \frac{1}{T-T^\dagger+1}\sum_{s=T^\dagger}^{T}L(\mathbf{W}^{(s)})\le 2\epsilon,
    \end{align}
    (ii) the model does not learn weak signal $\vb$ well enough, compared to initialization, i.e. 
    \begin{align}
        \max_{j\in\{\pm 1\},r\in[m]} \left|\langle\mathbf{w}_{j,r}^{(t)},\mathbf{v}\rangle\right|\le 2\sqrt{2\log(16m/p)}\cdot \sigma_0\|\mathbf{v}\|_2.
    \end{align}
\end{theorem}

In the small learning rate regime, the dynamics go through two stages, in which the strong signal $\mathbf{u}$ will be firstly learned exponentially fast, and subsequently fully fit the given data point $(\mathbf{x},y)$ therefore stabilizing the training process.
The following lemma plays an important role in the exponentially increasing stage, for which we single it out here.
\begin{lemma}[Derivative lower bound]\label{lemma:one-data, lower bound fitting residual}
    For any $0<\tau<1$ to be tuned later, suppose at some time $t$ there holds
    \begin{align}
        \max_{r\in[m], j\in\{\pm 1\}}\left\{\left|\langle\mathbf{w}_{j,r}^{(t)},\mathbf{u}\rangle\right|, \left|\langle\mathbf{w}_{j,r}^{(t)},\mathbf{v}\rangle\right|\right\}\le\sqrt{\frac{\tau}{2}},
    \end{align}
    then we can lower bound the fitting residual by $-y\ell^{(t)}\ge 1-\tau$.
\end{lemma}
\begin{proof}[Proof of Lemma~\ref{lemma:one-data, lower bound fitting residual}]
    Plug into the CNN model definition \eqref{eq: single data cnn}, we have that
    \begin{align}
    -y\ell^{(t)}=1-F_y(\mathbf{x};\mathbf{W}^{(t)})+F_{-y}(\mathbf{x};\mathbf{W}^{(t)})\ge 1-F_y(\mathbf{x};\mathbf{W}^{(t)}).
    \end{align}
    We can upper bound $F_y(\mathbf{x};\mathbf{W}^{(t)})$ further by
    \begin{align}
    F_y(\mathbf{x};\mathbf{W}^{(t)})&=\frac{1}{m}\sum_{r\in[m]}\sigma(\langle\mathbf{w}_{y,r}^{(t)},y\mathbf{u}\rangle)+\sigma(\langle\mathbf{w}_{y,r}^{(t)},y\mathbf{v}\rangle)\\
    &\le\max_{r\in[m]}\left\{\langle\mathbf{w}_{y,r}^{(t)},y\mathbf{u}\rangle^2+\langle\mathbf{w}_{y,r}^{(t)},y\mathbf{v}\rangle^2\right\}\\
    &\le \tau.
    \end{align}
    Then it follows that $-y\ell^{(t)}\ge 1-\tau$.
\end{proof}

\subsection{Stage 1. Exponential Growth}
We will mainly track the maximal inner product between $\mathbf{w}$ and the signal vectors $\mathbf{v}$ and $\mathbf{u}$, i.e., 
\begin{align}
    \Psi^{(t)}=\max_{j,r} \left|\langle\mathbf{w}_{j,r}^{(t)},\mathbf{v}\rangle\right|,\quad \Phi^{(t)}=\max_{j,r} \left|\langle\mathbf{w}_{j,r}^{(t)},\mathbf{u}\rangle\right|.
\end{align}
In the following, we would take
\begin{align}
    \tau&=\max\left\{2\sigma_0\|\mathbf{u}\|_2\left(2\log(16m/p)\right)^{1/2-\|\mathbf{v}\|_2^2/4\|\mathbf{u}\|_2^2},1-\frac{6\sqrt{2}\|\mathbf{v}\|_2^2}{\|\mathbf{u}\|^2_2\log(2/\sqrt{2\log(16m/p)})}\right\},\label{eq: tau}\\
    \iota&=\sigma_0\|\mathbf{u}\|_2\cdot\exp\left\{\frac{1-\tau}{6}\log\left(\frac{\sqrt{\tau/2}}{\sigma_0\|\mathbf{u}\|_2\sqrt{2\log(16m/p)}}\right)\right\}.\label{eq: iota}
\end{align}
By the conditions in Assumption~\ref{cond:model_params_one_data small-lr} on $\|\mathbf{v}\|_2^2/\|\mathbf{u}\|_2^2$, we find $\tau,\iota$ both constants in $(0,1)$.

\begin{lemma}[First stage: one training data case]\label{lemma:one-data small-lr 1st-stage}
Under the same conditions as Theorem~\ref{thm:one-data small-lr formal}, there exists time
\begin{align}
    T^\dagger=\frac{m}{\eta(1-\tau)\|\mathbf{u}\|^2_2}\log\left(\frac{2\iota}{\sigma_0\|\mathbf{u}\|_2}\right),
\end{align}
such that: 
(i) the model learns the strong signal to a constant level, i.e., 
\begin{align}
    \max_{r\in[m]} \langle\mathbf{w}_{y,r}^{(T^\dagger)},y\mathbf{u}\rangle\ge\iota,
\end{align} 
(ii) compared to the random initialization, the model does not learn weak signal that much, i.e., 
\begin{align}
    \max_{j\in\{\pm 1\},r\in[m]} \left|\langle\mathbf{w}_{j,r}^{(T^\dagger)},\mathbf{v}\rangle\right|\le2\sqrt{2\log(16m/p)}\cdot\sigma_0\|\mathbf{v}\|_2.
\end{align}
\end{lemma}

\begin{proof}[Proof of Lemma~\ref{lemma:one-data small-lr 1st-stage}]
Firstly, we would find $\{\Psi^{(t)},\Phi^{(t)}\}_{t\ge 0}$ having an exponentially growing upper bound. Recursively, we would have that
\begin{align}
    \Psi^{(t+1)}&\le \Psi^{(t)}+\max_{j\in\{\pm 1\},r\in[m]}\left|\frac{jy\eta}{m}\cdot \big(f(\mathbf{x};\mathbf{W}^{(t)}) - y\big)\cdot\sigma^{\prime}(\langle\mathbf{w}_{j,r}^{(t)},y\mathbf{v}\rangle)\cdot\|\mathbf{v}\|_2^2\right|\\
    & = \Psi^{(t)}+\frac{\eta}{m}\cdot \left|\ell^{(t)}\right|\cdot \|\mathbf{v}\|_2^2\cdot\max_{j\in\{\pm 1\},r\in[m]}\sigma^{\prime}(\langle\mathbf{w}_{j,r}^{(t)},y\mathbf{v}\rangle)\\
    & \le \Psi^{(t)}+\frac{2\eta}{m}\cdot\left|\ell^{(t)}\right|\cdot \|\mathbf{v}\|^2_2\cdot \Psi^{(t)}\\
    &\le\exp\left(\frac{6\eta\|\mathbf{v}\|_2^2}{m}\right)\cdot \Psi^{(t)}.
\end{align}
Therefore, we have that 
\begin{align}
    \Psi^{(t)}\le\exp\left(\frac{6\eta\|\mathbf{v}\|_2^2t}{m}\right)\cdot\Psi^{(0)}\le \exp\left(\frac{6\eta\|\mathbf{v}\|_2^2t}{m}\right)\cdot\sqrt{2\log(16m/p)} \cdot \sigma_0\|\mathbf{v}\|_2.\label{eq: psi exp bound}
\end{align} 
It follows by the same argument that that
\begin{align}
    \Phi^{(t)}\le\exp\left(\frac{6\eta\|\mathbf{u}\|^2_2t}{m}\right)\Phi^{(0)}\le \exp\left(\frac{6\eta\|\mathbf{u}\|_2^2t}{m}\right)\cdot\sqrt{2\log(16m/p)}\cdot \sigma_0\|\mathbf{u}\|_2.\label{eq: phi exp bound}
\end{align}
Note that the growing rates of these two bounds differ a lot due to the different magnitudes of $\|\mathbf{u}\|_2$ and $\|\mathbf{v}\|_2$.
Our subsequent analysis illustrates that $\Phi^{(t)}$ can grow into a constant-level magnitude since the strong signal $\mathbf{u}$ is significant enough. 
Now we can track how well our model learns $\mathbf{u}$ by
\begin{align}
    A^{(t)}=\max_{r\in[m]} \langle\mathbf{w}_{y,r}^{(t)},y\mathbf{u}\rangle.
\end{align}
By the definition, $A^{(t)}\le \Phi^{(t)}$ also admits an exponentially growing upper bound. For a certain $\tau\in(0,1)$, due to the previous upper bounds \eqref{eq: psi exp bound} and \eqref{eq: phi exp bound}, $\max\{\Phi^{(t)},\Psi^{(t)}\}\le \sqrt{\tau/2}$ remains true at least until
\begin{align}\label{eq: T1}
    T_1=\frac{m}{6\eta\|\mathbf{u}\|^2}\log\left(\frac{\sqrt{\tau/2}}{\sigma_0\|\mathbf{u}\|_2\sqrt{2\log(16m/p)}}\right).
\end{align}
Consequently, until at least $T_1$, we can use Lemma~\ref{lemma:one-data, lower bound fitting residual} to conclude that $-y\ell^{(t)}\ge 1-\tau$, which enables lower bounding $A^{(t)}$. 
Specifically, start with the updating rule
\begin{align}
   \langle\mathbf{w}_{y,r}^{(t+1)},y\mathbf{u}\rangle&=\langle\mathbf{w}_{y,r}^{(t)},y\mathbf{u}\rangle+\frac{\eta}{m}\cdot \big(-y\ell^{(t)}\big)\cdot\sigma^{\prime}(\langle\mathbf{w}_{y,r}^{(t)},y\mathbf{u}\rangle)\cdot\|\mathbf{u}\|^2_2\\
   &\ge\langle\mathbf{w}_{y,r}^{(t)},y\mathbf{u}\rangle+\frac{2\eta(1-\tau)\|\mathbf{u}\|^2_2}{m}\cdot \max_{r\in[m]}\left\{\langle\mathbf{w}_{y,r}^{(t)},y\mathbf{u}\rangle,0\right\},
\end{align}
and take maximum over $r\in[m]$ to see that
\begin{align}
    A^{(t+1)}\ge A^{(t)}+\frac{2\eta(1-\tau)\|\mathbf{u}\|^2_2}{m}\cdot A^{(t)}\ge\exp\left(\frac{\eta(1-\tau)\|\mathbf{u}\|^2_2}{m}\right)\cdot A^{(t)},
\end{align}
where the last equality is by $1+z\ge\exp(z/2)$ for any $0\le z\le2$. 
Consequently, we would have 
\begin{align}
A^{(t)}\ge \exp\left(\frac{\eta(1-\tau)\|\mathbf{u}\|_2^2t}{m}\right)\cdot A^{(0)}\ge \exp\left(\frac{\eta(1-\tau)\|\mathbf{u}\|_2^2t}{m}\right)\cdot \sigma_0\|\mathbf{u}\|_2/2,\label{eq: exp lower bound}
\end{align}
at least until $t\le T_1$ defined in \eqref{eq: T1}. 
Then we define another time
\begin{align}
    T_2=\frac{m}{\eta(1-\tau)\|\mathbf{u}\|^2_2}\log\left(\frac{2\iota}{\sigma_0\|\mathbf{u}\|_2}\right)\le T_1,
\end{align}
where the inequality is due to the scaling of $\iota$ upon $\tau$. Plugging $T_2$ into the exponential lower bound \eqref{eq: exp lower bound}, we can conclude that 
\begin{align}
    \Phi^{(T_2)}\ge A^{(T_2)}\ge\iota,   
\end{align}
 which already grows up to a constant level magnitude by the time $T_2$. Lastly, we plug the definition of $T_2$ to upper bound $\Psi^{(T_2)}$ as
\begin{equation*}
    \Psi^{(T_2)}\le \exp\left(\frac{6\|\mathbf{v}\|_2^2}{(1-\tau)\|\mathbf{u}\|_2^2}\log\left(\frac{2\iota}{\sigma_0\|\mathbf{u}\|_2}\right)\right)\cdot \sqrt{2\log(16m/p)}\cdot \sigma_0\|\mathbf{v}\|_2 \le 2\sqrt{2\log(16m/p)}\cdot \sigma_0\|\mathbf{v}\|_2.
\end{equation*}
In conclusion, by taking $T^\dagger=T_2$, this lemma is completely proved.
\end{proof}

\subsection{Stage 2. Stabilized Convergence}
In the second stage, our lemmas would suggest that before the model really learns the weak signal $\mathbf{v}$ (i.e. before $\max_{j,r} |\langle\mathbf{w}_{j,r}^{(t)},\mathbf{v}\rangle|$ breaks the $\widetilde{\mathcal{O}}(\sigma_0\|\mathbf{v}\|_2)$ upper bound), the model already fits the given data point by exploiting the strong signal $\mathbf{u}$ and decreasing the loss to $\epsilon$.

\begin{lemma}[Second stage: one training data case]\label{lemma:one-data small-lr 2nd-stage}
    Under the same conditions as Theorem~\ref{thm:one-data small-lr formal}, for any $\epsilon=0.01$, there exists time
    \begin{equation*}
        T=T^\dagger+\left\lfloor\frac{Cm^3}{2\eta\epsilon\norm{\ub}_2^2}\right\rfloor
    \end{equation*}
    such that: (i) the average loss over iterations within this stage has decreased to 
    $2\epsilon$, i.e., 
    \begin{align}
        \frac{1}{T-T^\dagger+1}\sum_{s=T^\dagger}^{T}L(\mathbf{W}^{(s)})\le 2\epsilon,
    \end{align}
    (ii) throughout the training dynamics $0\le t\le T$, there holds 
    \begin{align}
        \max_{j\in\{\pm 1\},r\in[m]} |\langle\mathbf{w}_{j,r}^{(t)},\mathbf{v}\rangle|\le 2\sqrt{2\log(16m/p)}\sigma_0\|\mathbf{v}\|_2.
    \end{align}
\end{lemma}

The proof of Lemma~\ref{lemma:one-data small-lr 2nd-stage} relies on the following three lemmas (Lemmas~\ref{lemma:one-data small-lr 2nd starts}, \ref{lemma:inner-prod grad reference pt}, and \ref{lemma:one-data small-lr descent lemma}).
We present these lemmas with proofs and then combine them to give the proof of Lemma~\ref{lemma:one-data small-lr 2nd-stage}.

Firstly, we identify when the upper bound on $\langle\mathbf{w}_{j,r}^{(t)},\mathbf{v}\rangle$ breaks and find that the conclusions of Lemma~\ref{lemma:one-data small-lr 1st-stage} still holds before that time.

\begin{lemma}\label{lemma:one-data small-lr 2nd starts}
Under the same conditions in Theorem~\ref{thm:one-data small-lr formal}, take $\eta\le m/6\|\ub\|_2^2$. There exists a time
\begin{align}
    T^\ddagger=\frac{m}{6\eta\|\mathbf{v}\|_2^2}\log(2)\ge T^\dagger
\end{align}
such that
\begin{align}
    \max_{r\in[m]} \langle\mathbf{w}_{y,r}^{(t)},y\mathbf{u}\rangle\ge \iota/2,\quad \max_{j\in\{\pm 1\},r\in[m]} \left|\langle\mathbf{w}_{j,r}^{(t)},\mathbf{v}\rangle\right|\le2\sqrt{2\log(16m/p)}\cdot\sigma_0\|\mathbf{v}\|_2.
\end{align}
hold for any $T^\dagger\le t\le T^\ddagger$.
\end{lemma}

\begin{proof}[Proof of Lemma~\ref{lemma:one-data small-lr 2nd starts}]
Firstly, we need to adopt the exponential upper bound derived in proving Lemma~\ref{lemma:one-data small-lr 1st-stage},
\begin{align}
    \Psi^{(t)}\le\exp\left(\frac{6\eta\|\mathbf{v}\|_2^2t}{m}\right)\cdot\Psi^{(0)}\le \exp\left(\frac{6\eta\|\mathbf{v}\|_2^2t}{m}\right)\cdot\sqrt{2\log(16m/p)} \cdot\sigma_0\|\mathbf{v}\|_2.
\end{align}
Then we naturally find that before $T^\ddagger$, it would always hold that 
\begin{align}
    \max_{j\in\{\pm 1\},r\in[m]} \left|\langle\mathbf{w}_{j,r}^{(t)},\mathbf{v}\rangle\right|\le 2\sqrt{2\log(16m/p)}\cdot\sigma_0\|\mathbf{v}\|_2.
\end{align}
Due to the conditions on $\norm{\ub}_2^2/\norm{\vb}_2^2$, $T^\ddagger$ is found to be much larger than $T^\dagger$.
Then we proceed by induction to prove the other assertion. 
At time $t=T^\dagger$, the lower bound $\max_{r} \langle\mathbf{w}_{y,r}^{(t)},y\mathbf{u}\rangle\ge \iota/2$ holds as a consequence of the previous lemma. 
Suppose it holds until time $t$. 
We restate the updating rule by
\begin{align}
\langle\mathbf{w}_{y,r}^{(t+1)},y\mathbf{u}\rangle=\langle\mathbf{w}_{y,r}^{(t)},y\mathbf{u}\rangle+\frac{\eta}{m}\cdot \big(1-yf(\mathbf{x};\mathbf{W}^{(t)})\big)\cdot\sigma^{\prime}(\langle\mathbf{w}_{y,r}^{(t)},y\mathbf{u}\rangle)\cdot \|\mathbf{u}\|_2^2,
\end{align}
from which we find $\max_{r\in[m]} \langle\mathbf{w}_{y,r}^{(t+1)},y\mathbf{u}\rangle\ge \max_{r\in[m]} \langle\mathbf{w}_{y,r}^{(t)},y\mathbf{u}\rangle$ must hold when $yf(\mathbf{x};\mathbf{W}^{(t)})\le 1$. Otherwise, once $yf(\mathbf{x};\mathbf{W}^{(t)})> 1$, it immediately follows that
\begin{align}
    1&<yf(\mathbf{x};\mathbf{W}^{(t)})=F_y(\mathbf{x};\mathbf{W}^{(t)})-F_{-y}(\mathbf{x};\mathbf{W}^{(t)})\\
    &\le F_y(\mathbf{x};\mathbf{W}^{(t)})=\frac{1}{m}\sum_{r\in[m]}\sigma(\langle\mathbf{w}^{(t)}_{y,r},y\mathbf{u}\rangle)+\sigma(\langle\mathbf{w}^{(t)}_{y,r},y\mathbf{v}\rangle)\\
    &\le \max_{r\in[m]} \langle\mathbf{w}_{y,r}^{(t)},y\mathbf{u}\rangle^2 + 2\log(16m/p)\cdot \sigma_0^2\|\mathbf{v}\|_2^2.
\end{align}
Consequently, for the specific neuron $r^\ast=\argmax_{r\in[m]}\langle\mathbf{w}_{y,r}^{(t)},y\mathbf{u}\rangle$, there holds
\begin{align}
    \langle\mathbf{w}_{y,r^\ast}^{(t+1)},y\mathbf{u}\rangle&\ge \langle\mathbf{w}_{y,r^\ast}^{(t)},y\mathbf{u}\rangle-\frac{3\eta}{m}\cdot \langle\mathbf{w}_{y,r^\ast}^{(t)},y\mathbf{u}\rangle\cdot \|\mathbf{u}\|_2^2\\
    &\ge\left(1-2\log(16m/p)\cdot\sigma_0^2\|\mathbf{v}\|^2_2\right)\cdot\left(1-\frac{3\eta}{m}\cdot\|\mathbf{u}\|^2_2\right)\\
    &\ge\frac{\iota}{2},
\end{align}
where the last inequality is enabled by taking $\eta\le m/6\|\ub\|^2_2$ and $\sigma_0\le \sqrt{1-\iota}/(2\log(16m/p)\|\vb\|_2)$.
Thus by induction, we have finished the proof of Lemma~\ref{lemma:one-data small-lr 2nd starts}.
\end{proof}

Our subsequently analysis confirms that even before $T^\ddagger$, the model can already fit the given data point by exploiting $\ub$.
For the given $0<\epsilon<1$, define a reference point $\mathbf{W}^\ast$ as
\begin{align}\label{eq: w star one data}
    \mathbf{w}_{j,r}^\ast = \frac{4m (1+\epsilon)}{\iota}\cdot\frac{j\mathbf{u}}{\|\mathbf{u}\|^2_2},\quad j\in\{\pm 1\},r\in[m].
\end{align}

\begin{lemma}\label{lemma:inner-prod grad reference pt}
    Under the same condition as the previous lemma, for all $T^\dagger\le t\le T^\ddagger$, there holds that 
    \begin{align}
        y\langle\nabla f(\mathbf{x};\mathbf{W}^{(t)}),\mathbf{W}^\ast\rangle\ge 2\cdot (1+\epsilon).
    \end{align}
\end{lemma}

\begin{proof}[Proof of Lemma~\ref{lemma:inner-prod grad reference pt}]
    Recall the definition of the CNN in \eqref{eq: single data cnn} and that $\mathbf{u}\perp\mathbf{v}$, so we have
    \begin{align*}
        y\langle\nabla f(\mathbf{x};\mathbf{W}^{(t)}),\mathbf{W}^\ast\rangle&=\frac{1}{m}\sum_{j\in\{\pm 1\},r\in[m]}\sigma^\prime(\langle \mathbf{w}_{j}^{(t)},y\mathbf{u}\rangle)\cdot\langle \mathbf{w}_{j,r}^{\ast},y\mathbf{u}\rangle\\
        &=\sum_{j\in\{\pm 1\},r\in[m]}\sigma^\prime(\langle \mathbf{w}_{j,r}^{(t)},y\mathbf{u}\rangle)\cdot\frac{4(1+\epsilon)}{\iota}\\
        &\ge \max_{r\in[m]} \langle\mathbf{w}_{y,r}^{(t)},y\mathbf{u}\rangle\cdot \frac{4(1+\epsilon)}{\iota}\\
        &\ge2\cdot(1+\epsilon),
    \end{align*}
    where the last inequality is by $\max_{r\in[m]} \langle\mathbf{w}_{y,r}^{(t)},y\mathbf{u}\rangle\ge \iota/2$ as shown by the previous lemma.
\end{proof}

\begin{lemma}\label{lemma:one-data small-lr descent lemma}
    Continued from the previous setting, we know that for $T^\dagger\le t\le T^\ddagger$, it holds that
    \begin{equation*}
        \|\mathbf{W}^{(t)}-\mathbf{W}^\ast\|_F^2-\|\mathbf{W}^{(t+1)}-\mathbf{W}^\ast\|_F^2\ge 2\eta L(\mathbf{W}^{(t)})-2\eta\epsilon^2.
    \end{equation*}
\end{lemma}
\begin{proof}[Proof of Lemma~\ref{lemma:one-data small-lr descent lemma}]
    Firstly we expand the difference by
    \begin{align}
        \|\mathbf{W}^{(t)}-\mathbf{W}^\ast\|_F^2-\|\mathbf{W}^{(t+1)}-\mathbf{W}^\ast\|_F^2=2\eta\cdot \langle\nabla L(\mathbf{W}^{(t}),\mathbf{W}^{(t)}-\mathbf{W}^\ast\rangle-\eta^2\cdot\|\nabla L(\mathbf{W}^{(t})\|_F^2.\label{eq:para-L2-distance difference expansion}
    \end{align}
    With one data point, $\nabla L(\mathbf{W}^{(t)})=\ell^{(t)}\nabla f(\mathbf{x};\mathbf{W}^{(t)})$ admits a simplified expression, where $\ell^{(t)}=f(\mathbf{W}^{(t)},\mathbf{x})-y$ denotes the fitting residual.
    Since the neural network $f(\mathbf{W},\mathbf{x})$ is $2$-homogeneous in $\mathbf{W}$ due to  the activation function $\sigma(z)=\max\{z,0\}^2$, we can have
    \begin{align}
        \langle\nabla f(\mathbf{x};\mathbf{W}^{(t)}),\mathbf{W}^{(t)}\rangle=2f(\mathbf{x};\mathbf{W}^{(t)}).
    \end{align}
    Stack these observations into the first term of previous difference expansion to obtain that
    \begin{align*}
        \langle\nabla L(\mathbf{W}^{(t}),\mathbf{W}^{(t)}-\mathbf{W}^\ast\rangle&=\ell^{(t)}\cdot\langle\nabla f(\mathbf{x};\mathbf{W}^{(t)}),\mathbf{W}^{(t)}-\mathbf{W}^\ast\rangle\\
        &=\ell^{(t)}\cdot\big(2f(\mathbf{x};\mathbf{W}^{(t)})-\langle\nabla f(\mathbf{x};\mathbf{W}^{(t)}),\mathbf{W}^\ast\rangle\big)\\
        &=2\ell^{(t)}\cdot\big(f(\mathbf{x};\mathbf{W}^{(t)})-y\big)+\ell^{(t)}\cdot y\cdot\big(2-y\langle\nabla f(\mathbf{x};\mathbf{W}^{(t)}),\mathbf{W}^\ast\rangle\big).
    \end{align*}
    Note that the first term is exactly $4L(\mathbf{W}^{(t)})$. As for the second term, we need to plug in Lemma~\ref{lemma:inner-prod grad reference pt} to see $2-y\langle\nabla f(\mathbf{x};\mathbf{W}^{(t)}),\mathbf{W}^\ast\rangle\leq-2\epsilon<0$, so that
    \begin{align*}
        \left|\ell^{(t)}\cdot y\cdot\big(2-y\langle\nabla f(\mathbf{x};\mathbf{W}^{(t)}),\mathbf{W}^\ast\rangle\big)\right|\leq\frac{1}{2}\big(\ell^{(t)}\big)^2+2\epsilon^2=L(\mathbf{W}^{(t)})+2\epsilon^2.
    \end{align*}
    As a result, we would know $\langle\nabla L(\mathbf{W}^{(t}),\mathbf{W}^{(t)}-\mathbf{W}^\ast\rangle\ge 3L(\mathbf{W}^{(t)})-2\epsilon^2$. Next, an upper bound on the second order term $\eta^2\cdot\|\nabla L(\mathbf{W}^{(t})\|_F^2$ is given by
    \begin{align*}
        \eta^2\cdot\|\nabla L(\mathbf{W}^{(t)})\|_F^2&=\eta^2\cdot\big(\ell^{(t)}\big)^2\cdot\left(\|\mathbf{u}\|^2_2\sum_{j\in\{\pm 1\},r\in[m]}\sigma^\prime(\langle\mathbf{w}^{(t)}_{j,r},y\mathbf{u}\rangle)+\|\mathbf{v}\|_2^2\sum_{j\in\{\pm 1\},r\in[m]}\sigma^\prime(\langle\mathbf{w}^{(t)}_{j,r},y\mathbf{v}\rangle)\right)\\
        &\le \mathcal{O}\big(\max\{\|\mathbf{u}\|_2^2,\|\mathbf{v}\|_2^2\}\big)\cdot\eta^2\cdot L(\mathbf{W}^{(t}),
    \end{align*}
    since the dynamics of the inner products $\langle\mathbf{w}^{(t)}_{j,r},y\mathbf{u}\rangle,\langle\mathbf{w}^{(t)}_{j,r},y\mathbf{v}\rangle$ are well bounded by $\mathcal{O}(1)$ throughout the time we are considering.
    By scaling $\eta \mathcal{O}(\max\{\|\mathbf{u}\|^2,\|\mathbf{v}\|^2\})\le 1$, we would know that $\eta^2\|\nabla L(\mathbf{W}^{(t})\|_F^2\le \eta L(\mathbf{W}^{(t})$. Eventually, continued from \eqref{eq:para-L2-distance difference expansion}, we can completely prove this lemma.
\end{proof}

Equipped with Lemmas~\ref{lemma:one-data small-lr 2nd starts}, \ref{lemma:inner-prod grad reference pt}, and \ref{lemma:one-data small-lr descent lemma}, we are ready to prove the main lemma for the second stage.

\begin{proof}[Proof of Lemma~\ref{lemma:one-data small-lr 2nd-stage}]
Continued from Lemma~\ref{lemma:one-data small-lr descent lemma}, for any $t\ge T^\dagger$, it holds that
\begin{equation*}
    \frac{1}{t-T^\dagger+1}\sum_{s=T^\dagger}^{t}L(\mathbf{W}^{(s)})\le \frac{\|\mathbf{W}^{(T^\dagger)}-\mathbf{W}^\ast\|_F^2}{2\eta(t-T^\dagger+1)}+\epsilon^2.
\end{equation*}
Before proceeding to scale time $t$, it would be helpful to decompose $\|\mathbf{W}^{(T^\dagger)}-\mathbf{W}^\ast\|_F^2$ and to have an upper bound on this term,
\begin{align}
    &\|\mathbf{W}^{(T^\dagger)}-\mathbf{W}^\ast\|_F^2\\
    &\qquad  = \sum_{j\in\{\pm 1\},r\in[m]}\frac{\langle \wb_{j,r}^{(T^\dagger)}-\wb_{j,r}^\ast,\ub\rangle^2}{\norm{\ub}_2^2}+\frac{\langle \wb_{j,r}^{(T^\dagger)}-\wb_{j,r}^\ast,\vb\rangle^2}{\norm{\vb}_2^2}+\left\|\left(\Ib_d-\frac{\vb\vb^\top}{\|\vb\|_2^2} - \frac{\ub\ub^\top}{\|\ub\|_2^2}\right)(\wb_{j,r}^{(T^\dagger)}-\wb_{j,r}^\ast)\right\|_2^2\\
    &\qquad \le\sum_{j\in\{\pm 1\},r\in[m]}\frac{2\langle \wb_{j,r}^{(T^\dagger)},\ub\rangle^2+2\langle \wb_{j,r}^\ast,\ub\rangle^2}{\norm{\ub}_2^2}+\frac{\langle \wb_{j,r}^{(T^\dagger)},\vb\rangle^2}{\norm{\vb}_2^2}+\left\|\left(\Ib_d-\frac{\vb\vb^\top}{\|\vb\|_2^2} - \frac{\ub\ub^\top}{\|\ub\|_2^2}\right)\wb_{j,r}^{(0)}\right\|_2^2,\label{eq: proof lemma stage 2}
\end{align}
where we exploit the fact that $\wb^\ast$ is parallel to $\ub$ by Lemma \eqref{eq: w star one data}, and the gradient steps only updates $\wb$ along the directions of $\ub,\vb$. 
Recall that by Lemma~\ref{lemma:one-data small-lr 2nd starts},
\begin{align}
    \max_{j\in \pm 1,r\in[m]}\langle\wb^{(T^\dagger)}_{j,r},j\ub\rangle=\Omega(1),\quad  \max_{j\in \pm 1,r\in[m]}\left|\langle\wb^{(T^\dagger)}_{j,r},\vb\rangle\right| = \widetilde{\mathcal{O}}(\sigma_0\|\mathbf{v}\|_2),
\end{align}
and also that $\norm{\wb^{(0)}_{j,r}}=\widetilde{\mathcal{O}}(\sigma_0\sqrt{d})$, the leading term in \eqref{eq: proof lemma stage 2} would be $\sum_{j\in\{\pm 1\},r\in[m]}\langle \wb_{j,r}^\ast,\ub\rangle^2/\norm{\ub}^2_2$.
Therefore, we conclude that $\|\mathbf{W}^{(T^\dagger)}-\mathbf{W}^\ast\|_F^2\le Cm^3/\norm{\ub}_2^2$ for some constant $C>0$. 
As a result, the average loss after iterations $T^\dagger$ can be bounded by
\begin{equation*}
    \frac{1}{t-T^\dagger+1}\sum_{s=T^\dagger}^{t}L(\mathbf{W}^{(s)})\le \frac{Cm^3}{2\eta\norm{\ub}_2^2(t-T^\dagger+1)}+\epsilon^2,\quad \forall t^{\dagger}\leq t\leq T^{\ddagger}.
\end{equation*}
Then we choose time $T=T^\dagger+\lfloor Cm^3/(2\eta\epsilon\norm{\ub}_2^2)\rfloor$ as stated in Lemma~\ref{lemma:one-data small-lr 2nd-stage}.
Since $\norm{\ub}_2^2/\norm{\vb}_2^2\ge\widetilde{\Omega}(m^2)$ by Assumption~\ref{cond:model_params_one_data small-lr}, we can verify that $T\le T^\ddagger$.
In conclusion, the final output would be 
\begin{align}
    \frac{1}{T-T^\dagger+1}\sum_{s=T^\dagger}^{T}L(\mathbf{W}^{(s)})\le \epsilon+\epsilon^2\le 2\epsilon.
\end{align}    
This finishes the proof of Lemma~\ref{lemma:one-data small-lr 2nd-stage}.
\end{proof}

Combine Lemmas~\ref{lemma:one-data small-lr 1st-stage} and \ref{lemma:one-data small-lr 2nd-stage} to obtain the full version of Theorem~\ref{thm:one-data small-lr formal}.

\section{Proofs for Main Theoretical Results (Section~\ref{sec: main theory})}\label{sec: multiple data case proof}

In this section, we give a detailed proof for our main theoretical results for the multiple training data case, i.e., Theorem~\ref{thm: main}.
The proofs follow the similar idea as the proofs for single training data case (Appendix~\ref{sec: one data case proof}).
The readers interested in the proofs are encouraged to first go through Appendix~\ref{sec: one data case proof} to get the idea of the core steps.
In Appendix~\ref{subsec: preliminary analysis}, we give a preliminary analysis of the SGD training dynamics. 
In Appendix~\ref{subsec: overview of analysis}, we give an overview of the proofs with our fundamental reasoning towards weak signal learning.
In Appendix~\ref{subsec: proof main}, we give the proof of Theorem~\ref{thm: main}. 
We prove other lemmas in subsequent sections.

\subsection{Preliminary Analysis}\label{subsec: preliminary analysis}

Recall that $\mathcal{W}$ is the index set of training data points which lack the strong feature patch.
By Equation~\eqref{eq: two way multiple data sgd}, the CNN weights are updated according to 
\begin{align}
    \mathbf{w}_{j,r}^{(t+1)} &= \mathbf{w}_{j,r}^{(t)} - \frac{j\eta}{m}\cdot\big(f(\mathbf{x}_{i_t};\mathbf{W}^{(t)}) - y_{i_t}\big)\cdot\left(\sigma^\prime(\langle\mathbf{w}_{j,r}^{(t)},y_{i_t}\mathbf{u}\rangle)\cdot y_{i_t}\mathbf{u}\cdot\mathbf{1}\{i_t\notin\mathcal{W}\}\right. \\
    &\qquad + \sigma^\prime(\langle\mathbf{w}_{j,r}^{(t)},y_{i_t}\mathbf{v}\rangle)\cdot y_{i_t}\mathbf{v} + \sigma^{\prime}(\langle\mathbf{w}_{j,r}^{(t)}\cdot \boldsymbol{\xi}_{i_t}\rangle)\cdot\boldsymbol{\xi}_{i_t} + \left.\sigma^{\prime}(\langle\mathbf{w}_{j,r}^{(t)}\cdot \widetilde{\boldsymbol{\xi}}_{i_t}\rangle)\cdot\widetilde{\boldsymbol{\xi}}_{i_t}\cdot\mathbf{1}\{i_t\in\mathcal{W}\} \right).\label{eq: weight update}
\end{align}
Also, recall that the correct index sets for the strong and weak signal patches are defined as
\begin{align}\label{eq: index set signal appendix}
    \cU_{j,+}^{(t)}:= \left\{r\in[m]: \langle \mathbf{w}_{j,r}^{(t)}, j\mathbf{u}\rangle \geq 0\right\},\quad \cV_{j,+}^{(t)}:= \left\{r\in[m]: \langle \mathbf{w}_{j,r}^{(t)}, j\mathbf{v}\rangle \geq 0\right\},
\end{align} 
By the CNN expression \eqref{eq: cnn}, for each $j\in\{\pm 1\}$, the inner products that matter are 
\begin{enumerate}
    \item \emph{positive neurons:} $\langle \mathbf{w}_{j,r}^{(t)},j\mathbf{u}\rangle$ for $r\in\cU_{j,+}^{(t)}$, $\langle \mathbf{w}_{j,r}^{(t)},j\mathbf{v}\rangle$ for $r\in\cV_{j,+}^{(t)}$;
    \item \emph{negative neurons:} $\langle \mathbf{w}_{-j,r}^{(t)},j\mathbf{u}\rangle$ for $r\notin\cU_{-j,+}^{(t)}$, $\langle \mathbf{w}_{-j,r}^{(t)},j\mathbf{v}\rangle$ for $r\notin\cV_{-j,+}^{(t)}$.
\end{enumerate}
By \eqref{eq: weight update}, the update formula of these inner products of interests are given by
\begin{align}
    \langle \mathbf{w}_{j,r}^{(t+1)},j\mathbf{u}\rangle &= \langle \mathbf{w}_{j,r}^{(t)},j\mathbf{u}\rangle + \frac{\eta\|\mathbf{u}\|_2^2}{m}\cdot \big(1 - y_{i_t} f(\mathbf{x}_{i_t};\mathbf{W}^{(t)})\big)\cdot\sigma^{\prime}(\langle \mathbf{w}_{j,r}^{(t)},j\mathbf{u}\rangle jy_{i_t})\cdot\mathbf{1}\{i_t\notin\mathcal{W}\}, \label{eq: strong signal update}\\ 
    \langle \mathbf{w}_{j,r}^{(t+1)},j\mathbf{v}\rangle &= \langle \mathbf{w}_{j,r}^{(t)},j\mathbf{v}\rangle + \frac{\eta\|\mathbf{v}\|_2^2}{m}\cdot \big(1 - y_{i_t} f(\mathbf{x}_{i_t};\mathbf{W}^{(t)})\big)\cdot\sigma^{\prime}(\langle\mathbf{w}_{j,r}^{(t)}, j\mathbf{v}\rangle jy_{i_t}),\label{eq: weak signal update}
\end{align}
Also by \eqref{eq: weight update}, the update formula of the inner products with noise vectors are given by
\begin{align}
    \langle \mathbf{w}_{j,r}^{(t+1)},\boldsymbol{\xi}_{i}\rangle &= \langle \mathbf{w}_{j,r}^{(t)},\boldsymbol{\xi}_i\rangle + \frac{\eta\cdot jy_{i_t}}{m}\cdot \big(1 - y_{i_t} f(\mathbf{x}_{i_t};\mathbf{W}^{(t)})\big)\cdot\left(\sigma^{\prime}(\langle\mathbf{w}_{j,r}^{(t)}, \boldsymbol{\xi}_{i_t}\rangle)\cdot\langle\boldsymbol{\xi}_{i_t},\boldsymbol{\xi}_i\rangle\right.\\
    & \qquad +\left.\sigma^{\prime}(\langle\mathbf{w}_{j,r}^{(t)},\widetilde{\boldsymbol{\xi}}_{i_t}\rangle)\cdot\langle\widetilde{\boldsymbol{\xi}}_{i_t},\boldsymbol{\xi}_i\rangle\cdot\mathbf{1}\{i_t\in\mathcal{W}\}\right),\quad i\in[n], \label{eq: noise update} \\
    \langle \mathbf{w}_{j,r}^{(t+1)},\widetilde{\boldsymbol{\xi}}_{i}\rangle &= \langle \mathbf{w}_{j,r}^{(t)},\widetilde{\boldsymbol{\xi}}_i\rangle + \frac{\eta\cdot jy_{i_t}}{m}\cdot \big(1 - y_{i_t} f(\mathbf{x}_{i_t};\mathbf{W}^{(t)})\big)\cdot\left(\sigma^{\prime}(\langle\mathbf{w}_{j,r}^{(t)}, \boldsymbol{\xi}_{i_t}\rangle)\cdot\langle\boldsymbol{\xi}_{i_t},\widetilde{\boldsymbol{\xi}}_i\rangle\right.\\
    & \qquad +\left.\sigma^{\prime}(\langle\mathbf{w}_{j,r}^{(t)}, \widetilde{\boldsymbol{\xi}}_{i_t}\rangle)\cdot\langle\widetilde{\boldsymbol{\xi}}_{i_t},\widetilde{\boldsymbol{\xi}}_i\rangle\cdot\mathbf{1}\{i_t\in\mathcal{W}\}\right),\quad i\in\mathcal{W}.\label{eq: noise update weak}
\end{align}
At first glance, the update formulas given above seem intangible. The following proposition indicates that we can separate the neurons into two parts, with each individual part learning one kind of sample independently.
For simplicity, we let $\et\coloneqq 2\eta \norm{\ub}^2_2/m$ and $\alpha = \norm{\vb}_2^2 /\norm{\ub}_2^2$ and $\ft_t = y_{{i_t}} f(\xb_{i_t}; \Wb^{(t)})$. 
For any $j\in \{\pm 1\}$ and $s\ge 0$, we define effective running time for learning label-$j$-samples as 
\begin{align}\label{eq: tjs}
t_j(s)=\min\big\{t\in \NN :t>t_j(s-1),\; y_{i_t} = j\big\},
\end{align}
with $t_j(0)=\min\{t\in \NN :y_{i_t} = j\}$. 

\begin{proposition}\label{prop: independent iteration}
Suppose that the sign stability condition holds before some $T_{\mathrm{sign}}$, i.e., $\cU^{(t)}_{\pm j,\pm} =\cU^{(0)}_{\pm j,\pm}:=\cU_{\pm j,\pm}$ and $\cV^{(t)}_{\pm j,\pm} =\cV^{(0)}_{\pm j,\pm}:=\cV_{\pm j,\pm}$ for $t\le T_{\mathrm{sign}}$. Then for $s$ such that $t_j(s)\le T_{\mathrm{sign}}$, it holds that 
    \begin{align}
        \dotp{\wb_{j,r}^{(t_j(s+1))}}{j\ub}&=\dotp{\wb_{j,r}^{(t_j(s))}}{j\ub} \cdot \big(1+\et (1-\ft_{t_j(s)})\cdot \mathbf{1}\{i_{t_j(s)} \in \cW\}\big), \quad \forall r\in \cU_{j,+};\label{eq: wu positive s ip iteration}\\
        \dotp{\wb_{-j,r}^{(t_j(s+1))}}{-j\ub}&=\dotp{\wb_{-j,r}^{(t_j(s))}}{-j\ub}\cdot \big(1-\et (1-\ft_{t_j(s)})\cdot \mathbf{1}\{i_{t_j(s)} \in \cW\}\big), \quad \forall r\in \cU_{-j,-};\label{eq: wu negative s ip iteration} \\ 
                \dotp{\wb_{j,r}^{(t_j(s+1))}}{j\vb}&=\dotp{\wb_{j,r}^{(t_j(s))}}{j\vb} \cdot \big(1+\alpha \et (1-\ft_{t_j(s)})\big), \quad \forall r\in \cV_{j,+};\label{eq: wv positive s ip iteration}\\
        \dotp{\wb_{-j,r}^{(t_j(s+1))}}{-j\vb}&=\dotp{\wb_{-j,r}^{(t_j(s))}}{-j\vb}\cdot \big(1-\alpha\et (1-\ft_{t_j(s)})\big), \quad \forall r\in \cV_{-j,-}.\label{eq: wv negative s ip iteration}
    \end{align}
Moreover, for every $t\in (t_j(s),t_j(s+1)]$, it holds that 
\begin{align}
    \dotp{\wb_{j,r}^{(t)}}{j\ub} &= \dotp{\wb_{j,r}^{(t_j(s)+1)}}{j\ub}, \quad \forall r\in \cU_{j,+};\\
    \dotp{\wb_{-j,r}^{(t)}}{-j\ub} &= \dotp{\wb_{-j,r}^{(t_j(s)+1)}}{-j\ub}, \quad \forall r\in \cU_{-j,-};\\
    \dotp{\wb_{j,r}^{(t)}}{j\vb} &= \dotp{\wb_{j,r}^{(t_j(s)+1)}}{j\vb}, \quad \forall r\in \cV_{j,+};\\
    \dotp{\wb_{-j,r}^{(t)}}{-j\vb} &= \dotp{\wb_{-j,r}^{(t_j(s)+1)}}{-j\vb}, \quad \forall r\in \cV_{-j,-}.
\end{align}
\end{proposition}

\begin{proof}[Proof of Proposition~\ref{prop: independent iteration}]
See Appendix~\ref{proof: independent iter} for a detailed proof.
\end{proof}

This proposition helps to break down the dynamics of multiple data training into two folds. To see this, it suffices to note that the following components
\begin{align}
    \ft_{t_j(s)}, \quad \left\{\dotp{\wb_{j,r}^{(t)}}{j\ub}\right\}_{r\in \cU_{j,+}}, \quad \left\{\dotp{\wb_{-j,r}^{(t)}}{-j\ub}\right\}_{r\in \cU_{-j,-}}, \quad \left\{\dotp{\wb_{j,r}^{(t)}}{j\vb}\right\}_{r\in \cV_{j,+}}, \quad \left\{\dotp{\wb_{-j,r}^{(t)}}{-j\vb}\right\}_{r\in \cV_{-j,+}}
\end{align}
are independent with $\{\ft_{t}\}_{i_t = -j}$ and the rest inner products associated with $-j$.

\subsection{Overview of Analysis}\label{subsec: overview of analysis}

The roadmap towards proving our main theorem shares nearly the same logic with the proof of single data setup (Appendix~\ref{sec: one data case proof}).
Basically, as long as the weak signal component and the noise component are not learned in the sense that these inner products remain negligible, we can prove that the inner products associated with the strong signal would dominate and the oscillation would accumulate at a linear rate. 
This further gives Lemma~\ref{lem: linear accumulation} which shows the CNN would learn effectively learn the weak signal $\langle \mathbf{w}_{j,r}^{(t)},j\mathbf{v}\rangle$. 
Meanwhile, we can prove that the influences from the negative part of weak signal learning $\langle \mathbf{w}_{-j,r}^{(t)},j\mathbf{v}\rangle$ and the noise memorization $\langle \mathbf{w}_{\pm 1,r}^{(t)},\boldsymbol{\xi}_i\rangle$,  $\langle \mathbf{w}_{\pm 1,r}^{(t)},\widetilde{\boldsymbol{\xi}}_i\rangle$ can be well controlled (Propositions~\ref{prop: noise memorization}).
Putting all together, we can prove the main result Theorem~\ref{thm: main}.

To be formal, we define two important stopping times as follows:
\begin{align}
    T^{j}_{(\mathbf{v})} &= \min_{t\geq 0}\left\{t: \frac{1}{m}\sum_{r\in[m]}\sigma(\langle\mathbf{w}_{j,r}^{(t)},j\mathbf{v}\rangle)\ge  \delta/2\right\}, \label{eq: Tv j}\\
     T_{(\boldsymbol{\xi})} &= \min_{t\geq 0}\left\{t: \max_{r\in [m], j\in\{\pm 1\}}\left\{\max_{i\in[n]}\left|\langle \mathbf{w}_{j,r}^{(t)},\boldsymbol{\xi}_i\rangle\right|,\max_{i\in\mathcal{W}}\left|\langle \mathbf{w}_{j,r}^{(t)},\widetilde{\boldsymbol{\xi}}_i\rangle\right|\right\}\ge \delta/4\right\}. \label{eq: T xi}
\end{align}
We recap that $\mathcal{W}$ denotes the index set of weak data, and $\widetilde{\boldsymbol{\xi}}_i$ denotes the Gaussian noise appearing on the lacking strong signal patch for those weak data.
Also we note that $T^{j}_{(\mathbf{v})}$, $T_{(\boldsymbol{\xi})}\le +\infty$, where the equal sign is attainable.
We then define 
\begin{align}
    T_{\max}^j = \min\left\{T^{j}_{(\mathbf{v})}, T_{(\boldsymbol{\xi})}\right\}.    
\end{align}
  

In the first place, following the same arguments as in the single data setup (Appendix~\ref{sec: one data case proof}), we can derive the boundedness and sign stability results before time $T_{\max}^j$. 

\begin{lemma}[Boundedness and sign stability]\label{lem: boundedness}
Under Assumptions~\ref{ass: conditions} and~\ref{ass: oscillation}, for fixed $j\in \{\pm 1\}$, the followings hold with probability at least $1 - p = 1-1/\mathrm{poly}(d)$:
    \begin{enumerate}[leftmargin=0.6cm]
        \item it holds that $\cU_{j,+}^{(t)} = \cU_{j,+}^{(0)}\neq \emptyset$ and $\cV_{j,+}^{(t)} = \cV_{j,+}^{(0)}\neq \emptyset$ for any $t\in[0,T_{\max}^j]$. Hence the superscript $(t)$ can be dropped;
        \item for any $t\in[0,T_{\max}^j]$, we have that
        \begin{align}
            \max_{r\in [m]} \dotp{\wb_{j,r}^{(t)}}{j\ub} &\le 1.5\cdot(1.05 \beta_{\ub,j}^* m)^{1/2},\label{eq wu upper bound general}
        \end{align}
        where $\beta_{\ub,j}^{*}$ is defined in Lemma~\ref{lem: similar behavior};
        \item for any $t\in[0,T_{\max}^j]$ it holds that 
        \begin{align}
     \min_{r\in [m]} \dotp{\wb_{-j,r}^{(t)}}{-j\ub} &\ge  -\MaxInitIPU, \label{eq min wu negative lower bound}\\
     \min_{r\in [m]} \dotp{\wb_{-j,r}^{(t)}}{-j\vb} &\ge  -\MaxInitIPV; \label{eq min wv negative lower bound}
        \end{align}
        \item for any $t\in[0,T_{\max}^j]$ such that $y_{i_t} = j$, it holds that 
        \begin{align}
            \big|1 - y_{i_t}f(\mathbf{x}_{i_t};\mathbf{W}^{(t)})\big|\leq 2.
        \end{align}
    \end{enumerate}
\end{lemma}

\begin{proof}[Proof of Lemma~\ref{lem: boundedness}]
See Appendix~\ref{subsec: proof boundedness} for a detailed proof.
\end{proof}

This boundedness and sign stability result further implies weak signal learning with an exponential rate, which is formally presented as follows.
 
\begin{lemma}[Weak signal learning]\label{lem: linear accumulation}
    Under Assumptions \ref{ass: conditions} and \ref{ass: oscillation}, with probability at least $1-1/\mathrm{poly}(d)$, it holds that for any $j\in\{\pm 1\}$ and $0\leq t_1\leq t_2\leq T_{\max}^j$ that 
    \begin{align}
        \sum_{\substack{s=t_1 \\ y_{i_s} = j}}^{t_2} \big(1 - y_{i_s}f(\mathbf{x}_{i_s};\mathbf{W}^{(s)}) \big)\geq \frac{\delta}{16}\cdot\big(1 - (1.05 - \delta/4 )^{\frac{1}{2}}\big) \cdot(t_2 - t_1 + 1) - \frac{m {(1.05)^{\frac{1}{2}}}}{2\eta\|\mathbf{u}\|_2^2(1.05 - \delta/4)^{\frac{1}{2}}},
    \end{align}
    where $\delta$ is specified in Assumption \ref{ass: oscillation}. Moreover, for $r\in \cV_{j,+}^{(0)}$, we have that
        \begin{align}
        \!\!\!\!\!\!\!\!\!\!\dotp{\wb_{j,r}^{(t_2+1)}}{j\vb} &\geq \dotp{\wb_{j,r}^{(t_1)}}{j\vb} \!\cdot\!\exp\left\{\frac{\eta\|\mathbf{v}\|_2^2}{32m}\!\cdot\!\big(\delta - \delta(1.05 - \delta/4 )^{\frac{1}{2}}\big) \!\cdot\!(t_2 - t_1 + 1) - \frac{\|\mathbf{v}\|_2^2 {(1.05 )^{\frac{1}{2}}}}{\|\mathbf{u}\|_2^2(1.05 - \delta/4)^{\frac{1}{2}}}\right\}. \label{eq:exp linear acc}
    \end{align}
\end{lemma}
\begin{proof}[Proof of Lemma~\ref{lem: linear accumulation}]
    See Appendix~\ref{subsec: proof linear accumulation} for a detailed proof.
\end{proof}

The last component to complete our proof of main theorem is controlling noise memorization during the training process. 
For simplicity, we define the maximum absolute value of the noise inner products over data as 
\begin{align}
    \Upsilon^{(t)} = \max_{r\in [m], j\in\{\pm 1\}}\left\{\max_{i\in[n]}\left|\langle \mathbf{w}_{j,r}^{(t)},\boldsymbol{\xi}_i\rangle\right|,\max_{i\in\mathcal{W}}\left|\langle \mathbf{w}_{j,r}^{(t)},\widetilde{\boldsymbol{\xi}}_i\rangle\right|\right\}.
\end{align}

We have the following proposition to control the growth of $\Upsilon^{(t)}$.

\begin{proposition}[Noise memorization]\label{prop: noise memorization}
    Under Assumptions \ref{ass: conditions} and \ref{ass: oscillation}, then with probability at least $1-1/\mathrm{poly}(d)$, it holds for any $0\leq t_1  \leq \min_{j\in\{\pm 1\}}\{T_{\max}^j\} - \widetilde{T}_{(\boldsymbol{\xi})}$ and $j\in\{\pm 1\}$ that 
    \begin{align}
        \Upsilon^{(t)}\leq \Upsilon^{(t_1)}\cdot(1+\epsilon),\quad \forall r\in[m],\quad \forall t_1\leq t\leq t_1+\widetilde{T}_{(\boldsymbol{\xi})}
    \end{align}
    where $\widetilde{T}_{(\boldsymbol{\xi})} = \widetilde{\Theta}(mn\cdot\eta^{-1}\cdot\epsilon\cdot(1+\epsilon)^{-1}\cdot(\sigma_p^2d)^{-1})$ for any $\epsilon>0$.
\end{proposition}

\begin{proof}[Proof of Proposition \ref{prop: noise memorization}]
    See Appendix~\ref{subsec: proof noise memorization} for a detailed proof.
\end{proof}

\subsection{Proof of Theorem~\ref{thm: main}}\label{subsec: proof main}

With Lemma~\ref{lem: boundedness} and~\ref{lem: linear accumulation} and Proposition~\ref{prop: noise memorization}, we are ready to prove Theorem~\ref{thm: main}.
\begin{proof}[Proof of Theorem~\ref{thm: main}]
For the $j=\argmin_{j'\in\{\pm 1\}}\{T_{\max}^{j'}\}$, we are going to prove that $T_{\max}^{j}$ is bounded by a polynomial time by using contradiction. 
Specifically, we prove that
\begin{align}
    T_{\max}^j \le  T_{j,0} \coloneqq \frac{32 m}{\eta \norm{\vb}_2^2}\cdot \big(\delta-\delta(1.05 - \delta/4)^{\frac{1}{2}}\big)^{-1}\cdot \Bigg\{\log\left(\frac{2\sqrt{m\delta}}{\sigma_0 \norm{\vb}_2}\right)+1.5\cdot \frac{\|\mathbf{v}\|_2^2}{\|\mathbf{u}\|_2^2} \cdot \sqrt{\frac{1.05}{1.05 - \delta/4}}\Bigg\}.\label{eq: tmax upper bound}
\end{align}
Suppose that the result fails, then by definition we have that
\begin{align}
    T_{\max}^j = T_{(\vb)}^j \wedge T_{(\bxi)} > T_{j,0}.
\end{align}
Then Lemma~\ref{lem: boundedness} and Lemma~\ref{lem: linear accumulation} hold on $[0,T_{j,0}] $. By applying the lower bound in Inequality~\eqref{eq:exp linear acc} as well as Lemma~\ref{lem: initialization}, we have that
\begin{align}
 \max_{r\in [m]} \dotp{\wb_{j,r}^{(T_{j,0})}}{j\vb}&\geq  \max_{r\in [m]}\dotp{\wb_{j,r}^{(0)}}{j\vb} \cdot\exp\left\{\frac{\eta\|\mathbf{v}\|_2^2}{32m}\cdot\big(\delta - \delta(1.05 - \delta/4 )^{\frac{1}{2}}\big) \!\cdot\! T_{j,0}- 1.5\cdot\frac{\|\mathbf{v}\|_2^2}{\|\mathbf{u}\|_2^2} \cdot\sqrt{\frac{1.05}{1.05 - \delta/4}}\right\}\\
&\ge \frac{1}{2} \sigma_0 \norm{\vb}_2\cdot\frac{2\sqrt{m\delta}}{\sigma_0 \norm{\vb}_2  } \\
&=\sqrt{m\delta}.
\end{align}
This leads to the following,
\begin{align}
    \frac{1}{m}\sum_{r\in[m]} \sigma(\dotp{\wb^{(T_{j,0})}_{j,r}}{j\mathbf{v}})\ge \frac{1}{m}\cdot \max_{r\in[m]}(\dotp{\wb^{(T_{j,0})}_{j,r}}{j\mathbf{v}})^2\ge \delta,
\end{align}
which clearly contradicts the definition of $T^j_{(\vb)}$ in \eqref{eq: Tv j}, 
hence it must be that $T_{\max}^j\leq T_{j,0}<+\infty$.

In the following, we prove that $T_{\max}^j = T_{(\vb)}^j< T_{(\boldsymbol{\xi})}$, for which our conclusion directly follows due to \eqref{eq: Tv j}. 
Again we prove by contradiction.
Suppose that $T_{\max}^j = T_{(\boldsymbol{\xi})}< T_{(\vb)}^j$, then $\Upsilon^{(t)}$ would reach $\delta/4$ for time less than $T_{j,0}$ due to definition of $T_{(\boldsymbol{\xi})}$ in \eqref{eq: T xi}. 
But by Proposition~\ref{prop: noise memorization} with $\epsilon=1$, we know $\Upsilon^{(t)}$ takes at least 
\begin{align}
    K\widetilde{T}_{(\boldsymbol{\xi})}:= \widetilde{\Theta}\left(\frac{mn\sigma_p^2d}{\eta}\cdot\log\left(\frac{\delta}{\sigma_0\sigma_p\sqrt{d}}\right)\right)
\end{align}
steps to reach $\delta/4$, where $\widetilde{T}_{(\boldsymbol{\xi})}$ is defined in Proposition~\ref{prop: noise memorization}.
Then due to the scaling of the signal strength in Assumption~\ref{ass: conditions}, we can also find that $K\widetilde{T}_{(\boldsymbol{\xi})} > T_{j,0}$, which is a contradiction with the definition of $T_{(\boldsymbol{\xi})}$. 
Therefore, we must have that $T_{\max}^j = T_{(\vb)}^j$, which means that at $t^{\star} = T_{\max}^j$, 
\begin{align}
    \frac{1}{m}\sum_{r\in[m]}\sigma(\langle\mathbf{w}_{j,r}^{(t^{\star})},j\mathbf{v}\rangle)\ge  \frac{\delta}{2}.
\end{align}
In the meanwhile, Proposition~\ref{lem: boundedness} guarantees that at time $t^{\star} = T_{\max}^j$, 
\begin{align}
    \frac{1}{m}\sum_{r\in[m]}\sigma(\langle\mathbf{w}_{-j,r}^{(t^{\star})},j\mathbf{v}\rangle)\leq   \frac{\delta}{4}.
\end{align}
Thus, we conclude that at time $t^{\star}$, 
\begin{align}
    \frac{1}{m}\sum_{r\in[m]}\sigma(\langle\mathbf{w}_{j,r}^{(t^{\star})},j\mathbf{v}\rangle) - \frac{1}{m}\sum_{r\in[m]}\sigma(\langle\mathbf{w}_{-j,r}^{(t^{\star})},j\mathbf{v}\rangle) \geq \frac{\delta}{4}.
\end{align}
This finishes the proof of Theorem~\ref{thm: main}.
\end{proof}


\subsection{Proof of Lemma~\ref{lem: boundedness}}\label{subsec: proof boundedness}

\newcommand{\xt}{\widetilde{\bxi}} 
\begin{proof}[Proof of Lemma~\ref{lem: boundedness}]
Recall that $T_{\max}^j = T_{(\vb)}^j \wedge T_{(\bxi)}$, where
\begin{align}
    T_{(\vb)}^{j} &= \min_{t\geq 0}\left\{t:\frac{1}{m} \sum_{r\in [m]} \sigma\big(\dotp{\wb_{j,r}^{(t )}}{j\vb}\big) \ge \delta/2 \right\},\\
    T_{(\bxi)} &= \min_{t\geq 0}\left\{t: \max_{j\in\{\pm 1\},r\in[m],i\in [n]} \left|\dotp{\wb_{j,r}^{(t)}}{\bxi_i}\right| \vee \max_{j\in\{\pm 1\},r\in[m],i\in \cW}\left|\dotp{\wb_{j,r}^{(t)}}{\xt_{i}}\right|  \ge \delta/4\right\}.
\end{align}
We now define some notations to simplify our presentation. Recall that $\ft = y_{i_t}f(\xb_{i_t};\Wb^{(t)})$, $\et = 2\eta \norm{\ub}_2^2 /m$ and $\alpha = \norm{\vb}_2^2/\norm{\ub}_2^2$. 
For fixed $j$, we define
\begin{align}\label{eq: bar sjk}
\bar{S}_{j,k}\coloneqq \min\left\{s\in \NN: s>\bar{S}_{j,k-1} \text{ such that } \ft_{t_j(s)} \ge 1 \text{ and } \ft_{t_{j}(\max\{ s'<s  : i_{t_j(s')}\notin \cW\})}< 1\right\},
\end{align}
 with $\bar{S}_{j,0} = 0$, and 
 \begin{align}\label{eq: munderbar sjk}
 \munderbar{S}_{j,k}\coloneqq \min\left\{s\in \NN: s>\munderbar{S}_{j,k-1} \text{ such that } i_{t_j(s)}\notin \cW,\;\ft_{t_j(s)} < 1 \text{ and } \ft_{t_{j}(\max\{ s'<s  : i_{t_j(s')}\notin \cW\})}\ge 1\right\},
 \end{align}
 with $\munderbar{S}_{j,0} = 0$. From the definition of $T_{\max}^j$, we know that $i_{t_j(\bar{S}_{j,k})} \notin \cW$ for $k$ such that $t_j(\bar{S}_{j,k}) \le T_{\max}^j$, since otherwise $\ft_{t_j(\bar{S}_{j,k})}<\delta<1$. 
 Moreover we define 
 \begin{align}
     \cU_{j,+} \coloneqq\left\{r\in [m]:\dotp{\wb_{j,r}^{(t_j(0))}}{j\ub}>0\right\},\quad \cU_{j,-} \coloneqq\left\{r\in [m]:\dotp{\wb_{j,r}^{(t_j(0))}}{j\ub}\le 0\right\},
 \end{align}
and $\cV_{j,\pm}$ are defined analogously. 

The following analysis is nearly the same as the proof of one data case in the sense that each step of the proof is  organized exactly the same to the proof of one data case. Therefore we suggest readers to refer to the roadmap provided in Appendix~\ref{subsec: basic props one data} frequently for better understanding. 

We note that for $t\le T_{(\vb)}^j \wedge T_{(\bxi)}$, which is the time scale we are mainly focusing on, it holds that
\begin{align}
    \max_{j\in\{\pm 1\},r\in[m],i\in[n],i'\in\mathcal{W}} \left|\dotp{\wb_{j,r}^{(t)}}{\bxi_i}\right| \vee \left|\dotp{\wb_{j,r}^{(t)}}{\xt_{i'}}\right| &\le  \frac{\delta}{4},\quad \frac{1}{m}\sum_{r\in [m]} \sigma \big(\dotp{\wb_{j,r}^{(t)}}{j\vb}\big) <\frac{\delta}{2}.
\end{align}

Now we start presenting the proof of boundedness and sign stability with step-by-step analysis. The $j$ is arbitrary but fixed throughout analysis. 
In view of the effective running time for learning label-$j$-samples defined in \eqref{eq: tjs}, \eqref{eq: bar sjk}, and \eqref{eq: munderbar sjk}, the proofs for a certain $j$ directly do induction over $s$.

\paragraph{Step 1: Pre-$\bar{S}_{j,1}$ analysis.} Clearly at $s=0$, the lower bounds in Inequalities~\eqref{eq min wu negative lower bound} and~\eqref{eq min wv negative lower bound} are both guaranteed by Lemma~\ref{lem: initialization}. 
Also we know that $\ft_{t_j(0)}\ll 1$ and so $\bar{S}_{j,1}\ge 1$. The initialization also guarantees that the upper bound in Inequality~\eqref{eq wu upper bound general} holds at $s=0$.

For $s\in [0, \bar{S}_{j,1})$, the definition of $\bar{S}_{j,1}$ indicates that $\ft_{t_j(s)}\le 1$. Therefore, from Proposition~\ref{prop: independent iteration} we can see that the $\dotp{\wb_{j,r}^{t_j(s)}}{j\ub},\; r\in \cU_{j,+}$ and $\dotp{\wb_{-j,r}^{t_j(s)}}{-j\ub},\; r\in \cU_{-j, -}$ are non-decreasing in $s$ during this stage. One naturally infers that for all  $r\in \cU_{-j,-}$, it holds that
\begin{align}
    \dotp{\wb_{-j,r}^{(t_j(\bar{S}_{j,1}))}}{-j\ub} \ge \dotp{\wb_{-j,r}^{(t_j(\bar{S}_{j,1} - 1))}}{-j\ub} \ge \dotp{\wb_{-j,r}^{(t_j(0))}}{-j\ub} \ge - \MaxInitIPU.
\end{align}
Same can be verified for $\dotp{\wb_{-j,r}^{(t_j(s))}}{-j\vb}$ with $r\in \cV_{-j,-}$. Hence the lower bounds in Inequality~\eqref{eq min wu negative lower bound} and~\eqref{eq min wv negative lower bound} are extended to $s\in[0, \bar{S}_{j,k}]$. Also, for these inner products, the sign stability holds on $[0, \bar{S}_{j,1}]$, since $\et <4/5$ and thus we have
\begin{align}
    1\pm \frac{2\eta \norm{\ub}_2^2}{m} \cdot ( 1- \ft_{t_j(s)}) &\ge 1- \et \big(1+o(1)\big) \ge 0,\\
    1\pm \frac{2\eta \norm{\vb}_2^2}{m} \cdot ( 1- \ft_{t_j(s)}) &\ge 1- \alpha \et \big(1+o(1)\big)  \ge 0.
\end{align}

Then we turn to upper bound $ \dotp{\wb_{j,r}^{(t_j(s))}}{j\ub}$ for $s\in[0, \bar{S}_{j,1})$. Note that, for $s$ such that $i_{t_j(s)} \notin \cW$, the definition of $T_{\max}^j$ implies that 
\begin{align}
    1-\delta \ge \ft_{t_j(s)} &\ge  \frac{1}{m} \sum_{r\in [m]} \sigma \big(\dotp{\wb_{j,r}^{(t_j(s))}}{j\ub}\big) + \frac{1}{m} \sum_{r\in [m]}\sigma \bigl(\dotp{\wb_{j,r}^{(t_j(s))}}{j\vb} \bigr) \\
    &\qquad  - \frac{1}{m} \sum_{r\in [m]} \sigma \big(-\dotp{\wb_{-j,r}^{(t_j(s))}}{-j\ub}\big) - \frac{1}{m} \sum_{r\in [m]} \sigma \big(-\dotp{\wb_{-j,r}^{(t_j(s))}}{-j\vb}\big)\\
    &\qquad  -  \Big| \frac{1}{m} \sum_{r\in [m]} \sigma \big(\dotp{\wb_{j,r}^{(t_j(s))}}{\bxi_i}\big) + \frac{1}{m} \sum_{r\in [m]} \sigma \big(\dotp{\wb_{-j,r}^{(t_j(s))}}{\bxi_i}\big)  \Big|\\ 
    &\ge  \frac{1}{m} \sum_{r\in [m]} \sigma \big(\dotp{\wb_{j,r}^{(t_j(s))}}{j\ub}\big) - 2\times o(1) - 2\times \delta/4.
\end{align}
As a consequence, $$\frac{1}{m}\sum_{r\in [m]} \sigma \big(\dotp{\wb_{j,r}^{(t_j(s))}}{j\ub}\big)\le  1-\delta/2 +o(1)\le 1.05.$$ Thanks to the local sign stability, Lemma~\ref{lem: similar behavior} implies that 
\begin{align}
    \max_{r\in [m]} \dotp{\wb_{j,r}^{(t_j(s))}}{j\ub} \le (1.05 \beta^*_{\ub,j} m)^{1/2}.    
\end{align}
If otherwise $s\in [0,\bar{S}_{j,k-1})$ such that $i_{t_j(s)} \in \cW$, then we choose $\tilde s =\min\{s'> s: i_{t_j(s)}\notin \cW\}$. It holds that either $\tilde{s} <  \bar{S}_{j,1}$, or $\tilde{s} = \bar{S}_{j,1}$. For the previous case, similar argument implies that
\begin{align}
    \dotp{\wb_{j,r}^{(t_j(s))}}{j\ub} \le \dotp{\wb_{j,r}^{(t_j(\tilde{s} ))}}{j\ub}\le (1.05\beta^*_{\ub,j} m)^{1/2}. \label{eq: upper bound tsk-1}
\end{align}
The latter case reduces to establish the upper bound for $t_j(\bar{S}_{j,1})$, which is derived in the next step.

\paragraph{Step 2.1: Bounding $\dotp{\wb_{j,r}^{(t_j(s))}}{j\ub}$ for $s\in [\bar{S}_{j,1},\munderbar{S}_{j,1})$.} 
Since the weak data does not contributes to learning strong signal, as indicated in Proposition~\ref{prop: independent iteration}, we know that 
$$
\dotp{\wb_{j,r}^{(t_j(\bar{S}_{j,1}))}}{j\ub}=\dotp{\wb_{j,r}^{(t_j(\tilde{S}_{j,1}))}}{j\ub}\cdot \big(1+ \et  (1-\ft_{t_j(\tilde{S}_{j,1})})\big), \label{eq: s_k-1}
$$
where $\tilde{S}_{j,1}= \max\{s\le \tilde{S}_{j,1} - 1 : i_{t_j(s)}\notin \cW\}$.

We begin with a proposition that is parallel to Proposition~\ref{prop: f Tk-1 lower bound}. 

\begin{proposition}\label{prop: f tk-1 lower bound general ver}
For simplicity, we assume that $i_{t_j(\bar{S}_{j,k}-1)} \notin \cW$ and $\tilde{S}_{j,k} =\bar{S}_{j,k}-1 $. Otherwise we can find the last step before $\bar{S}_{j,k}$ such that $i_{t_j(s)} \notin \cW$ and leverage the previous observation in Equation~\eqref{eq: s_k-1}. Suppose that
\begin{align}
    E_{t_j(\bar{S}_{j,k}-1)}\coloneqq& \frac{1}{m}\sum_{r\in [m]}\sigma\big(-\dotp{\wb_{-j,r}^{(t_j(\bar{S}_{j,k}-1))}}{-j\ub}\big)+ \sigma\big(-\dotp{\wb_{j,r}^{(t_j(\bar{S}_{j,k}-1))}}{-j\vb}\big) \le \delta/4,\\
    \Upsilon^{(t_j(\bar{S}_{j,k}))}\coloneqq &\max_{j\in\{\pm 1\},r\in[m],i\in [n]} \left|\dotp{\wb_{j,r}^{(t_j(\bar{S}_{j,k}))}}{\bxi_i}\right| \vee \max_{j\in\{\pm 1\},r\in[m],i\in \cW} \left|\dotp{\wb_{j,r}^{(t_j(\bar{S}_{j,k}))}}{\xt_i}\right| \le \delta/4,
\end{align}
then we have that 
\begin{align}
    \ft_{t_j(\bar{S}_{j,k}-1)} \ge \frac{\et +2 - \sqrt{\et^2+4\et}}{2\et}, \label{eq: f Sk - 1 lower bound general}
\end{align}
and that
\begin{align}
    \ft_{t_j(\bar{S}_{j,k})} &\le \left(1+\et\big(1-\ft_{t_j(\bar{S}_{j,k}-1)}\big)\right)^2 + 2E_{t_j(\bar{S}_{j,k}-1)} + {\Upsilon^{(t_j(\bar{S}_{j,k}))}}^2\\
    &\le \bigl(\et/2 + \sqrt{\et^2/4+\et}\bigr)^2 + 2E_{t_j(\bar{S}_{j,k}-1)} + {\Upsilon^{(t_j(\bar{S}_{j,k}))}}^2. \label{eq: f Sk upper bound general} 
\end{align}
\end{proposition}
\begin{proof}[Proof of Proposition~\ref{prop: f tk-1 lower bound general ver}]
See Appendix~\ref{proof: f tk-1} for a detailed proof.
\end{proof}

With this proposition, we can derive an upper bound on $\dotp{\wb_{j,r}^{(t_j(\bar{S}_{j,1}))}}{j\ub}$. 
One step gradient \eqref{eq: wu positive s ip iteration} implies that, for $r\in \cU_{j,+}$, it holds that
\begin{align}
    \dotp{\wb_{j,r}^{(t_j(\bar{S}_{j,1}))}}{j\ub}&=\dotp{\wb_{j,r}^{(t_j(\bar{S}_{j,1}-1))}}{j\ub}\cdot  \left(1+\et \big(1-\ft_{t_j(\bar{S}_{j,1} - 1)}\big) \right)\\ 
    & \le\dotp{\wb_{j,r}^{(t_j(\bar{S}_{j,1}-1))}}{j\ub}\cdot \left( 1+ \et \left(1-\frac{\et+2- \sqrt{\et^2+4\et}}{2\et}\right)\right) \\ 
    & \le 1.5\cdot (1.05\beta^*_{\ub,j}m)^{1/2}.
\end{align}
Here the first inequality comes from Inequality~\eqref{eq: f Sk - 1 lower bound general}. Note that the upper bound on $\dotp{\wb_{j,r}^{(t_j(\bar{S}_{j,1}-1))}}{j\ub}$ has been derived in Inequality~\eqref{eq: upper bound tsk-1} in \textbf{Step 1.}. 
Taking $\et=4/5$ easily implies the second inequality above.

This upper bound continues to hold for $\dotp{\wb_{j,r}^{(t_j(s))}}{j\ub}$ with $s\in [\bar{S}_{j,1},\munderbar{S}_{j,1})$ because of monotonicity. We consider the lower bound for $\dotp{\wb_{j,r}^{(t_j(s))}}{j\ub}$ with $s\in [\bar{S}_{j,1},\munderbar{S}_{j,1})$ and $i_{t_j(s)}\notin \cW$. For these $s$, we have that
\begin{align}
 1+\delta< \ft_{ t_j(s)} &\le \frac{1}{m}\sum_{r\in [m]}\sigma\big(\dotp{\wb_{j,r}^{(t_j(s))}}{j\ub}\big) +  \frac{1}{m}\sum_{r\in [m]}\sigma\big(\dotp{\wb_{j,r}^{(t_j(s))}}{j\vb}\big) + (\text{negative part}) + (\text{noise part})\\ 
  & \le \frac{1}{m}\sum_{r\in [m]}\sigma\big(\dotp{\wb_{j,r}^{(t_j(s))}}{j\ub}\big) + \delta/2 + \delta/4. 
\end{align}
Combining with Lemma~\ref{lem: similar behavior}, we obtain that
\begin{align}
    \max_{r\in [m]} \dotp{\wb_{j,r}^{(t_j(s))}}{j\ub} \ge ( \beta^*_{\ub,j} m)^{1/2},    
\end{align}
for $s\in [\bar{S}_{j,1},\munderbar{S}_{j,1})$ and $i_{t_j(s)}\notin \cW$.

\paragraph{Step 2.2: Lower bounding $\max_r \dotp{\wb_{y,r}^{(t_j(\munderbar{S}_{j,k}))}}{y\ub}$.} Same to the proof of Proposition~\ref{prop: f tk-1 lower bound general ver}, we can assume that $i_{t_j(\munderbar{S}_{j,1} - 1)}\notin \cW$, without loss of generality. 

Note that, while  $\ft_{t_j(s)}$ is not monotonically changing in $s\in [\bar{S}_{j,1},\munderbar{S}_{j,1}]$ in the presence of noise components, the inner products are still changing monotonically, as indicated in Proposition~\ref{prop: independent iteration}. Therefore, we can leverage the control over noise components to derive an approximate monotonic property. Specifically, the inner products 
\begin{align}
    \sigma\big(\dotp{\wb_{j,r}^{(t_j(s))}}{j\ub}\big), \quad -\sigma\big(\dotp{\wb_{-j,r}^{(t_j(s))}}{j\ub}\big), \quad 
    \sigma\big(\dotp{\wb_{j,r}^{(t_j(s))}}{j\vb}\big),\quad- \sigma\big(\dotp{\wb_{-j,r}^{(t_j(s))}}{j\vb}\big), 
\end{align}
 are decreasing in $s\in [\bar{S}_{j,1},\munderbar{S}_{j,1}]$. From Inequality~\eqref{eq: ftsk expand} in the proof of Proposition~\ref{prop: f tk-1 lower bound general ver}, we know that
\begin{align}
    \ft_{t_j(\munderbar{S}_{j,1}-1)} 
&\le \ft_{t_j(\bar{S}_{j,1})} - \frac{1}{m} \sum_{r\in [m]}  \sigma\big(\dotp{\wb_{j,r}^{(t_j(\bar{S}_{j,1}))}}{\bxi_{i_{t_j(\bar{S}_{j,1})}}} \big)-\sigma\big( \dotp{\wb_{-j,r}^{(t_j(\bar{S}_{j,1}))}}{\bxi_{i_{t_j(\bar{S}_{j,1})}}} \big) \\
& \qquad + \frac{1}{m} \sum_{r\in [m]}  \sigma\big(\dotp{\wb_{j,r}^{(t_j(\munderbar{S}_{j,1}-1))}}{\bxi_{i_{t_j(\munderbar{S}_{j,1}-1)}}}\big)- \sigma\big(\dotp{\wb_{-j,r}^{(t_j(\munderbar{S}_{j,1}-1))}}{\bxi_{i_{t_j(\munderbar{S}_{j,1}-1)}}} \big) \\
&\le  \big(\et/2 + \sqrt{\et^2/4+\et}\big)^2+ 2E_{T_j(\bar{S}_{j,1}-1)}+{\Upsilon^{(t_j(\munderbar{S}_{j,1}))}}^2.\label{eq: f upper bound}
\end{align}
Recall that in \textbf{Step 1.}, we proved that $E_{j(\bar{S}_{j,1}-1)} \le 2(\MaxInitIPU)^2 =o(1)$. 
Thus from doing one-step gradient descent \eqref{eq: wu positive s ip iteration}, we know that for $r\in \cU_{j,+}$, it holds that
\begin{align}
    \dotp{\wb_{j,r}^{(t_j(\munderbar{S}_{j,1}))}}{j\ub}&=\dotp{\wb_{j,r}^{(t_j(\munderbar{S}_{j,1} - 1))}}{j\ub}\cdot  \left(1+\et \big(1-\ft_{t_j(\munderbar{S}_{j,1} - 1)}\big)  \right)\\ 
    & \ge\dotp{\wb_{j,r}^{(t_j(\munderbar{S}_{j,1}- 1) )}}{j\ub}\cdot \left( 1+ \et \big(1-(\et/2+\sqrt{\et^2/4+\et})^2\big) -0.1\right) \\ 
    & \ge 0.1\cdot (\beta^*_{\ub,j}m)^{1/2}. \label{eq: s_1 lower bound}
\end{align}
Here the first inequality is by \eqref{eq: f upper bound}, $E_{j(\bar{S}_{j,1}-1)} = o(1)$ and the definition of $T_{(\bxi)}$, and the second inequality is a consequence of Inequality~\eqref{eq: et sign lower bound}. 
This positive constant-level lower bound guarantees that our $\et$ choice is sufficient for the sign stability to hold for $s\in [\bar{S}_{j,1}, \munderbar{S}_{j,2}]$, as shown in the next step.

\paragraph{Step 2.3: Lower bounding $\dotp{\wb_{-j,r}^{(t_j(\munderbar{S}_{-j,1}))}}{-j\ub}$ and $\dotp{\wb_{-j,r}^{(t_j(\munderbar{S}_{j,1}))}}{-j\vb}$.} 
It suffices to consider $r\in \cV_{-j,-}$ and $r\in \cU_{-j,-}$.  Inequality~\eqref{eq: s_1 lower bound} indicates that
\begin{align}
    1\ll \frac{\dotp{\wb_{j,r}^{(t_j(\munderbar{S}_{j,1}))}}{j\ub}}{ \dotp{\wb_{j,r}^{(t_j(0))}}{j\ub}} = \prod_{s=0,i_{t_j(s)}\notin \cW}^{\munderbar{S}_{j,1} - 1} \left(1+\et \big(1-\ft_{t_j(s)}\big) \right) \le \exp\left\{\et \sum_{s=0,i_{t_j(s)}\notin \cW}^{\munderbar{S}_{j,1} - 1} \big(1-\ft_{t_j(s)}\big)\right\}.
\end{align}
Therefore, 
\begin{align}
    \sum_{s=0,i_{t_j(s)}\in \cW}^{\munderbar{S}_{j,1} - 1}\big(1-\ft_{t_j(s)}\big)>0.    
\end{align}
Now we can prove the lower bound. 
For $r\in \cU_{-j,-}$, we have that 
\begin{align}
    \dotp{\wb_{-j,r}^{(t_j(\munderbar{S}_{j,1}))}}{-j\ub}&=\dotp{\wb_{-j,r}^{(t_j(0))}}{-j\ub}\cdot \prod_{s=0,i_{t_j(s)}\notin \cW}^{\munderbar{S}_{j,1} - 1} \left(1-\et \big(1-\ft_{t_j(s)}\big) \right)\\ 
    &\ge  \dotp{\wb_{-j,r}^{(t_j(0))}}{-j\ub}\cdot\exp\left\{-\et \sum_{s=0,i_{t_j(s)}\notin \cW}^{\munderbar{S}_{j,1} - 1} \big(1-\ft_{t_j(s)}\big)\right\}\\
    &\ge  \dotp{\wb_{-j,r}^{(t_j(0))}}{-j\ub} \ge -\MaxInitIPU.\label{eq: S_1 neg wu lower bound}
\end{align}
On the other hand, for $s\in [0,\munderbar{S}_{j,1} - 1]$ such that $i_{t_j(s)}\in \cW$, condition $t\le T_{\max}^j = T_{(\vb)}^j \wedge T_{(\bxi)}$ guarantees that 
\begin{align}
    \ft_{t_j(s)}& =  \frac{1}{m}\sum_{r\in [m]}\sigma\big(\dotp{\wb_{j,r}^{(t_j(s))}}{j\vb}\big) - \sigma\big(-\dotp{\wb_{-j,r}^{(t_j(s))}}{-j\vb}\big)\\
    &\qquad + \frac{1}{m}\sum_{r\in [m]}\sigma\big(\dotp{\wb_{j,r}^{(t_j(s))}}{\bxi_{i_{t_j(s)}}}\big) - \sigma\big(\dotp{\wb_{-j,r}^{(t_j(s))}}{\bxi_{i_{t_j(s)}}}\big)\\
    &\qquad + \frac{1}{m}\sum_{r\in [m]}\sigma\big(\dotp{\wb_{j,r}^{(t_j(s))}}{\xt_{i_{t_j(s)}}}\big) - \sigma\big(\dotp{\wb_{-j,r}^{(t_j(s))}}{\xt_{i_{t_j(s)}}}\big) \\ 
    & \le  \delta/2  + 2\times \delta/4 \\
    &\le 1.
\end{align}
Hence we derive that
\begin{align}
    \sum_{s\in [0,\munderbar{S}_{j,1} - 1]} \bigl(1 - \ft_{t_j(s)}\bigr) =  \sum_{\substack{s\in [0,\munderbar{S}_{j,1} - 1] \\ i_{t_j(s)}\in \cW}} \bigl(1 - \ft_{t_j(s)}\bigr)+\sum_{\substack{s\in [0,\munderbar{S}_{j,1} - 1] \\ i_{t_j(s)}\notin \cW}} \bigl(1 - \ft_{t_j(s)}\bigr) \ge 0.\label{eq: sum 1-yf over all data non negative}
\end{align}
And in consequence, for $r\in \cV_{j,-}$ it holds that
\begin{align}
    \dotp{\wb_{-j,r}^{(t_j(\munderbar{S}_{j,1 }))}}{-j\vb} &= \dotp{\wb_{-j,r}^{(t_j(0))}}{-j\vb} \cdot \prod_{s=0}^{\munderbar{S}_{j,1} - 1}  \left(1-\alpha\et \big(1-\ft_{t_j(s)}\big) \right) \\
    &\ge \dotp{\wb_{-j,r}^{(t_j(0))}}{-j\vb} \cdot \exp\left\{-\alpha\et \sum_{s=0}^{\munderbar{S}_{j,1} - 1}\big(1-\ft_{t_j(s)}\big) \right\} \\
    &\ge \dotp{\wb_{-j,r}^{(t_j(0))}}{-j\vb}\ge -\MaxInitIPV. \label{eq: S_1 neg wv lower bound}
\end{align}
The monotonicity again extends lower bounds in Inequalities~\eqref{eq: S_1 neg wv lower bound} and~\eqref{eq: S_1 neg wu lower bound} to $s\in [\bar{S}_{j,1}, \bar{S}_{j,2}]$. And the sign stability naturally holds on this interval.

\paragraph{Step 2.4 \& Finalizing.} Thanks to the sign stability proved in the last step, we can now apply Lemma~\ref{lem: similar behavior} to derive that the upper bound in Inequality~\eqref{eq wu upper bound general} continues to hold for all $s\in [\munderbar{S}_{j,1}, \bar{S}_{j,2}]$, with exactly the same argument as the \textbf{Step 1.}. With an inductive argument that exactly repeats the argument above, we can infer that all the results in Lemma~\ref{lem: boundedness} holds for $t_j(s)$ with $s\in [\bar{S}_{j,1}, \bar{S}_{j,2}]$. Proposition~\ref{prop: independent iteration} implies that for any $t\in \NN$ we can find $s_t \in \NN$ such that $t_j(s_t)\le t < t_j(s_t+1)$ and $\dotp{\wb_{j,r}^{(t)}}{j\ub}=\dotp{\wb_{j,r}^{(t_j(s_t))}}{j\ub}$ (so are all other inner products), therefore all the results in Lemma~\ref{lem: boundedness} hold for $t\le T_{\max}^j$.
\end{proof}

\subsection{Proof of Lemma~\ref{lem: linear accumulation}}\label{subsec: proof linear accumulation}

\begin{proof}[Proof of Lemma \ref{lem: linear accumulation}]
    For any $0\leq t_1\leq t_2 \leq T_{\max}^j$, $j\in\{\pm 1\}$, and $r\in\cU_{j,+}^{(t)}$, by \eqref{eq: strong signal update}, 
    \begin{align}
        \langle \mathbf{w}_{j,r}^{(t_2+1)},j\mathbf{u}\rangle &= \langle \mathbf{w}_{j,r}^{(t_1)},j\mathbf{u}\rangle + \frac{2\eta\|\mathbf{u}\|_2^2}{m}\cdot\sum_{\substack{t_1\leq s\leq t_2\\y_{i_s} = j,i_s\notin\mathcal{W}\\y_{i_s}f(\mathbf{x}_{i_s};\mathbf{W}^{(s)}) > 1}}\big(1 - y_{i_s}f(\mathbf{x}_{i_s};\mathbf{W}^{(s)})\big)\cdot\langle \mathbf{w}_{j,r}^{(s)},j\mathbf{u}\rangle \\ 
        &\qquad+  \frac{2\eta\|\mathbf{u}\|_2^2}{m}\cdot\sum_{\substack{t_1\leq s\leq t_2\\y_{i_s} = j,i_s\notin\mathcal{W}\\y_{i_s}f(\mathbf{x}_{i_s};\mathbf{W}^{(s)}) < 1}}\big(1 - y_{i_s}f(\mathbf{x}_{i_s};\mathbf{W}^{(s)})\big)\cdot\langle \mathbf{w}_{j,r}^{(s)},j\mathbf{u}\rangle,\label{eq: proof linear accumulation 1}
    \end{align}
    Note that for $0\leq t_1\leq t_2\leq T_{\max}^j$, we can apply the conclusions of Lemma~\ref{lem: boundedness}. 
    Specifically, we consider the maximal neuron $r^* = \argmax_{r\in[m]} \dotp{\wb_{j,r}^{(t)}}{j\ub}$, and 
    \begin{align}
         {(1.05)^{\frac{1}{2}}}\cdot(\beta^*_{\ub,j} m)^{\frac{1}{2}} &\geq \left|\langle \mathbf{w}_{j,r^*}^{(t_2+1)},j\mathbf{u}\rangle - \langle \mathbf{w}_{j,r^*}^{(t_1)},j\mathbf{u}\rangle\right|  \label{eq: proof linear accumulation 2}\\
        & = \frac{2\eta\|\mathbf{u}\|_2^2}{m}\cdot\left|\sum_{\substack{t_1\leq s\leq t_2\\y_{i_s} = j,i_s\notin\mathcal{W}\\y_{i_s}f(\mathbf{x}_{i_s};\mathbf{W}^{(s)}) > 1}}\big(1 - y_{i_s}f(\mathbf{x}_{i_s};\mathbf{W}^{(s)})\big)\cdot\underbrace{\langle \mathbf{w}_{j,r^*}^{(s)},j\mathbf{u}\rangle}_{\textcolor{red}{> (\beta^*_{\ub,j} m)^{\frac{1}{2}}}} \right.\\
        &\qquad \left.- \sum_{\substack{t_1\leq s\leq t_2\\y_{i_s} = j,i_s\notin\mathcal{W}\\y_{i_s}f(\mathbf{x}_{i_s};\mathbf{W}^{(s)}) < 1}}\big(1 - y_{i_s}f(\mathbf{x}_{i_s};\mathbf{W}^{(s)})\big)\cdot\underbrace{\langle \mathbf{w}_{j,r^*}^{(s)},j\mathbf{u}\rangle}_{\textcolor{red}{< (1.05 - \delta/4)^{\frac{1}{2}}\cdot(\beta^*_{\ub,j} m)^{\frac{1}{2}}}} \right|,
    \end{align}
    where the red remarks follow from Lemma~\ref{lem: similar behavior}.
    Rearranging terms, we conclude from \eqref{eq: proof linear accumulation 2} that
    \begin{align}
        &\sum_{\substack{t_1\leq s\leq t_2\\y_{i_s} = j,i_s\notin\mathcal{W}\\y_{i_s}f(\mathbf{x}_{i_s};\mathbf{W}^{(s)}) < 1}}\big(1 - y_{i_s}f(\mathbf{x}_{i_s};\mathbf{W}^{(s)})\big)\label{eq: proof linear accumulation 3} \\
        &\qquad \geq \left(1.05 - \delta/4 \right)^{-\frac{1}{2}}\cdot \sum_{\substack{t_1\leq s\leq t_2\\y_{i_s} = j,i_s\notin\mathcal{W}\\y_{i_s}f(\mathbf{x}_{i_s};\mathbf{W}^{(s)}) > 1}} \big(y_{i_s}f(\mathbf{x}_{i_s};\mathbf{W}^{(s)})-1\big) - \frac{m {(1.05)^{\frac{1}{2}}}}{2\eta\|\mathbf{u}\|_2^2(1.05 - \delta/4)^{\frac{1}{2}}},
    \end{align}
    Under Assumption \ref{ass: conditions}, with probability at least $1-1/\mathrm{poly}(d)$, it holds that 
    \begin{align}
        \left|\left\{t_1\leq s\leq t_2:y_{i_s} = j,i_s\notin\mathcal{W}\right\}\right|\geq \frac{1}{4}\cdot(t_2 -t_1 + 1),\label{eq: proof linear accumulation 4}
    \end{align}
    Combining \eqref{eq: proof linear accumulation 3} and \eqref{eq: proof linear accumulation 4}, we can finally prove that
    \begin{align}
        \sum_{\substack{s=t_1 \\ y_{i_s} = j,i_s\notin\mathcal{W}}}^{t_2} \!\!\!\!\!\!1 - y_{i_s}f(\mathbf{x}_{i_s};\mathbf{W}^{(s)}) \geq \frac{\delta}{8}\!\cdot \!\big(1 - (1.05 - \delta/4)^{\frac{1}{2}}\big)\!\cdot\! (t_2 - t_1 + 1) - \frac{m {(1.05)^{\frac{1}{2}}}}{2\eta\|\mathbf{u}\|_2^2(1.05 - \delta/4)^{\frac{1}{2}}}.\label{eq: proof linear accumulation 5}
    \end{align}
    Finally, for the weak data $i_s\in\mathcal{W}$, under Assumption \ref{ass: conditions}, with probability at least $1-1/\mathrm{poly}(d)$,
    \begin{align}
        \left|\left\{t_1\leq s\leq t_2:y_{i_s} = j,i_s\in\mathcal{W}\right\}\right| \leq 2\rho\cdot(t_2 -t_1 + 1) \leq \frac{\delta}{32}\cdot\big(1 - (1.05 - \delta/4)^{\frac{1}{2}}\big)\cdot (t_2 - t_1 + 1),\label{eq: proof linear accumulation 6}
    \end{align}
    where the second inequality follows from the condition on $\rho$ by Assumption \ref{ass: conditions}.
    By Lemma~\ref{lem: boundedness}, we have that $|1 - y_{i_s}f(\xb_{i_s};\mathbf{W}^{(s)})|\leq 2$ for $0\leq s\leq T_{\max}^j$ and $y_s = j$.
    Therefore, we have that 
    \begin{align}
        \sum_{\substack{s=t_1 \\ y_{i_s} = j,i_s\in\mathcal{W}}}^{t_2} 1 - y_{i_s}f(\mathbf{x}_{i_s};\mathbf{W}^{(s)}) \geq -\frac{\delta}{16}\cdot\big(1 - (1.05 - \delta/4)^{\frac{1}{2}}\big)\cdot (t_2 - t_1 + 1).\label{eq: proof linear accumulation 7}
    \end{align}
    Combining \eqref{eq: proof linear accumulation 5} and \eqref{eq: proof linear accumulation 7}, we can conclude that
    \begin{align}
        \sum_{\substack{s=t_1 \\ y_{i_s} = j}}^{t_2} 1 - y_{i_s}f(\mathbf{x}_{i_s};\mathbf{W}^{(s)}) \geq \frac{\delta}{16}\!\cdot \!\big(1 - (1.05 - \delta/4)^{\frac{1}{2}}\big)\!\cdot\! (t_2 - t_1 + 1) - \frac{m {(1.05)^{\frac{1}{2}}}}{2\eta\|\mathbf{u}\|_2^2(1.05 - \delta/4)^{\frac{1}{2}}}.
    \end{align}
    This finishes the proof of the first part in Lemma~\ref{lem: linear accumulation}.
Now we consider the second part. For simplicity, we denote by $\alpha = \norm{\vb}_2^2 / \norm{\ub}_2^2$ and $\tilde{\eta} = 2\|\mathbf{u}\|_2^2/m$.
Consider that for any $0 \leq t\leq T_{\max}^j$, $j\in\{\pm 1\}$, and $r\in\cV_{j,+}^{(t)}$, due to Lemma~\ref{lem: boundedness}, the $\vb$-sign stability condition is true on $[0, T_{\max}^j]$.
In view of \eqref{eq: weak signal update}, this means that
\begin{align}
1+ \alpha\tilde{\eta} \cdot\big(1-y_{i_s}f(\xb_{i_s};\Wb^{(s)})\big) > 0, \quad \forall 0\leq s\le T_{\max}^j\,\,\,\text{s.t.}\,\,\,y_{i_s} = j.
\end{align}
Then for any $t_1 \le t \le T_{\max}^j$, $t_1\leq s\leq t$,  and $r\in \cC_{j,+}^{(t)}$, since $-2\le 1-y_{i_s}f(\xb_{i_s};\Wb^{(t)})\le 2$ due to Lemma~\ref{lem: boundedness}, we can lower bound the relative increment as
\begin{align}
    \sum_{\substack{s=t_1\\y_{i_s} = j}}^{t} &\log \Big(1+\alpha \et \big(1- y_{i_s}f(\xb_{i_s};\Wb^{(s)})\big)\Big)= \sum_{\substack{s=t_1\\y_{i_s} = j}}^{t} \int_0^\alpha \frac{\et\big(1- y_{i_s}f(\xb_{i_s};\Wb^{(s)})\big) }{1+\et z \big(1- y_{i_s}f(\xb_{i_s};\Wb^{(s)})\big)} dz \\
    &= \sum_{\substack{s=t_1\\y_{i_s} = j}}^{t} \int_0^\alpha \frac{\et\Big(\big(1- y_{i_s}f(\xb_{i_s};\Wb^{(s)})\big) + 2\Big) }{1+\et z \big(1- y_{i_s}f(\xb_{i_s};\Wb^{(s)})\big)} dz 
    -  \sum_{\substack{s=t_1\\y_{i_s} = j}}^{t} \int_0^\alpha \frac{2\et }{1+\et z \big(1- y_{i_s}f(\xb_{i_s};\Wb^{(s)})\big)} dz \\
    &\ge  \sum_{\substack{s=t_1\\y_{i_s} = j}}^{t} \int_0^\alpha \frac{\et\Big(\big(1- y_{i_s}f(\xb_{i_s};\Wb^{(s)})\big) + 2\Big) }{1+2\et z} dz -  \sum_{\substack{s=t_1\\y_{i_s} = j}}^{t} \int_0^\alpha \frac{2\et }{1-2\et z} dz\\
    &\ge \int_0^\alpha \frac{\et\Big(\sum_{s=t_1,y_{i_s} = j}^{t}\big(1- y_{i_s}f(\xb_{i_s};\Wb^{(t)})\big) + 2N_j(t_0,t)\Big) }{1+2\et z} dz -  \int_0^\alpha \frac{2\et N_j(t_1,t)}{1-2\et z} dz, \label{eq: proof signal learning 1}
\end{align}
where for simplicity we denote $N_j(t_1,t) = |\{t_1\leq s\leq t:y_{i_s} = j\}|$.
Now we can use Proposition~\ref{lem: linear accumulation} to lower bound the right hand side of \eqref{eq: proof signal learning 1},
Denoting $\epsilon = \delta\cdot(1-(1.05 - \delta/4)^{\frac{1}{2}})/16$, we have that
\begin{align}
    \sum_{\substack{s=t_1\\y_{i_s} = j}}^{t}&\log \Big\{1+\alpha \et \big(1- y_{i_s}f(\xb_{i_s};\Wb^{(s)})\big)\Big\}\\
    &\ge \int_0^\alpha \frac{\et\Big( \epsilon (t - t_1 +1) - \Delta + 2N_j(t_1,t)\Big) }{1+2\et z} dz -  \int_0^\alpha \frac{2\et N_j(t_1,t)}{1-2\et} dz \\ 
    &\ge \bigg(\frac{\epsilon}{2}(t -t_1 +1 ) - \frac{\Delta}{2}  + N_j(t_1,t)\bigg)\cdot \log(1+2\alpha\et) + N_j(t_1,t)\cdot\log(1-2\alpha\et )\\
    &\ge \bigg(\frac{\epsilon}{2}(t -t_1 +1 ) - \frac{\Delta}{2} \bigg)\cdot \log(1+2\alpha\et)  + N_j(t_1,t)\cdot \log\big\{(1+2\alpha\et)\cdot(1-2\alpha\et )\},
\end{align}
where $\Delta$ is defined in Lemma~\ref{lem: linear accumulation}.
Moreover since for our choice of $\alpha\et \ll 1$ in Assumption~\ref{ass: conditions},
\begin{align}
    \log(1+2\alpha\et)\geq \alpha\et,\quad \log\big\{(1+2\alpha\et)\cdot(1-2\alpha\et )\big\} = \log \big\{1-4 \alpha^2\et^2\big\}
    \ge -2\alpha^2\et^2,
\end{align}
and using the fact that with probability at least $1-1/\mathrm{poly}(d)$ it holds that $N_j(t_1,t) \leq (t - t_1 + 1) / 2$,
 we finally have that
\begin{align}
    \sum_{\substack{s=t_1\\y_{i_s} = j}}^{t} \log \Big\{1+\alpha \et \big(1- y_{i_s}f(\xb_{i_s};\Wb^{(s)})\big)\Big\}&\ge \frac{\alpha\et}{2}\!\cdot\!(\epsilon -2\alpha\et)\!\cdot\! (t_1-t_1+1)-\Delta \cdot \log(1+2\alpha \et) \\
    &\ge \frac{1}{4}\alpha\et\cdot \epsilon\cdot (t_1- t_1 + 1)-2\alpha \et\cdot \Delta.
\end{align}
Finally, using \eqref{eq: weak signal update} again, we can lower bound our target as 
\begin{align}
    \dotp{\wb_{j,r}^{(t+1)}}{j\vb} &=\dotp{\wb_{j,r}^{(t_1)}}{j\vb} \cdot\prod _{\substack{s=t_1\\y_{i_s} = j}}^{t} \Big(1+\alpha\et \big(1- y_{i_s}f(\xb_{i_s};\Wb^{(s)})\big)\Big) \\
    &= \dotp{\wb_{j,r}^{(t_1)}}{j\vb} \cdot\exp\left\{\sum_{\substack{s=t_1\\y_{i_s} = j}}^{t} \log \Big\{1+\alpha \et \big(1- y_{i_s}f(\xb_{i_s};\Wb^{(s)})\big)\Big\}\right\} \\
    &\geq \dotp{\wb_{j,r}^{(t_1)}}{j\vb} \cdot\exp\left\{\frac{1}{4}\alpha\et\cdot \epsilon\cdot (t_1- t_1 + 1)-2\alpha \et\cdot \Delta\right\}.
\end{align}
Plugging in the definition of $\epsilon$, $\Delta$, $\alpha$, and $\et$, we can arrive at 
\begin{align}
    \dotp{\wb_{j,r}^{(t+1)}}{j\vb} \geq \dotp{\wb_{j,r}^{(t_1)}}{j\vb} \cdot\exp\left\{\frac{\eta\|\mathbf{v}\|_2^2}{32m}\!\cdot\!\big(\delta - \delta(1.05 - \delta/4 )^{\frac{1}{2}}\big) \!\cdot\!(t - t_1 + 1) - \frac{\|\mathbf{v}\|_2^2 {(1.05 )^{\frac{1}{2}}}}{\|\mathbf{u}\|_2^2(1.05 - \delta/4)^{\frac{1}{2}}}\right\}.
\end{align}
This finishes the proof of Lemma~\ref{lem: linear accumulation}.
\end{proof}

\subsection{Proof of Technical Results}\label{sec: technical lemmas}

\subsubsection{Proof of Proposition~\ref{prop: independent iteration}}\label{subsubsec: proof: independent iter}

\begin{proof}[Proof of Proposition~\ref{prop: independent iteration}]\label{proof: independent iter}
    From Equation~\eqref{eq: strong signal update}, for $r\in \cU_{j,+}$ and $t'\in  (t_j(s),t_j(s+1)]$, we can infer that 
    \begin{align}
        \dotp{\wb_{j,r}^{(t')}}{j\ub}&=\dotp{\wb_{j,r}^{(t_j(s))}}{j\ub} + \sum_{t = t_j(s)}^{t'- 1} \frac{\eta\norm{\ub}_2^2}{m}\cdot \big(1 -\ft_{t}\big)\cdot\sigma' \bigl(\dotp{\wb_{j,r}^{(t)}}{j\ub}\cdot jy_{i_t}\bigr)\cdot \mathbf{1}\{i_t \in \cW \}.
    \end{align}
    The definition of $t_j(s)$ implies that for $t\in (t_j(s), t_j(s+1))$, $y_{i_t}j =-1$. On the other hand $y_{i_{t_j(s)}} j = 1$, hence we obtain that
    \begin{align}
        \dotp{\wb_{j,r}^{(t')}}{j\ub}&=\dotp{\wb_{j,r}^{(t_j(s))}}{j\ub} + \frac{2\eta\norm{\ub}_2^2}{m}\cdot \big(1 -\ft_{t_j(s)}\big)\cdot\dotp{\wb_{j,r}^{(t)}}{j\ub}\cdot \mathbf{1}\{i_{t_j(s)} \in \cW \}.
    \end{align}
    Since the sign stability holds, this multiplication by a non-negative factor does not change the sign of the inner products. Therefore we have that
    \begin{align}
        \dotp{\wb_{j,r}^{(t_j(s+1))}}{j\ub} &= \dotp{\wb_{j,r}^{(t_j(s+1) -1)}}{j\ub}  = \cdots = \dotp{\wb_{j,r}^{(t_j(s) +1 )}}{j\ub} \\
        & = \dotp{\wb_{j,r}^{(t_j(s))}}{j\ub} \cdot \big(1+\et\cdot(1-\ft_{t_j(s)})\big).
    \end{align}
    which concludes our result for $\dotp{\wb_{j,r}^{(t)}}{j\ub},\;r\in \cU_{j,+}$. Other result can be proved analogously and are omitted here. 
\end{proof}

\subsubsection{Proof of Lemma~\ref{lem: similar behavior} for Multiple Data Setting}\label{subsubsec: proof similar behavior multiple data setting}

\begin{proof}[Proof of Lemma~\ref{lem: similar behavior} (multiple data setting)]
In the multiple data setting, the (positive) inner products only changes at the steps where the corresponding data label aligns with the directions of the neurons (i.e., $j= \pm 1$). Define
\begin{align}
\beta^{*,(t_1)}_{\ub,j}&= \frac{\max_{r\in[m]}\sigma(\dotp{\wb_{j,r}^{(t_1)}}{j\ub})}{\sum_{r\in[m]} \sigma(\dotp{\wb_{j,r}^{(t_1)}}{j\ub})}.
\end{align}

Again, the local sign stability assumption ensures that each inner product grows proportionally and the superscript $(t)$ in the neuron index sets can be dropped, with
\begin{align}
   \frac{1}{m} \sum_{r\in [m]} \sigma\big(\dotp{\wb_{j ,r}^{(t)}}{j \ub}\big) & = \frac{1}{m} \sum_{r\in \cU_{j,+}} \sigma\big( \dotp{\wb_{j,r}^{(t_1 )}}{j \ub}\big) \cdot \prod_{\substack{t'\in [t_1,t-1]:\\ y_{i_t}=j,\;i_t\notin \cW}} \left(1+ \frac{2\eta\norm{\ub}_2^2}{m} \cdot\big( 1- y_{i_{t'}}f( \xb_{i_{t'}}; \Wb^{(t')})\big)\right)^2 \\ 
     &= \frac{\max_{r\in [m]}\sigma\big( \dotp{\wb_{j,r}^{(t_1)}}{j\ub} \big)}{m\beta^{*,(t_1)}_{\ub,j}}   \cdot \prod_{\substack{t'\in [t_1,t-1]:\\ y_{i_{t'}}=j,\;i_t\notin \cW}} \left(1+ \frac{2\eta\norm{\ub}_2^2}{m} \cdot \big( 1- y_{i_{t'}}f( \xb_{i_{t'}}; \Wb^{(t')})\big)\right)^2 \\
     &= \frac{\sigma\big(\max_{r\in [m]} \dotp{\wb_{j,r}^{(t)}}{j \ub} \big)}{m\beta^{*,(t_1)}_{\ub,j}} 
     \label{eq: g(x,y) mutiple}
\end{align}
Here the first and the third equality is true because Equation~\eqref{eq: ip positive gd one data} implies that all the positive $\dotp{\wb_{y,r}^{(t)}}{y\ub}$ iterates by sequentially multiplying the same factor 
\begin{align}
    1+ \frac{2\eta\norm{\ub}_2^2}{m} \cdot \big( 1- yf( \xb; \Wb^{(t')})\big).    
\end{align}
The second equality comes from the definition of $\beta^{*,(t_1)}_{\ub}$ in Lemma~\ref{lem: single neuron behaves similarly}. 

Therefore, $m^{-1}\cdot \sum_{r\in [m]} \sigma(\dotp{\wb_{j,r}^{(t)}}{j\ub})>c$ implies that $\sigma(\max_{r\in [m]} \dotp{\wb_{y,r}^{(t)}}{y \ub})>\beta^{*,(t_1)}_{\ub,j}mc$ and the desired lower bound follows. The upper bound can be proved analogously. 
\end{proof}

\subsubsection{Proof of Proposition \ref{prop: noise memorization}}\label{subsubsec: proof: noise memorization}
\begin{proof}[Proof of Proposition \ref{prop: noise memorization}]\label{subsec: proof noise memorization}
    We prove the result by induction.
    For the step $t = t_1$, the result holds trivially.
    Suppose that this result holds for each step $t_1,\cdots,t$, then 
    by \eqref{eq: noise update}, we have that
    \begin{align}
        \langle \mathbf{w}_{j,r}^{(t+1)},\boldsymbol{\xi}_{i}\rangle &= \langle \mathbf{w}_{j,r}^{(t_1)},\boldsymbol{\xi}_i\rangle + \sum_{s=t_1}^t\frac{\eta\cdot jy_{i_s}}{m}\cdot \big(1 - y_{i_s} f(\mathbf{x}_{i_s};\mathbf{W}^{(s)})\big)\cdot\left(\sigma^{\prime}(\langle\mathbf{w}_{j,r}^{(s)}\cdot \boldsymbol{\xi}_{i_s}\rangle)\cdot\langle\boldsymbol{\xi}_{i_s},\boldsymbol{\xi}_i\rangle\right.\\
    & \qquad +\left.\sigma^{\prime}(\langle\mathbf{w}_{j,r}^{(s)}\cdot \widetilde{\boldsymbol{\xi}}_{i_s}\rangle)\cdot\langle\widetilde{\boldsymbol{\xi}}_{i_s},\boldsymbol{\xi}_i\rangle\cdot\mathbf{1}\{i_s\in\mathcal{W}\}\right)\\
    & = \langle \mathbf{w}_{j,r}^{(t_1)},\boldsymbol{\xi}_i\rangle + \sum_{\substack{s=t_1\\i_s = i}}^t\frac{\eta\|\boldsymbol{\xi}_i\|_2^2\cdot jy_{i_s}}{m}\cdot \big(1 - y_{i} f(\mathbf{x}_{i};\mathbf{W}^{(s)})\big)\cdot\sigma^{\prime}(\langle\mathbf{w}_{j,r}^{(s)}\cdot \boldsymbol{\xi}_{i}\rangle)\\
    & \qquad + \sum_{\substack{s=t_1\\i_s\neq i}}^t\frac{\eta\|\boldsymbol{\xi}_i\|_2^2\cdot jy_{i_s}}{m}\cdot \big(1 - y_{i_s} f(\mathbf{x}_{i_s};\mathbf{W}^{(s)})\big)\cdot\left(\sigma^{\prime}(\langle\mathbf{w}_{j,r}^{(s)}\cdot \boldsymbol{\xi}_{i_s}\rangle)\cdot\frac{\langle\boldsymbol{\xi}_{i_s},\boldsymbol{\xi}_i\rangle}{\|\boldsymbol{\xi}_i\|_2^2}\right.\\
    & \qquad \qquad +\left.\sigma^{\prime}(\langle\mathbf{w}_{j,r}^{(s)}\cdot \widetilde{\boldsymbol{\xi}}_{i_s}\rangle)\cdot\frac{\langle\widetilde{\boldsymbol{\xi}}_{i_s},\boldsymbol{\xi}_i\rangle}{\|\boldsymbol{\xi}_i\|_2^2}\cdot\mathbf{1}\{i_s\in\mathcal{W}\}\right)
    \end{align}
    Taking absolute value, we have that
    \begin{align}
        \left|\langle \mathbf{w}_{j,r}^{(t+1)},\boldsymbol{\xi}_{i}\rangle\right| &\leq  \left|\langle\mathbf{w}_{j,r}^{(t_1)},\boldsymbol{\xi}_i\rangle\right| + \left|\sum_{\substack{s=t_1\\i_s = i}}^t\frac{\eta\|\boldsymbol{\xi}_i\|_2^2\cdot jy_{i_s}}{m}\cdot \big(1 - y_{i} f(\mathbf{x}_{i};\mathbf{W}^{(s)})\big)\cdot\sigma^{\prime}(\langle\mathbf{w}_{j,r}^{(s)}\cdot \boldsymbol{\xi}_{i}\rangle)\right|\\
    & \qquad + \left|\sum_{\substack{s=t_1\\i_s\neq i}}^t\frac{\eta\|\boldsymbol{\xi}_i\|_2^2\cdot jy_{i_s}}{m}\cdot \big(1 - y_{i_s} f(\mathbf{x}_{i_s};\mathbf{W}^{(s)})\big)\right.\\
    & \qquad\qquad\left.\cdot\left(\sigma^{\prime}(\langle\mathbf{w}_{j,r}^{(s)}\cdot \boldsymbol{\xi}_{i_s}\rangle)\cdot\frac{\langle\boldsymbol{\xi}_{i_s},\boldsymbol{\xi}_i\rangle}{\|\boldsymbol{\xi}_i\|_2^2}
     +\sigma^{\prime}(\langle\mathbf{w}_{j,r}^{(s)}\cdot \widetilde{\boldsymbol{\xi}}_{i_s}\rangle)\cdot\frac{\langle\widetilde{\boldsymbol{\xi}}_{i_s},\boldsymbol{\xi}_i\rangle}{\|\boldsymbol{\xi}_i\|_2^2}\cdot\mathbf{1}\{i_s\in\mathcal{W}\}\right)\right| 
\end{align}
Applying the definition of $\Upsilon^{(s)}$, we further obtain that 
\begin{align}
     \left|\langle \mathbf{w}_{j,r}^{(t+1)},\boldsymbol{\xi}_{i}\rangle\right|&\leq \Upsilon^{(t_1)} + \sum_{\substack{s=t_1\\i_s = i}}^t\frac{2\eta\|\boldsymbol{\xi}_i\|_2^2}{m}\cdot \big|1 - y_{i} f(\mathbf{x}_{i};\mathbf{W}^{(s)})\big|\cdot\Upsilon^{(s)}\\
    &\qquad + \sum_{\substack{s=t_1\\i_s\neq i}}^t\frac{2\eta\|\boldsymbol{\xi}_i\|_2^2}{m}\cdot \big|1 - y_{i_s} f(\mathbf{x}_{i_s};\mathbf{W}^{(s)})\big|\cdot\Upsilon^{(s)}\cdot\left(\frac{|\langle\boldsymbol{\xi}_{i_s},\boldsymbol{\xi}_i\rangle|}{\|\boldsymbol{\xi}_i\|_2^2}
     +\frac{|\langle\widetilde{\boldsymbol{\xi}}_{i_s},\boldsymbol{\xi}_i\rangle|}{\|\boldsymbol{\xi}_i\|_2^2}\right).\label{eq: proof noise memorization}
\end{align}
Now using Lemma~\ref{lem: boundedness}, for $s\leq \min_{j\in \{\pm 1\}}\{T_{\max}^j\}$, we have that $|1 - y_{i_s}f(\mathbf{x}_{i_s};\mathbf{W}^{(s)})|\leq 2$.
Also, by Lemma~\ref{lem: noise norm and correlation},  it holds that$\|\boldsymbol{\xi}_i\|_2^2\leq 3\sigma_p^2d/2$ and $|\langle\boldsymbol{\xi},\boldsymbol{\xi}'\rangle|/\|\boldsymbol{\xi}\|_2^2 \leq \widetilde{\mathcal{O}}(d^{-1/2})$.
Thus we can further upper bound \eqref{eq: proof noise memorization} as 
\begin{align}
    \max_{i\in[n]}\left|\langle \mathbf{w}_{j,r}^{(t+1)},\boldsymbol{\xi}_{i}\rangle\right| &\leq \Upsilon^{(t_1)} + \frac{6\eta\sigma_p^2d}{m}\cdot\left(\sum_{\substack{s=t_1\\i_s= i}}^t\Upsilon^{(s)} + \widetilde{\cO}(d^{-1/2})\cdot\sum_{\substack{s=t_1\\i_s\neq i}}^t\Upsilon^{(s)}\right) \label{eq: proof noise memorization 2}\\
    & =  \Upsilon^{(t_1)} + \frac{6\eta\sigma_p^2d}{m}\cdot\left(\sum_{\substack{s=t_1\\i_s= i}}^t\Upsilon^{(t_1)} + \widetilde{\cO}(d^{-1/2})\cdot\sum_{\substack{s=t_1\\i_s\neq i}}^t\Upsilon^{(t_1)}\right) \\
    &\qquad + \frac{6\eta\sigma_p^2d}{m}\cdot\left(\sum_{\substack{s=t_1\\i_s= i}}^t\big(\Upsilon^{(t_1)} - \Upsilon^{(s)}\big) + \widetilde{\cO}(d^{-1/2})\cdot\sum_{\substack{s=t_1\\i_s\neq i}}^t\big(\Upsilon^{(t_1)} - \Upsilon^{(s)}\big) \right).
\end{align}
By our induction, we have that $\Upsilon^{(s)} - \Upsilon^{(t_1)} \leq \Upsilon^{(t_1)}\cdot\epsilon$, for which we can further bound \eqref{eq: proof noise memorization 2} as 
\begin{align}
    \max_{i\in[n]}\left|\langle \mathbf{w}_{j,r}^{(t+1)},\boldsymbol{\xi}_{i}\rangle\right| &\leq \Upsilon^{(t_1)}\cdot\left[1+\frac{6\eta\sigma_p^2d}{m}\cdot(1+\epsilon)\cdot\left(\sum_{\substack{s=t_1\\i_s= i}}^t1+\widetilde{\cO}(d^{-1/2})\cdot\sum_{\substack{s=t_1\\i_s\neq i}}^t1\right)\right] \\
    &\leq \Upsilon^{(t_1)}\cdot\left[1+\frac{6\eta\sigma_p^2d}{m}\cdot(1+\epsilon)\cdot\left(\frac{2(t - t_1 + 1)}{n}+\widetilde{\cO}(d^{-1/2})\cdot(t - t_1 + 1)\right)\right] \\
    &\leq \Upsilon^{(t_1)}\cdot\left[1+\frac{18\eta\sigma_p^2d}{mn}\cdot(1+\epsilon)\cdot(t - t_1 + 1)\right],
\end{align}
where in the first inequality we have utilized the fact that $|\{t_1\leq s\leq t:i_s=i\}|\leq2(t-t_1+1)/n$, and in the last inequality we apply the condition that $d = \widetilde{\Omega}(n^2)$ by Assumption~\ref{ass: conditions}.
Therefore, when $t\leq t_1 + \widetilde{T}_{(\boldsymbol{\xi})}-1$ with $\widetilde{T}_{(\boldsymbol{\xi})} = \widetilde{\Theta}(mn\cdot\eta^{-1}\cdot\epsilon\cdot(1+\epsilon)^{-1}\cdot(\sigma_p^2d)^{-1})$, it holds that 
\begin{align}
    \max_{i\in[n]}\left|\langle \mathbf{w}_{j,r}^{(t+1)},\boldsymbol{\xi}_{i}\rangle\right| &\leq \Upsilon^{(t_1)}\cdot (1+\epsilon).\label{eq: proof noise memorization 3}
\end{align}
By using the same argument as proving \eqref{eq: proof noise memorization 3}, we can also show that for $t\leq t_1 + \widetilde{T}_{(\boldsymbol{\xi})}-1$
\begin{align}
    \max_{i\in\mathcal{W}}\left|\langle \mathbf{w}_{j,r}^{(t+1)},\widetilde{\boldsymbol{\xi}}_{i}\rangle\right| \leq \Upsilon^{(t_1)}\cdot(1+\epsilon).\label{eq: proof noise memorization 4}
\end{align}
By combining \eqref{eq: proof noise memorization 3} and \eqref{eq: proof noise memorization 4}, we can arrive at 
\begin{align}
    \Upsilon^{(t+1)} = \max_{j\in\{\pm 1\},r\in[m]}\left\{\max_{i\in[n]}\left|\langle \mathbf{w}_{j,r}^{(t+1)},\boldsymbol{\xi}_{i}\rangle\right|, \max_{i\in\mathcal{W}}\left|\langle \mathbf{w}_{j,r}^{(t+1)},\widetilde{\boldsymbol{\xi}}_{i}\rangle\right|\right\} \leq \Upsilon^{(t_1)}\cdot(1+\epsilon).
\end{align}
Thus we have proved our induction statement for step $t+1$.
Repeating the induction completes the proof of Proposition \ref{prop: noise memorization}.
\end{proof}

\subsubsection{Proof of Proposition~\ref{prop: f tk-1 lower bound general ver}}\label{subsubsec: proof: f tk-1}

\begin{proof}[Proof of Proposition~\ref{prop: f tk-1 lower bound general ver}]\label{proof: f tk-1}
We expand $\ft_{t_j(\bar{S}_{j,k})}$ as follows.
\allowdisplaybreaks
\begin{align}
    \ft_{t_j(\bar{S}_{j,k})} &\le \frac{1}{m} \sum_{r\in [m]}\sigma\big(\dotp{\wb_{j,r}^{(t_j(\bar{S}_{j,k}-1))}}{j\ub}\big) \cdot \left(1+ \et \big( 1- \ft_{t_j(\bar{S}_{j,k}-1)}\big)\right)^2 \\
    &\qquad + \frac{1}{m}\sum_{r\in [m]}\sigma\big(\dotp{\wb_{j,r}^{(t_j(\bar{S}_{j,k}-1))}}{j\vb}\big) \cdot \left(1+ \alpha \et \big( 1- \ft_{t_j(\bar{S}_{j,k}-1)}\big)\right)^2 \\
    &\qquad - \frac{1}{m} \sum_{r\in [m]}\sigma\big(-\dotp{\wb_{-j,r}^{(t_j(\bar{S}_{j,k}-1))}}{-j\ub}\big) \cdot \left(1- \et \big( 1- \ft_{t_j(\bar{S}_{j,k}-1)}\big)\right)^2 \\
    &\qquad -\frac{1}{m}\sum_{r\in [m]}\sigma\big(-\dotp{\wb_{j,r}^{(t_j(\bar{S}_{j,k}-1))}}{-j\vb}\big) \cdot \left(1- \alpha \et \big( 1- \ft_{t_j(\bar{S}_{j,k}-1)}\big)\right)^2 \\ 
    &\qquad+ \left|\frac{1}{m} \sum_{r\in [m]} \sigma\big(\dotp{\wb_{j,r}^{(t_j(\bar{S}_{j,k}))}}{\bxi_{i_{t_j(\bar{S}_{j,k})}}} \big)- \frac{1}{m} \sum_{r\in [m]}\sigma\big(\dotp{\wb_{-j,r}^{(t_j(\bar{S}_{j,k}))}}{\bxi_{i_{t_j(\bar{S}_{j,k})}}}\big)\right| \\ 
    & \le  \frac{1}{m} \sum_{r\in [m]}\left(\sigma\big(\dotp{\wb_{j,r}^{(t_j(\bar{S}_{j,k}-1))}}{j\ub}\big)+\sigma\big(\dotp{\wb_{j,r}^{(t_j(\bar{S}_{j,k}-1))}}{j\vb}\big)\right. \\ 
    & \qquad\qquad \left.-\sigma\big(-\dotp{\wb_{-j,r}^{(t_j(\bar{S}_{j,k}-1))}}{-j\ub}\big)-\sigma\big(-\dotp{\wb_{j,r}^{(t_j(\bar{S}_{j,k}-1))}}{-j\vb}\big)\right) \cdot \left(1+ \et \big( 1- \ft_{t_j(\bar{S}_{j,k}-1)}\big)\right)^2 \\
    &\qquad +\frac{1}{m}\sum_{r\in [m]}\left(\sigma\big(-\dotp{\wb_{-j,r}^{(t_j(\bar{S}_{j,k}-1))}}{-j\ub}\big)+ \sigma\big(-\dotp{\wb_{j,r}^{(t_j(\bar{S}_{j,k}-1))}}{-j\vb}\big)\right)\cdot \left(2\et \big( 1- \ft_{t_j(\bar{S}_{j,k}-1)}\big)\right)\\
    &\qquad\qquad+ {\Upsilon^{(t_j(\bar{S}_{j,k}))} }^2\\ 
    & \le \ft_{t_j(\bar{S}_{j,k}-1)} \cdot  \left(1+\et \big( 1- \ft_{t_j(\bar{S}_{j,k}-1)}\big)\right)^2 + 2E_{t_j(\bar{S}_{j,k}-1)}+{\Upsilon^{(t_j(\bar{S}_{j,k}))}}^2. \label{eq: ftsk expand}
\end{align}
By Assumption~\ref{ass:oscilation_one_data} we know that $\ft_{t_j(\bar{S}_{j,k})}>1+\delta$. 
From the discussion in the proof of Lemma~\ref{prop: f Tk-1 lower bound}, we know that once 
\begin{align}
    2E_{t_j(\bar{S}_{j,k}-1)}+{\Upsilon^{(t_j(\bar{S}_{j,k}))}}^2 \le \delta,
\end{align}
we have that 
\begin{align}
    \ft_{t_j(\bar{S}_{j,k}-1)} \ge \frac{\et +2 - \sqrt{\et^2+4\et}}{2\et},
\end{align}
and that
\begin{align}
    \ft_{t_j(\bar{S}_{j,k})} &\le \left(1+\et(1-\ft_{t_j(\bar{S}_{j,k}-1)})\right)^2 + E_{t_j(\bar{S}_{j,k}-1)} +{\Upsilon^{(t_j(\bar{S}_{j,k}))}}^2  \\
    &\le \bigl(\et/2+\sqrt{\et^2/4+\et}\bigr)^2 + E_{t_j(\bar{S}_{j,k}-1)} +{\Upsilon^{(t_j(\bar{S}_{j,k}))}}^2 .
\end{align}
This finishes the proof of Proposition~\ref{prop: f tk-1 lower bound general ver}
\end{proof}

\section{Multiple Training Data Case: Small Learning Rate Regime}\label{sec:multiple-data small lr}

This section focuses on the multiple training data setup with a small learning rate.

Recall that $\mathcal{W}$ is the index set of training data points which lack the strong feature patch.
By \eqref{eq: two way multiple data sgd}, the CNN weights are updated according to 
\begin{align}
    \mathbf{w}_{j,r}^{(t+1)} &= \mathbf{w}_{j,r}^{(t)} - \frac{j\eta}{m}\cdot\big(f(\mathbf{x}_{i_t};\mathbf{W}^{(t)}) - y_{i_t}\big)\cdot\left(\sigma^\prime(\langle\mathbf{w}_{j,r}^{(t)},y_{i_t}\mathbf{u}\rangle)\cdot y_{i_t}\mathbf{u}\cdot\mathbf{1}\{i_t\notin\mathcal{W}\}\right. \\
    &\qquad + \sigma^\prime(\langle\mathbf{w}_{j,r}^{(t)},y_{i_t}\mathbf{v}\rangle)\cdot y_{i_t}\mathbf{v} + \sigma^{\prime}(\langle\mathbf{w}_{j,r}^{(t)}\cdot \boldsymbol{\xi}_{i_t}\rangle)\cdot\boldsymbol{\xi}_{i_t}\\
    &\qquad + \left.\sigma^{\prime}(\langle\mathbf{w}_{j,r}^{(t)}\cdot \widetilde{\boldsymbol{\xi}}_{i_t}\rangle)\cdot\widetilde{\boldsymbol{\xi}}_{i_t},\cdot\mathbf{1}\{i_t\in\mathcal{W}\}\right).\label{eq: small lr sgd}
\end{align}
Subsequently, by \eqref{eq: small lr sgd} the update formulas of those inner products of interests are given by
\begin{align}
    \langle \mathbf{w}_{j,r}^{(t+1)},j\mathbf{u}\rangle &= \langle \mathbf{w}_{j,r}^{(t)},j\mathbf{u}\rangle + \frac{\eta\|\mathbf{u}\|_2^2}{m}\cdot \big(1 - y_{i_t} f(\mathbf{x}_{i_t};\mathbf{W}^{(t)})\big)\cdot\sigma^{\prime}(\langle \mathbf{w}_{j,r}^{(t)},j\mathbf{u}\rangle jy_{i_t})\cdot\mathbf{1}\{i_t\notin\mathcal{W}\}, \\ 
    \langle \mathbf{w}_{j,r}^{(t+1)},j\mathbf{v}\rangle &= \langle \mathbf{w}_{j,r}^{(t)},j\mathbf{v}\rangle + \frac{\eta\|\mathbf{v}\|_2^2}{m}\cdot \big(1 - y_{i_t} f(\mathbf{x}_{i_t};\mathbf{W}^{(t)})\big)\cdot\sigma^{\prime}(\langle\mathbf{w}_{j,r}^{(t)}\cdot j\mathbf{v}\rangle jy_{i_t}),
\end{align}
and 
\begin{align}
    \langle \mathbf{w}_{j,r}^{(t+1)},\boldsymbol{\xi}_{i}\rangle &= \langle \mathbf{w}_{j,r}^{(t)},\boldsymbol{\xi}_i\rangle + \frac{\eta\cdot jy_{i_t}}{m}\cdot \big(1 - y_{i_t} f(\mathbf{x}_{i_t};\mathbf{W}^{(t)})\big)\cdot\left(\sigma^{\prime}(\langle\mathbf{w}_{j,r}^{(t)}\cdot \boldsymbol{\xi}_{i_t}\rangle)\cdot\langle\boldsymbol{\xi}_{i_t},\boldsymbol{\xi}_i\rangle\right.\\
    & \qquad +\left.\sigma^{\prime}(\langle\mathbf{w}_{j,r}^{(t)}\cdot \widetilde{\boldsymbol{\xi}}_{i_t}\rangle)\cdot\langle\widetilde{\boldsymbol{\xi}}_{i_t},\boldsymbol{\xi}_i\rangle\cdot\mathbf{1}\{i_t\in\mathcal{W}\}\right). \\
    \langle \mathbf{w}_{j,r}^{(t+1)},\widetilde{\boldsymbol{\xi}}_{i}\rangle &= \langle \mathbf{w}_{j,r}^{(t)},\widetilde{\boldsymbol{\xi}}_i\rangle + \frac{\eta\cdot jy_{i_t}}{m}\cdot \big(1 - y_{i_t} f(\mathbf{x}_{i_t};\mathbf{W}^{(t)})\big)\cdot\left(\sigma^{\prime}(\langle\mathbf{w}_{j,r}^{(t)}, \boldsymbol{\xi}_{i_t}\rangle)\cdot\langle\boldsymbol{\xi}_{i_t},\widetilde{\boldsymbol{\xi}}_i\rangle\right.\\
    & \qquad +\left.\sigma^{\prime}(\langle\mathbf{w}_{j,r}^{(t)}, \widetilde{\boldsymbol{\xi}}_{i_t}\rangle)\cdot\langle\widetilde{\boldsymbol{\xi}}_{i_t},\widetilde{\boldsymbol{\xi}}_i\rangle\cdot\mathbf{1}\{i_t\in\mathcal{W}\}\right),\quad i\in\mathcal{W}.
\end{align}
For convenience, we also write $\ell^{(t)}_i=f(\xb_i;\Wb^{(t)})-y_i$ as the fitting residual. 

The result of this section relies on the following conditions on the data model and the initialization.

\begin{assumption}[Conditions on hyperparameters]\label{cond:model_params_multiple_data small-lr}
    Suppose that the following holds. 
    For some $\epsilon\in(0,1),$
    \begin{enumerate}
        \item The weight initialization scale $\sigma_0 = \widetilde{\Theta}(\|\mathbf{u}\|_2^{-1})$;
        \item Strong signal strength $\norm{\ub}_2> \widetilde{\Omega}(m/\sqrt{n\epsilon})\cdot \sigma_p\sqrt{d}$ and weak signal strength $\sigma_p\sqrt{d}\ge \widetilde{\Omega}\left(m/\sqrt{n\epsilon}\right)\cdot\norm{\vb}_2$;
        \item The dimension $d$ satisfies $d = \Omega(\mathrm{polylog}(m,n))$.
    \end{enumerate}
\end{assumption}

\begin{theorem}[Restatement of Proposition~\ref{prop:multiple-data small-lr}]\label{thm:multiple-data small-lr formal}
    Under Assumption~\ref{cond:model_params_multiple_data small-lr}, choosing the learning rate $\eta\le m/6\|\ub\|^2_2$ small enough and $\epsilon^\prime\in(0,1)$, then with high probability $1-1/\mathrm{poly}(d)$, there exist
    \begin{align*}
        T^\dagger=\frac{4m}{\eta(1-\tau)(1-\rho)\|\mathbf{u}\|_2^2}\log\left(\frac{2\iota}{\sigma_0\|\mathbf{u}\|_2}\right),\quad T=T^\dagger+\left\lfloor\frac{Cm^3}{2\eta\epsilon\norm{\ub}_2^2}\right\rfloor,\quad T^\prime =T+\left\lfloor\frac{\left\|\Wb^{(T)}-\Wb^\star\right\|_F^2}{2\eta\epsilon}\right\rfloor,
    \end{align*}
    with $\tau$, $\iota$ defined in \eqref{eq: tau multi}, \eqref{eq: iota multi} such that: (i) the average loss on samples $\mathcal{W}^c$ decreases to $3\epsilon$ over iterations $[T^\dagger,T]$, i.e. 
    \begin{equation*}
         \frac{1}{2n}\sum_{i\in\mathcal{W}^c}\min_{T^\dagger\le t\le T}\left\{y_i-f(\xb_i;\Wb^{(t)})\right\}^2\le 3\epsilon,
    \end{equation*}
    (ii) average loss on samples $\mathcal{W}$ decreases to $3\epsilon^\prime$ over iterations $[T,T^\prime]$, i.e.
    \begin{equation*}
        \frac{1}{2n}\sum_{i\in\mathcal{W}}\min_{T\le t\le T^\prime}\left(y_i-f(\xb_i;\Wb^{(t)})\right)^2\le 3\epsilon^\prime,
    \end{equation*}
    (iii) the model does not learn weak signal $\vb$ well enough even until $T^\prime$, compared to initialization, i.e. 
    \begin{align}
        \max_{j,r} \left|\langle\mathbf{w}_{j\in \{\pm 1\},r\in[m]}^{(t)},\mathbf{v}\rangle\right|\le 2\sqrt{2\log(16m/p)}\cdot\sigma_0\|\mathbf{v}\|_2,\quad \forall t\le T^\prime.
    \end{align}
\end{theorem}

In the multiple training data small learning rate regime, the dynamics go through three stages, which we characterize in Appendices~\ref{subsec: stage 1}, \ref{subsec: stage 2}, and \ref{subsec: stage 3}, respectively. 
The following lemma plays an important role in the exponentially increasing stage (stage 1), for which we single it out here.

\begin{lemma}[Derivative lower bound]\label{lemma:multiple-data, lower bound fitting residual}
    For any $0<\tau<1$ to be tuned later, suppose at some time $t$ there holds
    \begin{equation*}
        \max_{j\in\{\pm 1\},r\in[m]}\left\{\left|\langle\mathbf{w}_{j,r}^{(t)},\mathbf{u}\rangle\right|, \left|\langle\mathbf{w}_{j,r}^{(t)},\mathbf{v}\rangle\right|, \max_{i\in[n]}\left|\langle\mathbf{w}_{j,r}^{(t)},\boldsymbol{\xi}_i\rangle\right|, \max_{i\in\mathcal{W}}\left|\langle\mathbf{w}_{j,r}^{(t)},\widetilde{\boldsymbol{\xi}}_i\rangle\right|\right\}\le\sqrt{\frac{\tau}{3}},
    \end{equation*}
    then we can lower bound the fitting residual $-y\ell^{(t)}_i\ge 1-\tau$ for every $i\in[n]$.
\end{lemma}
\begin{proof}[Proof of Lemma~\ref{lemma:multiple-data, lower bound fitting residual}]
    Plug into the CNN model definition \eqref{eq: cnn}, we have that
    \begin{align*}
    -y\ell^{(t)}_i=1-F_y(\mathbf{x}_i;\mathbf{W}^{(t)})+F_{-y}(\mathbf{x}_i;\mathbf{W}^{(t)})\ge 1-F_y(\mathbf{x}_i;\mathbf{W}^{(t)}).
    \end{align*}
    If $i\in\mathcal{W}^c$, we can upper bound $F_y(\mathbf{x}_i;\mathbf{W}^{(t)})$ further by
    \begin{align*}
    F_y(\mathbf{x}_i;\mathbf{W}^{(t)})&=\frac{1}{m}\sum_{r\in[m]}\sigma(\langle\mathbf{w}_{y,r}^{(t)},y\mathbf{u}\rangle)+\sigma(\langle\mathbf{w}_{y,r}^{(t)},y\mathbf{v}\rangle)+\sigma(\langle\mathbf{w}_{y,r}^{(t)},\boldsymbol{\xi}_i\rangle)\\
    &\le\max_{r\in[m]}\left\{\langle\mathbf{w}_{y,r}^{(t)},y\mathbf{u}\rangle^2+\langle\mathbf{w}_{y,r}^{(t)},y\mathbf{v}\rangle^2+\langle\mathbf{w}_{y,r}^{(t)},\boldsymbol{\xi}_i\rangle^2\right\}\le \tau.
    \end{align*}
    Otherwise, if $i\in\mathcal{W}$, we also have
    \begin{align*}
    F_y(\mathbf{x}_i;\mathbf{W}^{(t)})=&\frac{1}{m}\sum_{r\in[m]}\sigma(\langle\mathbf{w}_{y,r}^{(t)},\widetilde{\boldsymbol{\xi}}_i\rangle)+\sigma(\langle\mathbf{w}_{y,r}^{(t)},y\mathbf{v}\rangle)+\sigma(\langle\mathbf{w}_{y,r}^{(t)},\boldsymbol{\xi}_i\rangle)\\
    \le&\max_{r\in[m]}\left\{\langle\mathbf{w}_{y,r}^{(t)},\widetilde{\boldsymbol{\xi}}_i\rangle^2+\langle\mathbf{w}_{y,r}^{(t)},y\mathbf{v}\rangle^2+\langle\mathbf{w}_{y,r}^{(t)},\boldsymbol{\xi}_i\rangle^2\right\}\le \tau.
    \end{align*}
    Then it follows that $-y\ell^{(t)}\ge 1-\tau$. 
\end{proof}

\subsection{Stage 1. Learning Strong Signal Exponentially Fast}\label{subsec: stage 1}
In this stage, we mainly track the maximal inner product between $\mathbf{w}$ and the signal vectors $\mathbf{v}$ and $\mathbf{u}$, with extra attention to the maximal inner product between $\mathbf{w}$ and the noise vectors.
\begin{align*}
    \Psi^{(t)}&=\max_{j\in\{\pm 1\},r\in[m]} \left|\langle\mathbf{w}_{j,r}^{(t)},\mathbf{v}\rangle\right|,\quad \Phi^{(t)}=\max_{j\in\{\pm 1\},r\in[m]} \left|\langle\mathbf{w}_{j,r}^{(t)},\mathbf{u}\rangle\right|,\\
    \Gamma^{(t)}_i&=\max_{j\in\{\pm 1\},r\in[m]} \left|\langle\mathbf{w}_{j,r}^{(t)},\boldsymbol{\xi}_i\rangle\right|,\quad i\in[n],\\
    \widetilde{\Gamma}^{(t)}_i&=\max_{j\in\{\pm 1\},r\in[m]} \left|\langle\mathbf{w}_{j,r}^{(t)},\widetilde{\boldsymbol{\xi}}_i\rangle\right|,\quad i\in\mathcal{W}.
\end{align*}
\begin{lemma}[First stage: noise]\label{lemma:multiple-data small-lr upper bound Gamma}
Under the same conditions as Theorem~\ref{thm:multiple-data small-lr formal}, ever since initialization, at least until time 
\begin{align}
    T_+:=\frac{nm}{3\eta(4+\rho)\sigma_p^2d},
\end{align}
there still holds that 
    \begin{equation}\label{eq:multiple-data small-lr upper bound Gamma}
        \max_{i\in[n]}\Gamma_i^{(t)}\le \sigma_0\sigma_p\sqrt{d},\quad \max_{i\in\mathcal{W}}\widetilde{\Gamma}_i^{(t)}\le \sigma_0\sigma_p\sqrt{d}. 
    \end{equation}
\end{lemma}
\begin{proof}[Proof of Lemma~\ref{lemma:multiple-data small-lr upper bound Gamma}]
    For those inner products with noise vectors, $\forall i\in[n]$, the updating rules become
\begin{align}
    \left|\langle \mathbf{w}_{j,r}^{(t+1)},\boldsymbol{\xi}_{i}\rangle\right| &\le \left|\langle \mathbf{w}_{j,r}^{(t)},\boldsymbol{\xi}_i\rangle\right| + \frac{\eta}{m}\cdot \left|\ell^{(t)}_{i_t}\right|\cdot\left(\sigma^{\prime}(\langle\mathbf{w}_{j,r}^{(t)}\cdot \boldsymbol{\xi}_{i_t}\rangle)\cdot\left|\langle\boldsymbol{\xi}_{i_t},\boldsymbol{\xi}_i\rangle\right| +\sigma^{\prime}(\langle\mathbf{w}_{j,r}^{(t)}\cdot \widetilde{\boldsymbol{\xi}}_{i_t}\rangle)\cdot\left|\langle\widetilde{\boldsymbol{\xi}}_{i_t},\boldsymbol{\xi}_i\rangle\right|\cdot\mathbf{1}\{i_t\in\mathcal{W}\}\right)\\
    &\le \left|\langle \mathbf{w}_{j,r}^{(t)},\boldsymbol{\xi}_i\rangle\right| + \frac{6\eta}{m}\cdot \left(\left|\langle\mathbf{w}_{j,r}^{(t)}, \boldsymbol{\xi}_{i_t}\rangle\right|\cdot\left|\langle\boldsymbol{\xi}_{i_t},\boldsymbol{\xi}_i\rangle\right|+\left|\langle\mathbf{w}_{j,r}^{(t)}, \widetilde{\boldsymbol{\xi}}_{i_t}\rangle\right|\cdot\left|\langle\widetilde{\boldsymbol{\xi}}_{i_t},\boldsymbol{\xi}_i\rangle\right|\cdot\mathbf{1}\{i_t\in\mathcal{W}\}\right).
\end{align}
By taking maximum over $r\in[m]$, we conclude that
\begin{equation*}
    \Gamma^{(t+1)}_i\le \Gamma^{(t)}_i+\frac{6\eta}{m}\cdot \left(\Gamma_{i_t}^{(t)}\cdot \left|\langle\boldsymbol{\xi}_{i_t},\boldsymbol{\xi}_i\rangle\right|+\widetilde{\Gamma}_{i_t}^{(t)}\cdot \left|\langle\widetilde{\boldsymbol{\xi}}_{i_t},\boldsymbol{\xi}_i\rangle\right|\cdot\mathbf{1}\{i_t\in\mathcal{W}\}\right),\quad\forall i\in[n].
\end{equation*}
Similarly, we also have that 
\begin{equation*}
    \widetilde{\Gamma}^{(t+1)}_i\le \widetilde{\Gamma}^{(t)}_i+\frac{6\eta}{m}\cdot \left(\Gamma_{i_t}^{(t)}\cdot \left|\langle\boldsymbol{\xi}_{i_t},\widetilde{\boldsymbol{\xi}}_i\rangle\right|+\widetilde{\Gamma}_{i_t}^{(t)}\cdot \left|\langle\widetilde{\boldsymbol{\xi}}_{i_t},\widetilde{\boldsymbol{\xi}}_i\rangle\right|\cdot\mathbf{1}\{i_t\in\mathcal{W}\}\right),\quad\forall i\in\mathcal{W}.
\end{equation*}
We then use induction to rigorously prove our conclusions.
Firstly, \eqref{eq:multiple-data small-lr upper bound Gamma} holds at time $t=0$.
Now suppose that \eqref{eq:multiple-data small-lr upper bound Gamma} holds until some $\widetilde{T}< T_+$. Fixing some $i\in[n]$,
\begin{align*}
    \Gamma^{(\widetilde{T}+1)}_i &\le \frac{6\eta \sigma_0\sigma_p\sqrt{d}}{m}\cdot \sum_{t=0}^{ \widetilde{T}}\left(\left|\langle\boldsymbol{\xi}_{i_t},\boldsymbol{\xi}_i\rangle\right|+\left|\langle\widetilde{\boldsymbol{\xi}}_{i_t},\boldsymbol{\xi}_i\rangle\right|\cdot\mathbf{1}\{i_t\in\mathcal{W}\}\right)\\
    &\le \frac{6\eta \sigma_0\sigma_p\sqrt{d}}{m}\cdot \left(\frac{3\widetilde{T}\sigma_p^2d}{2n}+2\widetilde{T}\cdot (1+\rho)\cdot \sigma_p^2\sqrt{d\log(4n^2/p)}\right)\\
    &\le \frac{3\eta \sigma_0\sigma_p\sqrt{d}(4+\rho)\widetilde{T}\sigma_p^2d}{nm}\\
    &\le \sigma_0\sigma_p\sqrt{d}.
\end{align*}
The first inequality is by induction hypothesis. The second inequality is because that there are at most $\widetilde{T}/n$ many $i_t$'s would equal $i$ and at most $\rho\widetilde{T}$ many $i_t$'s would be in $\mathcal{W}$, and we also use Lemma~\ref{lem: noise norm and correlation} to control the correlations between noise vectors. The third inequality is by $d\ge 16n^2\log(4n^2/p)$ for Assumption~\ref{cond:model_params_multiple_data small-lr} , while the last inequality is due to $\widetilde{T}<T_+$.
Similarly, we can also control $\widetilde{\Gamma}_i^{(t+1)}$ for some fixed $i\in\mathcal{W}$ as
\begin{align*}
    \widetilde{\Gamma}^{(\widetilde{T}+1)}_i &\le \frac{6\eta \sigma_0\sigma_p\sqrt{d}}{m}\cdot \sum_{t=0}^{ \widetilde{T}}\left(\left|\langle\boldsymbol{\xi}_{i_t},\widetilde{\boldsymbol{\xi}}_i\rangle\right|+\left|\langle\widetilde{\boldsymbol{\xi}}_{i_t},\widetilde{\boldsymbol{\xi}}_i\rangle\right|\cdot\mathbf{1}\{i_t\in\mathcal{W}\}\right)\\
    &\le \frac{6\eta \sigma_0\sigma_p\sqrt{d}}{m}\left(\frac{3\widetilde{T}\sigma_p^2d}{2n}+2\widetilde{T}\cdot (1+\rho)\cdot \sigma_p^2\sqrt{d\log(4n^2/p)}\right)\\
    &\le \sigma_0\sigma_p\sqrt{d},
\end{align*}
where the second inequality is because there are at most $\widetilde{T}/n$ many $i_t$'s would equal $i\in\mathcal{W}$. In conclusion, \eqref{eq:multiple-data small-lr upper bound Gamma} holds at least until $T_+$.
\end{proof}

In the following, we would take
\begin{align}
    \tau&=\max\left\{2\sigma_0\|\mathbf{u}\|_2\left(2\log(16m/p)\right)^{1/2-\|\mathbf{v}\|_2^2/4(\sigma_p^2d)},1-\frac{6\sqrt{2}\sigma_p^2d}{\|\mathbf{u}\|_2^2\log(2/\sqrt{2\log(16m/p)})}\right\},\label{eq: tau multi}\\
    \iota&=2\sigma_0\|\mathbf{u}\|_2\cdot\exp\left\{\frac{(1-\tau)(1-\rho)\norm{\ub}_2^2}{3(4+\rho)\sigma_p^2d}\right\}.\label{eq: iota multi}
\end{align}
By the conditions in Assumption~\ref{cond:model_params_multiple_data small-lr} on $\|\mathbf{v}\|_2^2/\|\mathbf{u}\|_2^2$, we find $\tau,\iota$ both constant in $(0,1)$.
\begin{lemma}[First stage: signal]\label{lemma:multiple-data small-lr 1st-stage}
Under the same conditions as Theorem~\ref{thm:multiple-data small-lr formal}, there exists time
\begin{equation*}
    T^\dagger=\frac{4m}{\eta(1-\tau)(1-\rho)\|\mathbf{u}\|_2^2}\log\left(\frac{2\iota}{\sigma_0\|\mathbf{u}\|_2}\right),
\end{equation*}
such that: (i) the model learns strong signal to a constant level,
\begin{equation*}
    \max_{r\in[m]} \langle\mathbf{w}_{j,r}^{(T^\dagger)},j\mathbf{u}\rangle\ge\iota,\quad\forall j\in\{\pm 1\},   
\end{equation*}
(ii) compared to the random initialization, the model does not learn weak signal that much, i.e., 
\begin{align}
    \max_{j\in\{\pm 1\},r\in[m]} \left|\langle\mathbf{w}_{j,r}^{(T^\dagger)},\mathbf{v}\rangle\right|\le 2\sqrt{2\log(16m/p)}\cdot\sigma_0\|\mathbf{v}\|_2
\end{align}
\end{lemma}
\begin{proof}[Proof of Lemma~\ref{lemma:multiple-data small-lr 1st-stage}]
Firstly, we would find $\{\Psi^{(t)},\Phi^{(t)}\}_{t\geq 0}$ having an exponentially growing upper bound. Recursively, we would have that 
\begin{align}
    \Psi^{(t+1)}&\le \Psi^{(t)}+\max_{j\in\{\pm 1\},r\in[m]}\left|\frac{jy\eta}{m}\cdot \big(f(\mathbf{x}_i;\mathbf{W}^{(t)}) - y\big)\cdot\sigma^{\prime}(\langle\mathbf{w}_{j,r}^{(t)},y\mathbf{v}\rangle)\cdot\|\mathbf{v}\|_2^2\right|\\
    & = \Psi^{(t)}+\frac{\eta}{m}\cdot \left|\ell^{(t)}_i\right|\cdot \|\mathbf{v}\|_2^2\cdot\max_{j\in\{\pm 1\},r\in[m]}\sigma^{\prime}(\langle\mathbf{w}_{j,r}^{(t)},y\mathbf{v}\rangle)\\
    & \le \Psi^{(t)}+\frac{2\eta}{m}\cdot \left|\ell^{(t)}_i\right|\cdot \|\mathbf{v}\|_2^2\cdot \Psi^{(t)}\\
    &\le\exp\left(\frac{6\eta\|\mathbf{v}\|_2^2}{m}\right)\cdot \Psi^{(t)}.
\end{align}
Therefore, we have that 
\begin{align}
    \Psi^{(t)}\le\exp\left(\frac{6\eta\|\mathbf{v}\|_2^2t}{m}\right)\cdot\Psi^{(0)}\le \exp\left(\frac{6\eta\|\mathbf{v}\|_2^2t}{m}\right)\cdot\sqrt{2\log(16m/p)}\cdot  \sigma_0\|\mathbf{v}\|_2   \label{eq: psi exp bound multi} 
\end{align}
It follows similarly that
\begin{align}
    \Phi^{(t)}\le\exp\left(\frac{6\eta(1-\rho)\|\mathbf{u}\|_2^2t}{m}\right)\cdot\Phi^{(0)}\le \exp\left(\frac{6\eta(1-\rho)\|\mathbf{u}\|_2^2t}{m}\right)\cdot\sqrt{2\log(16m/p)} \cdot\sigma_0\|\mathbf{u}\|_2.\label{eq: phi exp bound multi} 
\end{align}
The extra factor $1-\rho$ appears because only a $1-\rho$ proportion of data points would contain $\mathbf{u}$, and therefore contribute to evolution of $\Phi^{(t)}$. Note that growing rates of these two bounds differ a lot due to the different magnitudes of $(1-\rho)\|\mathbf{u}\|_2$ and $\|\mathbf{v}\|_2$.

Our subsequent analysis illustrates that $\Phi^{(t)}$ can grow into a constant-level magnitude since strong signal $\mathbf{u}$ is significant enough. We can track how well our model learns $\mathbf{u}$ by
\begin{equation*}
    A^{(t)}_1=\max_{r\in[m], i_t\notin\mathcal{W}} \langle\mathbf{w}_{1,r}^{(t)},\mathbf{u}\rangle,\quad A^{(t)}_{-1}=\max_{r\in[m],i_t\notin\mathcal{W}} \langle\mathbf{w}_{-1,r}^{(t)},-\mathbf{u}\rangle.
\end{equation*}
By definition, $A^{(t)}_1,A^{(t)}_{-1}\le \Phi^{(t)}$ also admits an exponentially upper bound in \eqref{eq: phi exp bound multi} . 
For a certain $\tau\in(0,1)$, due to the exponential upper bounds~\eqref{eq: psi exp bound multi} and \eqref{eq: phi exp bound multi}, $\max\{\Phi^{(t)},\Psi^{(t)}\}\le \sqrt{\tau/3}$ is true at least until
\begin{equation*}
    T_1=\frac{m}{6\eta(1-\rho)\|\mathbf{u}\|_2^2}\log\left(\frac{\sqrt{\tau/2}}{\sigma_0\|\mathbf{u}\|_2\sqrt{2\log(16m/p)}}\right).
\end{equation*}
Moreover, since we have $(1-\rho)\|\ub\|_2^2\gg\sigma_p^2d/n$ by Assumption~\ref{cond:model_params_multiple_data small-lr}, we also know $T_1\le T_+$ where $T_+$ comes from Lemma~\ref{lemma:multiple-data small-lr upper bound Gamma}.
Therefore 
\begin{align}
    \Gamma_i^{(t)}\le \sigma_0\sigma_p\sqrt{d}\le\sqrt{\tau/3},\quad \widetilde{\Gamma}_i^{(t)}\le \sigma_0\sigma_p\sqrt{d}\le\sqrt{\tau/3},\quad \forall t\leq T_1
\end{align}
Consequently, until at least time $T_1$, we can use Lemma~\ref{lemma:multiple-data, lower bound fitting residual} to conclude that $-y_{i_t}\ell^{(t)}_{i_t}\ge 1-\tau$, which enables lower bounding $A^{(t)}$ in the following. 

The $i_t$-th sample would be used to update parameters, according to our multi-pass SGD updates \eqref{eq: two way multiple data sgd}. 
If $i_t\in\mathcal{W}$, then $\langle\mathbf{w}_{j,r}^{(t+1)},j\mathbf{u}\rangle=\langle\mathbf{w}_{j,r}^{(t)},j\mathbf{u}\rangle$ holds for any $j\in\{\pm 1\}$ and $r\in[m]$. 
If $i_t\notin\mathcal{W}$ but $y_{i_t}=-1$, then $\max_{r}\langle\mathbf{w}_{1,r}^{(t+1)},\mathbf{u}\rangle=\max_{r}\langle\mathbf{w}_{1,r}^{(t)},\mathbf{u}\rangle\ge 0$ since that neuron will not be activated. Otherwise, only if $i_t\notin\mathcal{W}$ and $y_{i_t}=1$, the updating rule becomes
\begin{align*}
   \langle\mathbf{w}_{1,r}^{(t+1)},\mathbf{u}\rangle&=\langle\mathbf{w}_{1,r}^{(t)},\mathbf{u}\rangle+\frac{\eta}{m}\cdot \big(-y_{i_t}\ell^{(t)}_{i_t}\big)\cdot\sigma^{\prime}(\langle\mathbf{w}_{1,r}^{(t)},\mathbf{u}\rangle)\cdot \|\mathbf{u}\|_2^2\\
   &\ge\langle\mathbf{w}_{1,r}^{(t)},\mathbf{u}\rangle+\frac{2\eta(1-\tau)\|\mathbf{u}\|_2^2}{m}\cdot \max\left\{\langle\mathbf{w}_{1,r}^{(t)},\mathbf{u}\rangle,0\right\}.
\end{align*}
Take maximum over $r\in[m]$ to see that 
\begin{align*}
    A^{(t+1)}_1\ge A^{(t)}_1+\frac{2\eta(1-\tau)\|\mathbf{u}\|_2^2}{m}\cdot A^{(t)}_1\ge\exp\left(\frac{\eta(1-\tau)\|\mathbf{u}\|^2}{m}\right)\cdot A^{(t)}_1,
\end{align*}
where the last equality is by $1+z\ge\exp(z/2)$ for any $0\le z\le2$. Consequently, when $t$ is large (larger than $n$), we would have 
\begin{align}
A^{(t)}_1\ge& \exp\left(\frac{\eta(1-\tau)\|\mathbf{u}\|_2^2}{m}\cdot \sum_{t^\prime\le t}\mathbf{1}\{i_{t^\prime}\notin\mathcal{W},y_{i_t}=1\}\right)\cdot A^{(0)}_1\\
\ge& \exp\left(\frac{\eta(1-\tau)\|\mathbf{u}\|_2^2\cdot(1-\rho)\cdot t}{4m}\right)\cdot \sigma_0\|\mathbf{u}\|_2/2,\label{eq: exp lower bound multi 1}
\end{align}
at least until step $t\le T_1$. 
We use the fact that $\sum_{t^\prime\le t}\mathbf{1}\{i_{t^\prime}\notin\mathcal{W},y_{i_t}=1\}\ge (1-\rho)t/4$ because the sample labels are balanced (Lemma~\ref{lem: balanced sample}) and $1-\rho$ proportion of samples come with the strong signal. In the same manner, we would have that 
\begin{align}
    A^{(t)}_{-1}\ge \exp\left(\frac{\eta(1-\tau)\|\mathbf{u}\|_2^2(1-\rho) t}{4m}\right)\cdot \sigma_0\|\mathbf{u}\|_2/2.\label{eq: exp lower bound multi -1}
\end{align}
Define the time when $A^{(t)}_{\pm 1}$ both break $\iota$,
\begin{align}
    T_2=\frac{4m}{\eta(1-\tau)(1-\rho) \|\mathbf{u}\|_2^2}\log\left(\frac{2\iota}{\sigma_0\|\mathbf{u}\|_2}\right)\le T_1,
\end{align}
where the inequality is due to the scaling of $\iota$ upon $\tau$. 
Moreover, we also need that $T_2\le T_+$, where $T_+$ is the time that $\langle\wb_{j,r}^{(t)},\widetilde{\boldsymbol{\xi}}_i\rangle,\langle\wb_{j,r}^{(t)},\boldsymbol{\xi}_i\rangle$ remains in $\mathcal{O}(\sigma_0\sigma_p\sqrt{d})$. And this requirement is also achieved by the selection of $\tau$ and $\iota$.
Plugging $T_2$ into the exponential lower bounds~\eqref{eq: exp lower bound multi 1} and \eqref{eq: exp lower bound multi -1}, we can conclude that 
\begin{align}
    \Phi^{(T_2)}\ge A^{(T_2)}_{\pm 1}\ge\iota,
\end{align}
which already grows up to a constant level magnitude by the time $T_2$. Lastly, plug the definition of $T_2$ to upper bound $\psi^{(T_2)}$ as 
\begin{align}
    \Psi^{(T_2)}\le \exp\left(\frac{24\|\mathbf{v}\|_2^2}{(1-\tau)(1-\rho)\|\mathbf{u}\|_2^2}\log\left(\frac{2\iota}{\sigma_0\|\mathbf{u}\|_2}\right)\right)\cdot \sqrt{2\log(16m/p)} \cdot \sigma_0\|\mathbf{v}\|_2 \le 2\sqrt{2\log(16m/p)} \cdot \sigma_0\|\mathbf{v}\|_2.
\end{align}
In conclusion, by taking $T^\dagger=T_2$, this lemma is completely proved.
\end{proof}

\subsection{Stage 2. Exploiting Strong Signal}\label{subsec: stage 2}
In the second stage, our lemmas suggest that before the model really learns the weak signal $\mathbf{v}$ or memorizes any noise vector, the model already fits a proportion $1-\rho$ of the data points (i.e., strong data) by exploiting strong signal $\mathbf{u}$.
\begin{lemma}[Second stage]\label{lemma:multiple-data small-lr 2nd-stage}
    Under the same conditions as Theorem~\ref{thm:multiple-data small-lr formal}, there exists time
    \begin{equation*}
        T=T^\dagger+\left\lfloor\frac{Cm^3}{2\eta\epsilon\norm{\ub}_2^2}\right\rfloor
    \end{equation*}
    such that: (i) the average loss over iterations within this stage has decreased to $2\epsilon$, i.e., 
    \begin{equation*}
        \frac{1}{2n}\sum_{i\in\mathcal{W}^c}\min_{T^\dagger\le t\le T}\left\{y_i-f(\bx_i;\Wb^{(t)})\right\}^2\le 3\epsilon,
    \end{equation*}
    (ii) all through the training dynamics $0\le t\le T$, there holds that 
    \begin{align}
        \max_{j\in\{\pm 1\},r\in[m]} \left|\langle\mathbf{w}_{j,r}^{(t)},\mathbf{v}\rangle\right|\le 2\sqrt{2\log(16m/p)}\sigma_0\cdot \|\mathbf{v}\|_2,
    \end{align}
    (iii) all through the training dynamics $0\le t\le T$, there holds that 
    \begin{align}
        \max_{j\in\{\pm 1\},r\in[m],i\in[n]} \left|\langle\mathbf{w}_{j,r}^{(t)},\boldsymbol{\xi}_i\rangle\right|\le \sigma_0\sigma_p\sqrt{d},\quad\max_{j\in\{\pm 1\},r\in[m],i\in\mathcal{W}} \left|\langle\mathbf{w}_{j,r}^{(t)},\widetilde{\boldsymbol{\xi}}_i\rangle\right|\le\sigma_0\sigma_p\sqrt{d}.
    \end{align}
\end{lemma}

In studying the second stage, we firstly identify when the upper bound on $\langle\mathbf{w}_{j,r}^{(t)},\mathbf{v}\rangle$ breaks and find that the conclusions of Lemma~\ref{lemma:multiple-data small-lr 1st-stage} still holds before that time.
\begin{lemma}\label{lemma:multiple-data small-lr 2nd starts}
Under the same conditions as Theorem~\ref{thm:multiple-data small-lr formal}, take $\eta\le m / 6\|\ub\|_2^2$. There exists a time
\begin{align}
    T^\ddagger=\frac{m}{6\eta\|\mathbf{v}\|_2^2}\log\left(\frac{\sqrt{\tau/2}}{\sigma_0\|\mathbf{v}\|_2\sqrt{2\log(16m/p)}}\right)\ge T^\dagger
\end{align}
such that \eqref{eq:multiple-data small-lr upper bound Gamma} and
\begin{align}
    \max_{j\in\{\pm 1\},r\in[m]} \left|\langle\mathbf{w}_{j,r}^{(t)},\mathbf{v}\rangle\right|\le 2\sqrt{2\log(16m/p)}\sigma_0\|\mathbf{v}\|,\quad
    \max_{r\in[m]} \langle\mathbf{w}_{j,r}^{(t)},j\mathbf{u}\rangle\ge \iota/2,\quad\forall j\in\{\pm 1\},
\end{align}
hold for any $T^\dagger\le t\le T^\ddagger$.
\end{lemma}

\begin{proof}[Proof of Lemma~\ref{lemma:multiple-data small-lr 2nd starts}]
Firstly, we need to adopt the exponential upper bound derived in proving Lemma~\ref{lemma:multiple-data small-lr 1st-stage},
\begin{equation*}
    \Psi^{(t)}\le\exp\left(\frac{6\eta\|\mathbf{v}\|_2^2t}{m}\right)\cdot\Psi^{(0)}\le \exp\left(\frac{6\eta\|\mathbf{v}\|_2^2t}{m}\right)\cdot \sqrt{2\log(16m/p)}\cdot \sigma_0\|\mathbf{v}\|_2.
\end{equation*}
Then we naturally find that before $T^\ddagger$, it would always hold that
\begin{align}
    \max_{j\in\{\pm 1\},r\in[m]} \left|\langle\mathbf{w}_{j,r}^{(t)},\mathbf{v}\rangle\right|\le2\sqrt{2\log(16m/p)}\cdot\sigma_0\|\mathbf{v}\|_2.
\end{align}
Due to the conditions on $\norm{\ub}_2^2/\norm{\vb}_2^2$, $T^\ddagger$ is found to be much larger than $T^\dagger$.
Then we proceed to prove the other assertion by induction. 
At time $t=T^\dagger$, the lower bound $\max_{j,r} \langle\mathbf{w}_{j,r}^{(t)},j\mathbf{u}\rangle\ge \iota/2$ holds as a consequence of the previous lemma. 
Suppose it holds until time $t$. 
If $i_t\in\mathcal{W}$, then $\langle\mathbf{w}_{j,r}^{(t+1)},j\mathbf{u}\rangle=\langle\mathbf{w}_{j,r}^{(t)},j\mathbf{u}\rangle$ holds for any $j\in\{\pm 1\},r\in[m]$. 
If $i_t\notin\mathcal{W}$ but $y_{i_t}=-1$, then $\max_{r\in[m]}\langle\mathbf{w}_{1,r}^{(t+1)},\mathbf{u}\rangle=\max_{r\in[m]}\langle\mathbf{w}_{1,r}^{(t)},\mathbf{u}\rangle\ge 0$ since that neuron will not be activated. 
Otherwise, if $i_t\notin\mathcal{W}$ and $y_{i_t}=1$, consider the updating rule 
\begin{align}\langle\mathbf{w}_{1,r}^{(t+1)},\mathbf{u}\rangle=\langle\mathbf{w}_{1,r}^{(t)},\mathbf{u}\rangle+\frac{\eta}{m}\cdot \big(1-y_{i_t}f(\mathbf{x}_{i_t};\mathbf{W}^{(t)})\big)\cdot\sigma^{\prime}(\langle\mathbf{w}_{1,r}^{(t)},\mathbf{u}\rangle)\cdot \|\mathbf{u}\|_2^2,
\end{align}
from which we find $\max_{r\in[m],i_t\notin\mathcal{W}} \langle\mathbf{w}_{y_i,r}^{(t+1)},y_i\mathbf{u}\rangle\ge \max_{r\in[m],i_t\notin\mathcal{W}} \langle\mathbf{w}_{y_i,r}^{(t)},y_i\mathbf{u}\rangle$ must hold if $y_{i_t}f(\mathbf{x}_{i_t};\mathbf{W}^{(t)})\le 1$. Otherwise, once $y_{i_t}f(\mathbf{x}_{i_t};\mathbf{W}^{(t)})> 1$, it immediately follows that
\begin{align}
    1&<y_{i_t}f(\mathbf{x}_{i_t};\mathbf{W}^{(t)})=F_{y_{i_t}}(\mathbf{x}_{i_t};\mathbf{W}^{(t)})-F_{-{y_{i_t}}}(\mathbf{x}_{i_t};\mathbf{W}^{(t)})\\
    &\le F_{y_{i_t}}(\mathbf{x}_{i_t};\mathbf{W}^{(t)})=\frac{1}{m}\sum_{r\in[m]}\sigma(\langle\mathbf{w}_{1,r}^{(t)},y_{i_t}\mathbf{u}\rangle)+\sigma(\langle\mathbf{w}_{1,r}^{(t)},y_{i_t}\mathbf{v}\rangle)+\sigma(\langle\mathbf{w}_{1,r}^{(t)},\mathbf{\xi}_{i_t}\rangle)\\
    &\le \max_{r\in[m]} \langle\mathbf{w}_{1,r}^{(t)},\mathbf{u}\rangle^2 + \sigma_0^2\|\mathbf{v}\|^2+\sigma_0^2\sigma_p^2d.
\end{align}
Consequently, for the specific neuron $r^\ast=\argmax_{r\in[m]}\langle\mathbf{w}_{1,r}^{(t)},\mathbf{u}\rangle^2$, there holds
\begin{align}
    \langle\mathbf{w}_{1,r^\ast}^{(t+1)},\mathbf{u}\rangle&\ge \langle\mathbf{w}_{1,r^\ast}^{(t)},\mathbf{u}\rangle-\frac{3\eta}{m}\cdot \langle\mathbf{w}_{1,r^\ast}^{(t)},\mathbf{u}\rangle\cdot \|\mathbf{u}\|_2^2\\
    &\ge\left(1-8\log(16m/p)\cdot \sigma_0^2\|\mathbf{v}\|_2^2-\sigma_0^2\sigma_p^2d\right)\cdot\left(1-\frac{3\eta}{m}\cdot \|\mathbf{u}\|_2^2\right)\\
    &\ge\frac{\iota}{2},
\end{align}
where the last inequality is enabled by taking 
\begin{align}
    \eta\le \frac{m}{6\|\ub\|_2^2},\quad \sigma_0\le \sqrt{\frac{1-\iota}{8\log(16m/p)\cdot \|\vb\|_2^2+\sigma_p^2d}}.
\end{align}
Therefore, we find that $\max_{r\in[m]}\langle\mathbf{w}_{1,r}^{(t+1)},\mathbf{u}\rangle\ge\iota/2$ must hold no matter what $i_t$ is. In the same way, one can also obtain $\max_{r\in[m]}\langle\mathbf{w}_{-1,r}^{(t+1)},-\mathbf{u}\rangle\ge\iota/2$. 
By induction, the induction proof is complete.
\end{proof}

Our subsequently analysis confirms that even before $T^\ddagger$, the model can already fit those data points with strong signal by exploiting $\ub$.
For the given $0<\epsilon<1$, define a reference point $\mathbf{W}^\ast$ as
\begin{equation}\label{eq:multiple-data reference pt 2nd}
    \mathbf{w}_{j,r}^\ast = \frac{4m (1+\epsilon)}{\iota}\cdot\frac{j\mathbf{u}}{\|\mathbf{u}\|_2^2},\quad j\in\{\pm 1\},r\in[m].
\end{equation}
\begin{lemma}\label{lemma:multiple-data inner-prod grad reference pt}
    Under the same condition as the previous lemma, for all $T^\dagger\le t\le T^\ddagger$, there holds 
    \begin{align}
        y_i\langle\nabla f(\mathbf{x}_i;\mathbf{W}^{(t)}),\mathbf{W}^\ast\rangle\ge2\cdot(1+\epsilon)    
    \end{align}
    for any $i\notin\mathcal{W}$.
\end{lemma}
\begin{proof}[Proof of Lemma~\ref{lemma:multiple-data inner-prod grad reference pt}]
    Recall that the definition of CNN in \eqref{eq: cnn}
    and that $\mathbf{u}\perp\mathrm{span}(\mathbf{v},\boldsymbol{\xi}_i)$. 
    We have that
    \begin{align*}
        y_i\langle\nabla f(\mathbf{x}_i;\mathbf{W}^{(t)}),\mathbf{W}^\ast\rangle&=\frac{1}{m}\sum_{j\in\{\pm 1\},r\in[m]}\sigma^\prime(\langle \mathbf{w}_{j,r}^{(t)},y_i\mathbf{u}\rangle)\cdot \langle \mathbf{w}_{j,r}^{\ast},y_i\mathbf{u}\rangle\\
        &=\sum_{j\in\{\pm 1\},r\in[m]}\sigma^\prime(\langle \mathbf{w}_{j,r}^{(t)},y_i\mathbf{u}\rangle)\cdot\frac{4(1+\epsilon)}{\iota}\\
        &\ge \max_{r\in[m]} \langle\mathbf{w}_{y_i,r}^{(t)},y_i\mathbf{u}\rangle\cdot \frac{4(1+\epsilon)}{\iota}\\
        &\ge 2\cdot (1+\epsilon),
    \end{align*}
    where the last inequality is by $\max_{r} \langle\mathbf{w}_{j,r}^{(t)},j\mathbf{u}\rangle\ge \iota/2$ for any $j\in\{\pm 1\}$ as shown by the previous lemma.
\end{proof}

\begin{lemma}\label{lemma:multiple-data small-lr descent lemma}
    Continued from the previous setting, for $T^\dagger\le t\le T^\ddagger$, if $i_t\notin\mathcal{W}$, there holds
    \begin{equation*}
        \|\mathbf{W}^{(t)}-\mathbf{W}^\ast\|_F^2-\|\mathbf{W}^{(t+1)}-\mathbf{W}^\ast\|_F^2\ge 2\eta \big(f(\mathbf{x}_{i_t};\mathbf{W}^{(t)})-y_{i_t}\big)^2-2\eta\epsilon^2.
    \end{equation*}
\end{lemma}
\begin{proof}[Proof of Lemma~\ref{lemma:multiple-data small-lr descent lemma}]
    Firstly we expand the difference by
    \begin{align}
        &\|\mathbf{W}^{(t)}-\mathbf{W}^\ast\|_F^2-\|\mathbf{W}^{(t+1)}-\mathbf{W}^\ast\|_F^2\notag\\
        &\qquad =2\eta\langle \ell^{(t)}_{i_t}\nabla f(\xb_{i_t};\mathbf{W}^{(t)}),\mathbf{W}^{(t)}-\mathbf{W}^\ast\rangle-\eta^2\cdot\left|\ell^{(t)}_{i_t}\right|^2\cdot \|\nabla f(\xb_{i_t};\mathbf{W}^{(t)})\|_F^2.\label{eq:multiple-data para-L2-distance difference expansion}
    \end{align}
    Since the neural network $f(\mathbf{x};\Wb)$ is $2$-homogeneous in $\mathbf{W}$ due to  the activation function $\sigma(z)=\max\{z,0\}^2$, we have that
    \begin{equation*}
        \langle\nabla f(\mathbf{x};\mathbf{W}^{(t)}),\mathbf{W}^{(t)}\rangle=2f(\mathbf{x};\mathbf{W}^{(t)}).
    \end{equation*}
    Stack these observations into the first term of previous difference expansion to obtain
    \begin{align*}
        \langle \ell^{(t)}_{i_t}\nabla f(\xb_{i_t};\mathbf{W}^{(t)}),\mathbf{W}^{(t)}-\mathbf{W}^\ast\rangle&=\ell^{(t)}_{i_t}\cdot \big(2f(\mathbf{x}_{i_t};\mathbf{W}^{(t)})-\langle\nabla f(\mathbf{x}_{i_t};\mathbf{W}^{(t)}),\mathbf{W}^\ast\rangle\big)\\
        &=2\ell^{(t)}_{i_t}\cdot\big(f(\mathbf{x}_{i_t};\mathbf{W}^{(t)})-y_{i_t}\big)+\ell^{(t)}_{i_t}\cdot y_{i_t}\cdot \big(2-y_{i_t}\langle\nabla f(\mathbf{x}_{i_t};\mathbf{W}^{(t)}),\mathbf{W}^\ast\rangle\big).
    \end{align*}
    Note that the first term is exactly $2(f(\mathbf{x}_{i_t};\mathbf{W}^{(t)})-y_{i_t})^2$. As for the second term, since $i_t\notin\mathcal{W}$, we need to plug in Lemma~\ref{lemma:multiple-data inner-prod grad reference pt} to see $2-y_{i_t}\langle\nabla f(\mathbf{x}_{i_t};\mathbf{W}^{(t)}),\mathbf{W}^\ast\rangle\le-2\epsilon<0$, so that
    \begin{align*}
        \left|\ell^{(t)}_{i_t}\cdot y_{i_t}\cdot \big(2-y_{i_t}\langle\nabla f(\mathbf{x}_{i_t};\mathbf{W}^{(t)}),\mathbf{W}^\ast\rangle\big)\right|\leq\frac{1}{2}\big(\ell^{(t)}_{i_t}\big)^2+2\epsilon^2.
    \end{align*}
    As a result, we would know that 
    \begin{align}
        \langle \ell^{(t)}_{i_t}\nabla f(\xb_{i_t};\mathbf{W}^{(t)}),\mathbf{W}^{(t)}-\mathbf{W}^\ast\rangle\ge \frac{3}{2}(f(\mathbf{x}_{i_t};\mathbf{W}^{(t)})-y_{i_t})^2-2\epsilon^2.
    \end{align}
    Next, an upper bound on the second order term $\eta^2\|\nabla L(\mathbf{W}^{(t})\|_F^2$ is given by
    \begin{align*}
       &\eta^2\cdot \left|\ell^{(t)}_{i_t}\right|^2\cdot \|\nabla f(\xb_{i_t};\mathbf{W}^{(t)})\|_F^2\\
       &\qquad =\eta^2\cdot \ell^{(t)2}_{i_t}\cdot \left(\|\mathbf{u}\|_2^2\cdot\sum_{j\in\{\pm 1\},r\in[m]}\sigma^\prime(\langle\mathbf{w}_{j,r}^{(t)},y\mathbf{u}\rangle)^2 \right.\\
       &\qquad\qquad\left.+\|\mathbf{v}\|_2^2\cdot\sum_{j\in\{\pm 1\},r\in[m]}\sigma^\prime(\langle\mathbf{w}_{j,r}^{(t)},y\mathbf{v}\rangle)^2+\|\boldsymbol{\xi}_{i_t}\|_2^2\cdot \sum_{j\in\{\pm 1\},r\in[m]}\sigma^\prime(\langle\mathbf{w}_{j,r}^{(t)},\boldsymbol{\xi}_{i_t}\rangle)^2\right)\\
        &\qquad \le \mathcal{O}(\max\{\|\mathbf{u}\|_2^2,\|\mathbf{v}\|_2^2,\|\boldsymbol{\xi}_{i_t}\|_2^2\})\cdot\eta^2\cdot\ell^{(t)2}_{i_t},
    \end{align*}
    since the dynamics of inner products $\langle\mathbf{w}^{(t)}_{j,r},y\mathbf{u}\rangle,\langle\mathbf{w}^{(t)}_{j,r},y\mathbf{v}\rangle,\langle\mathbf{w}^{(t)}_{j,r},\boldsymbol{\xi}_i\rangle$ are well bounded by $\mathcal{O}(1)$. Via scaling $\eta\cdot \mathcal{O}(\max\{\|\mathbf{u}\|_2^2,\|\mathbf{v}\|_2^2,\|\boldsymbol{\xi}_{i_t}\|_2^2\})\le 1$, we would know $\eta^2|\ell^{(t)}_{i_t}|^2\|\nabla f(\xb_{i_t};\mathbf{W}^{(t)})\|_F^2\le \eta\ell^{(t)2}_{i_t}$. Eventually, continued from \eqref{eq:multiple-data para-L2-distance difference expansion}, we can completely prove this lemma.
\end{proof}

\begin{lemma}\label{lemma:multiple-data small-lr descent lemma 2}
    Continued from the previous setting, for $T^\dagger\le t\le T^\ddagger$, if $i_t\in\mathcal{W}$, there holds
    \begin{equation}\label{eq:multiple-data small-lr descent 2}
        \|\mathbf{W}^{(t)}-\mathbf{W}^\ast\|_F^2-\|\mathbf{W}^{(t+1)}-\mathbf{W}^\ast\|_F^2\ge -C\eta \sigma_0^2\cdot \big(\norm{\vb}_2^2+\sigma_p^2d\big).
    \end{equation}
\end{lemma}

\begin{proof}[Proof of Lemma~\ref{lemma:multiple-data small-lr descent lemma 2}]
Same as the last lemma, from the SGD setting, we have that 
    \begin{align}
        &\|\mathbf{W}^{(t)}-\mathbf{W}^\ast\|_F^2-\|\mathbf{W}^{(t+1)}-\mathbf{W}^\ast\|_F^2\notag\\
        &\qquad =2\eta\langle \ell^{(t)}_{i_t}\nabla f(\xb_{i_t};\mathbf{W}^{(t)}),\mathbf{W}^{(t)}-\mathbf{W}^\ast\rangle-\eta^2\cdot \left|\ell^{(t)}_{i_t}\right|^2\cdot \|\nabla f(\xb_{i_t};\mathbf{W}^{(t)})\|_F^2,\label{eq:para-L2-distance difference expansion 2}
    \end{align}
    and from the $2$-homogeneity, it follows that
    \begin{equation*}
        \langle\nabla f(\mathbf{x}_{i_t};\mathbf{W}^{(t)}),\mathbf{W}^{(t)}\rangle=2f(\mathbf{x}_{i_t};\mathbf{W}^{(t)}).
    \end{equation*}
    Since $i_t\in\mathcal{W}$, every $\nabla_{\wb_{j,r}}f(\mathbf{x}_{i_t};\mathbf{W}^{(t)})$ is in $\mathrm{span}(\vb,\boldsymbol{\xi}_{i_t},\widetilde{\boldsymbol{\xi}}_{i_t})\perp\ub$, so
    \begin{equation*}
        \langle\nabla f(\mathbf{x}_{i_t};\mathbf{W}^{(t)}),\mathbf{W}^{\ast}\rangle=0.
    \end{equation*}
    As a result, the first term in \eqref{eq:para-L2-distance difference expansion 2} can be bounded by
    \begin{align*}
        &\left|2\eta\langle \ell^{(t)}_{i_t}\nabla f(\xb_{i_t};\mathbf{W}^{(t)}),\mathbf{W}^{(t)}-\mathbf{W}^\ast\rangle\right|\\
        &\qquad =4\eta\cdot\left|\big(f(\mathbf{x}_{i_t};\mathbf{W}^{(t)})-y_{i_t}\big)\cdot f(\mathbf{x}_{i_t};\mathbf{W}^{(t)})\right|\\
        &\qquad \le \frac{12\eta}{m}\cdot \left|\sum_{j\in\{\pm 1\},r\in[m]}\sigma(\langle\mathbf{w}_{j,r}^{(t)},\widetilde{\boldsymbol{\xi}}_{i_t}\rangle)+\sum_{j\in\{\pm 1\},r\in[m]}\sigma(\langle\mathbf{w}_{j,r}^{(t)},y_{i_t}\mathbf{v}\rangle)+\sum_{j\in\{\pm 1\},r\in[m]}\sigma(\langle\mathbf{w}_{j,r}^{(t)},\boldsymbol{\xi}_{i_t}\rangle)\right|\\
        &\qquad \le \mathcal{O}\big(\eta \sigma_0^2\norm{\vb}_2^2+\eta\sigma_0^2\sigma_p^2d\big).
    \end{align*}
    We can also deal with the second term in \eqref{eq:para-L2-distance difference expansion 2} by
    \begin{align*}
        &\eta^2\cdot \left|\ell^{(t)}_{i_t}\right|^2\cdot \|\nabla f(\xb_{i_t};\mathbf{W}^{(t)})\|_F^2\\
        &\qquad \le \eta^2\cdot \ell^{(t)2}_{i_t}\cdot \left(\|\widetilde{\boldsymbol{\xi}}_{i_t}\|^2_2\cdot\sum_{j\in\{\pm 1\},r\in[m]}\sigma^\prime(\langle\mathbf{w}_{j,r}^{(t)},\widetilde{\boldsymbol{\xi}}_{i_t}\rangle)^2\right.\\
        &\qquad\qquad\left.+\|\mathbf{v}\|_2^2\cdot \sum_{j\in\{\pm 1\},r\in[m]}\sigma^\prime(\langle\mathbf{w}_{j,r}^{(t)},y_{i_t}\mathbf{v}\rangle)^2+\|\boldsymbol{\xi}_{i_t}\|_2^2\cdot\sum_{j\in\{\pm 1\},r\in[m]}\sigma^\prime(\langle\mathbf{w}_{j,r}^{(t)},\boldsymbol{\xi}_{i_t}\rangle)^2\right)\\
        &\qquad \le \mathcal{O}\big(\eta^2(\sigma_0^4\norm{\vb}_2^4+\sigma_0^4\sigma_p^4d^2)\big),
    \end{align*}
    where the last inequality is due to $\ell^{(t)}_{i_t}$ being $\mathcal{O}(1)$. Since we already take $\eta\le m/6\|\ub\|_2^2$, the second term of \eqref{eq:para-L2-distance difference expansion 2} is ignorable compared to the first term of \eqref{eq:para-L2-distance difference expansion 2}. Therefore, we can conclude \eqref{eq:multiple-data small-lr descent 2}.
\end{proof}

Equipped with Lemmas~\ref{lemma:multiple-data small-lr 2nd starts}, \ref{lemma:multiple-data inner-prod grad reference pt}, \ref{lemma:multiple-data small-lr descent lemma}, and \ref{lemma:multiple-data small-lr descent lemma 2}, we are ready to prove the main lemma of second stage.

\begin{proof}[Proof of Lemma~\ref{lemma:multiple-data small-lr 2nd-stage}]
Continued from Lemmas~\ref{lemma:multiple-data small-lr descent lemma} and \ref{lemma:multiple-data small-lr descent lemma 2}, for any $t\ge T^\dagger$, it holds that 
\begin{align*}
    &\frac{1}{t-T^\dagger+1}\sum_{s=T^\dagger}^{t}\mathbf{1}\{i_s\notin\mathcal{W}\}\cdot\frac{1}{2}\big(f(\mathbf{x}_{i_t};\mathbf{W}^{(t)})-y_{i_t}\big)^2 \\
    &\qquad \le \frac{\|\mathbf{W}^{(T^\dagger)}-\mathbf{W}^\ast\|_F^2}{2\eta(t-T^\dagger+1)}+\epsilon^2(1-\rho)+C\sigma_0^2\big(\norm{\vb}_2^2+\sigma_p^2d\big)\rho.
\end{align*}
Before proceeding to scale time $t$, it is helpful to decompose $\|\mathbf{W}^{(T^\dagger)}-\mathbf{W}^\ast\|_F^2$ and have an upper bound,
\begin{align}
    \|\mathbf{W}^{(T^\dagger)}-\mathbf{W}^\ast\|_F^2&=\sum_{j\in\{\pm 1\},r\in[m]}\frac{\langle \wb_{j,r}^{(T^\dagger)}-\wb_{j,r}^\ast,\ub\rangle^2}{\norm{\ub}_2^2}+\frac{\langle \wb_{j,r}^{(T^\dagger)}-\wb_{j,r}^\ast,\vb\rangle^2}{\norm{\vb}_2^2}+\left\|\mathbf{P}_{\boldsymbol{\xi},\widetilde{\boldsymbol{\xi}}}(\wb_{j,r}^{(T^\dagger)}-\wb_{j,r}^\ast)\right\|^2_2\\
    &\qquad+\left\|\left(\Ib_d-\frac{\vb\vb^\top}{\|\vb\|_2^2} - \frac{\ub\ub^\top}{\|\ub\|_2^2}-\mathbf{P}_{\boldsymbol{\xi},\widetilde{\boldsymbol{\xi}}}\right)(\wb_{j,r}^{(T^\dagger)}-\wb_{j,r}^\ast)\right\|^2_2\\
    &\le\sum_{j\in\{\pm 1\},r\in[m]}\frac{2\langle \wb_{j,r}^{(T^\dagger)},\ub\rangle^2+2\langle \wb_{j,r}^\ast,\ub\rangle^2}{\norm{\ub}_2^2}+\frac{\langle \wb_{j,r}^{(T^\dagger)},\vb\rangle^2}{\norm{\vb}_2^2}+\left\|\mathbf{P}_{\boldsymbol{\xi},\widetilde{\boldsymbol{\xi}}}\wb_{j,r}^{(T^\dagger)}\right\|^2\\
    &\qquad+\left\|\left(\Ib_d-\frac{\vb\vb^\top}{\|\vb\|_2^2} - \frac{\ub\ub^\top}{\|\ub\|_2^2}-\mathbf{P}_{\boldsymbol{\xi},\widetilde{\boldsymbol{\xi}}}\right)(\wb_{j,r}^{(T^\dagger)}-\wb_{j,r}^\ast)\right\|_2^2,\label{eq: proof stage 2 multi}
\end{align}
where $\mathbf{P}_{\boldsymbol{\xi},\widetilde{\boldsymbol{\xi}}}$ denotes the projection matrix onto linear space $\mathrm{span}_{i\in[n], i'\in\mathcal{W}}(\boldsymbol{\xi}_i,\widetilde{\boldsymbol{\xi}}_{i'})$. 
In these derivations, we exploit the fact that $\wb^\ast$ is parallel to $\ub$, and the gradient steps only updates $\wb$ along the directions of $\ub,\vb$. 
Recall that by Lemma~\ref{lemma:multiple-data small-lr 2nd starts}, 
\begin{align}
    &\max_{j\in\{\pm 1\},r\in[m]}\langle\wb_{j,r}^{(T^\dagger)},j\ub\rangle=\Omega(1), \quad \max_{j\in\{\pm 1\},r\in[m]}\left|\langle\wb^{(T^\dagger)}_{j,r},\vb\rangle\right|=\widetilde{\mathcal{O}}(\sigma_0\norm{\vb}_2), \quad \norm{\wb^{(0)}_{j,r}}_2=\widetilde{\mathcal{O}}(\sigma_0\sqrt{d}),\\
    &\max_{j\in\{\pm 1\},r\in[m], i\in[n]}\left|\langle\wb^{(T^\dagger)}_{j,r},\boldsymbol{\xi}_i\rangle\right|=\widetilde{\mathcal{O}}(\sigma_0\sigma_p\sqrt{d}),
\end{align}
the leading term in \eqref{eq: proof stage 2 multi} is $\sum_{j\in\{\pm 1\},r\in[m]}\langle \wb_{j,r}^\ast,\ub\rangle^2/\norm{\ub}_2^2$. 
Therefore, we would conclude that $\|\mathbf{W}^{(T^\dagger)}-\mathbf{W}^\ast\|_F^2\le Cm^3/\norm{\ub}_2^2$. As a result, the average loss after iterations $T^\dagger$ can be bounded by
\begin{align*}
    &\frac{1}{t-T^\dagger+1}\sum_{s=T^\dagger}^{t}\mathbf{1}\{i_s\notin\mathcal{W}\}\cdot\frac{1}{2}\big(f(\mathbf{x}_{i_t};\mathbf{W}^{(t)})-y_{i_t}\big)^2 \\
    &\qquad\le \frac{Cm^3}{2\eta\norm{\ub}_2^2(t-T^\dagger+1)}+\epsilon^2(1-\rho)+C\sigma_0^2(\norm{\vb}_2^2+\sigma_p^2d)\rho.
\end{align*}
Then choose $T=T^\dagger+\lfloor Cm^3/(2\eta\epsilon\norm{\ub}_2^2)\rfloor$ as stated in Lemma~\ref{lemma:multiple-data small-lr 2nd-stage}.
Since $\norm{\ub}_2^2/\norm{\vb}_2^2\ge\widetilde{\Omega}(m^2)$, we can verify that $T\le T^\ddagger$ where $T^\ddagger$ is given in Lemma~\ref{lemma:multiple-data small-lr 2nd starts} until when the weak signal cannot be fully learned. Moreover, we also have $T\le T_+$ where $T_+$ is given in Lemma~\ref{lemma:multiple-data small-lr upper bound Gamma} when the noise is not memorized.

In conclusion, via scaling $\sigma_0^2\le \epsilon/(C\rho(\norm{\vb}_2^2+\sigma_p^2d))$, the final output would be
\begin{align*}
    \frac{1}{2n}\sum_{i\in\mathcal{W}^c}\min_{T^\dagger\le t\le T}\left\{y_i-f(\xb_i;\Wb^{(t)})\right\}^2&\le\frac{1}{t-T^\dagger+1}\sum_{s=T^\dagger}^{t}\mathbf{1}\{i_s\notin\mathcal{W}\}\cdot\frac{1}{2}\big(f(\mathbf{x}_{i_t};\mathbf{W}^{(t)})-y_{i_t}\big)^2 \\
    &\le \epsilon+\epsilon^2(1-\rho)+C\sigma_0^2\big(\norm{\vb}_2^2+\sigma_p^2d\big)\rho\\
    &\le 3\epsilon,
\end{align*}
ending the proof of Lemma~\ref{lemma:multiple-data small-lr 2nd-stage}.
\end{proof}

\subsection{Stage 3. Memorizing Noise}\label{subsec: stage 3}
After the second stage, the model already fits those data points with strong signal by exploiting $\ub$. Subsequently, in the following third stage, the residual $\ell^{(t)}_{i}$ for $i\in\mathcal{W}^c$ would remain quite small, preventing the model from learning $\ub$.

On the contrary, since $f(\xb_{i};\Wb^{(t)})=\widetilde{\mathcal{O}}(\sigma_0^2)$ is still far from its label $y_i$ for each sample $i\in\mathcal{W}$ without the strong signal $\ub$. 
Therefore, the weight vectors would still evolve in the directions perpendicular to $\ub$. In Assumption~\ref{cond:model_params_multiple_data small-lr}, the ratio between the weak signal $\mathbf{v}$ and the typical noise norm $\sigma_p\sqrt{d}$ is scaled by
\begin{align}
    \frac{\norm{\vb}_2}{\sigma_p\sqrt{d}}\le \widetilde{\mathcal{O}}\left(\frac{1}{\sqrt{n}}\right).\label{eq: stage 3 scale}
\end{align}
Therefore, the model will eventually interpolates the whole dataset by memorizing noise vectors $(\widetilde{\boldsymbol{\xi}}_i,\boldsymbol{\xi}_i),i\in\mathcal{W}$, Now we define a new reference point $\Wb^\star$ by
\begin{equation*}
    \wb^{\star}_{j,r}=\wb^\ast_{j,r}+\frac{4m(1+\epsilon^\prime)}{\iota}\cdot \left(\sum_{i\in\mathcal{W}}\mathbf{1}\{y_i=j\}\cdot\frac{\boldsymbol{\xi}_i}{\norm{\boldsymbol{\xi}_i}_2}+\mathbf{1}\{y_i=j\}\cdot\frac{\widetilde{\boldsymbol{\xi}}_i}{\norm{\widetilde{\boldsymbol{\xi}}_i}_2}\right),\quad j\in\{\pm 1\},r\in[m],
\end{equation*}
where $\wb^\ast_{j,r}$ defined in \eqref{eq:multiple-data reference pt 2nd} is the reference point we used in the second stage. The following lemma is an adaptation of Theorem 4.4 of \cite{cao2022benign} onto SGD with square loss.
\begin{lemma}[Third stage]\label{lemma:multiple-data small-lr 3rd-stage}
    Under the same conditions as Theorem~\ref{thm:multiple-data small-lr formal}, for some $\epsilon^\prime\in(0,1)$, let
    \begin{equation*}
        T^\prime=T+\left\lfloor\frac{\left\|\Wb^{(T)}-\Wb^\star\right\|_F^2}{2\eta\epsilon}\right\rfloor,
    \end{equation*}
    where $T$ is the end of the second stage in Lemma~\ref{lemma:multiple-data small-lr 2nd-stage}. 
    Then we would have that 
    \begin{align}
        \max_{j\in\{\pm 1\},r\in[m]}\left|\langle\wb^{(t)}_{j,r},\vb\rangle\right|\le \widetilde{\mathcal{O}}(\sigma_0\norm{\vb}_2),
    \end{align}
    even until $t\le T^\prime$. 
    But the whole dataset has already been interpolated during this interval,
    \begin{equation*}
        \frac{1}{2\rho n}\sum_{i\in\mathcal{W}}\min_{T\le t\le T^\prime}\left\{y_i-f(\xb_i;\Wb^{(t)})\right\}^2\le 3\epsilon^\prime.
    \end{equation*}
\end{lemma}
\begin{proof}[Proof of Lemma~\ref{lemma:multiple-data small-lr 3rd-stage}.]
    As the closing stage of the training dynamics, the evolution dynamics during this interval is straightforward based on all techniques developed in Appendices~\ref{sec: one data small lr} and \ref{sec:multiple-data small lr}. Inner products
    \begin{align}
        \left\{\langle\mathbf{w}_{j,r}^{(t)},\boldsymbol{\xi}_i\rangle,\langle\mathbf{w}_{j,r}^{(t)},\widetilde{\boldsymbol{\xi}}_i\rangle\right\}_{i\in\mathcal{W}}
    \end{align}
    would firstly go through a substage in which they exponentially increase to a constant level (just as stage 1). And then the model will fit all samples indexed by $\mathcal{W}$ by memorizing these noise vectors in polynomial time. All through this interval, $\max_{j\in\{\pm 1\},r\in[m]}|\langle\wb^{(t)}_{j,r},\vb\rangle|$ would stay $\widetilde{\mathcal{O}}(\sigma_0\norm{\vb})$ due to the scale of $\norm{\vb}_2/(\sigma_p\sqrt{d})$ in \eqref{eq: stage 3 scale}. 
    A detailed proof is omitted here for readability.  
\end{proof}

Combine Lemmas~\ref{lemma:multiple-data small-lr 1st-stage}, \ref{lemma:multiple-data small-lr 2nd-stage} and \ref{lemma:multiple-data small-lr 3rd-stage} to obtain the full version of Theorem~\ref{thm:multiple-data small-lr formal}.

\end{document}